\documentclass[12pt]{report}

\usepackage{natbib}
\usepackage{textcomp}
\usepackage[bitstream-charter]{mathdesign}
\usepackage{amsmath}
\usepackage{amsthm}
\usepackage{nicefrac}
\usepackage{bbm}
\usepackage{appendix}
\usepackage{hyperref}

\usepackage[usenames,dvipsnames]{xcolor}

\usepackage{graphicx}
\usepackage[labelfont=bf]{caption}
\usepackage[format=hang]{subcaption}
\usepackage{tikz}
\usetikzlibrary{arrows,decorations.pathmorphing,backgrounds,positioning,fit,petri}

\usepackage[algoruled]{algorithm2e}
\usepackage{listings}
\usepackage{fancyvrb}
\fvset{fontsize=\normalsize}

\usepackage[nameinlink,capitalise]{cleveref}
\crefname{listing}{Example Code}{Example Code}  
\Crefname{listing}{Example Code}{Example Code}

\usepackage[acronym,nowarn]{glossaries}
\setacronymstyle{long-sc-short}

\lstdefinestyle{mystyle}{
    commentstyle=\color{OliveGreen},
    keywordstyle=\color{BurntOrange},
    numberstyle=\tiny\color{black!60},
    stringstyle=\color{MidnightBlue},
    basicstyle=\ttfamily,
    breakatwhitespace=false,
    breaklines=true,
    captionpos=b,
    keepspaces=true,
    numbers=left,
    numbersep=5pt,
    showspaces=false,
    showstringspaces=false,
    showtabs=false,
    tabsize=2
}
\newtheorem{theorem}{Theorem}[section]

\newtheorem{proposition}{Proposition}[section]
\newtheorem{remark}{Remark}[section]
\newtheorem{definition}{Definition}[section]
\newtheorem{example}{Example}[section]

\DeclareMathOperator*{\argmin}{arg\,min}

\newcommand\dif{\mathop{}\!\mathrm{d}}

\newcommand{\const}{\textrm{const}}
\newcommand{\ordo}{\mathcal{O}}

\newcommand{\Prb}{\mathbb{P}} 
\newcommand{\Exp}{\mathbb{E}} 
\newcommand{\Var}{\textrm{Var}} 
\newcommand{\Kl}{\textrm{KL}} 
\newcommand{\symKL}{\textrm{D}_{\textrm{J}}} 
\newcommand{\grad}{\nabla} 
\newcommand{\hessian}{\nabla^2} 
\newcommand{\jacob}{J} 

\newcommand{\eqdef}{:=} 
\newcommand{\defeq}{=:} 
\newcommand{\iidsim}{\stackrel{\text{\iid}}{\sim}} 
\newcommand{\given}{\,|\,}
\newcommand{\idxby}{\,;\,}
\newcommand{\convD}{\stackrel{\textrm{d}}\longrightarrow}     
\newcommand{\convAS}{\stackrel{\textrm{a.s.}}\longrightarrow}     
\newcommand{\eqD}{\stackrel{\textrm{d}}=}

\newcommand{\fun}{h} 

\newcommand{\dirac}{\delta} 
\newcommand{\ident}{\mathbbm{1}}
\newcommand{\ess}{\textrm{ESS}} 

\newcommand{\xm}{\mathbf{x}} 
\newcommand{\ym}{\mathbf{y}} 
\newcommand{\nmod}{\gamma} 
\newcommand{\umod}{\widetilde{\gamma}} 

\newcommand{\tm}{t} 
\newcommand{\km}{k} 
\newcommand{\lm}{l} 
\newcommand{\mm}{m} 
\newcommand{\Tm}{T} 
\newcommand{\Km}{K} 
\newcommand{\Z}{Z} 
\newcommand{\mparams}{\theta} 
\newcommand{\y}{y} 
\newcommand{\dimY}{n_{\y}} 
\newcommand{\vn}{v} 
\newcommand{\en}{e} 
\newcommand{\zn}{z} 
\newcommand{\uv}{\mathrm{u}} 
\newcommand{\fmod}{f} 
\newcommand{\gmod}{g} 
\newcommand{\ufmod}{\widetilde{f}} 
\newcommand{\amod}{a} 
\newcommand{\cmod}{c} 
\newcommand{\model}{p} 
\newcommand{\tmp}{\tau} 

\newcommand{\ip}{i} 
\newcommand{\jp}{j} 
\newcommand{\Jp}{J} 
\newcommand{\Np}{N} 
\newcommand{\xp}{x} 
\newcommand{\dimX}{n_{\xp}}
\newcommand{\nwp}{w} 
\newcommand{\uwp}{\widetilde{w}} 
\newcommand{\ap}{a} 
\newcommand{\bp}{b} 
\newcommand{\prop}{q} 
\newcommand{\propwx}{Q} 
\newcommand{\uprop}{\widetilde{\eta}} 
\newcommand{\normprop}{\eta} 
\newcommand{\up}{\mathbf{u}} 

\newcommand{\Mp}{M} 
\newcommand{\nxp}{\bar\xp} 
\newcommand{\nuwp}{\widetilde{\nu}} 
\newcommand{\nprop}{r} 
\newcommand{\smod}{s} 
\newcommand{\rsdist}{\tau} 

\newcommand{\hatZ}{\widehat{Z}} 
\newcommand{\pparams}{\lambda} 
\newcommand{\pfun}{\eta} 

\newcommand{\varp}{\textrm{V}} 

\newcommand{\tpot}{\psi} 
\newcommand{\tparams}{\rho} 
\newcommand{\tpsi}{\bar{\psi}} 

\newcommand{\qSMC}{\textrm{Q}_{\text{SMC}}} 
\newcommand{\qCSMC}{\textrm{Q}_{\text{CSMC}}} 

\newcommand{\sigp}{\mathrm{\zn}} 
\newcommand{\sigw}{\mathrm{\omega}} 
\newcommand{\zmod}{\tilde \amod} 

\newcommand{\gadapt}{\mathrm{g}_{\Kl}} 
\newcommand{\eps}{\varepsilon} 
\newcommand{\gvsmc}{\mathrm{g}_{\mathrm{VSMC}}} 

\newcommand{\Uni}{\mathcal{U}}
\newcommand{\Norm}{\mathcal{N}}

\newcommand{\Poisson}{\textrm{Poisson}}
\newacronym{PDF}{pdf}{probability distribution function}
\newacronym{CDF}{cdf}{cumulative distribution function}
\newacronym{LVM}{lvm}{latent variable model}
\newacronym{PGM}{pgm}{probabilistic graphical model}
\newacronym{SSM}{ssm}{state space model}
\newacronym{RNN}{rnn}{recurrent neural network}
\newacronym{SDE}{sde}{stochastic differential equation}

\newacronym{UKF}{ukf}{unscented Kalman filter}
\newacronym{EKF}{ekf}{extended Kalman filter}

\newacronym{CLT}{clt}{central limit theorem}

\newacronym{MC}{mc}{Monte Carlo}
\newacronym{IS}{is}{importance sampling}
\newacronym{SIS}{sis}{sequential importance sampling}
\newacronym{AIS}{ais}{annealed importance sampling}
\newacronym{SIR}{sir}{sequential importance resampling}
\newacronym{SISR}{sisr}{sequential importance sampling and resampling}
\newacronym{PF}{pf}{particle filter}
\newacronym{SMC}{smc}{sequential Monte Carlo}
\newacronym{NSMC}{nsmc}{nested sequential Monte Carlo}
\newacronym{MH}{mh}{Metropolis-Hastings}
\newacronym{CSMC}{csmc}{conditional sequential Monte Carlo}
\newacronym{PG}{pg}{particle Gibbs}
\newacronym{PMCMC}{pmcmc}{particle Markov chain Monte Carlo}
\newacronym{IPMCMC}{ipmcmc}{interacting particle Markov chain Monte Carlo}
\newacronym{PMH}{pmh}{particle Metropolis-Hastings}
\newacronym{PIMH}{pimh}{particle independent Metropolis-Hastings}
\newacronym{PMMH}{pmmh}{particle marginal Metropolis-Hastings}
\newacronym{MCMC}{mcmc}{Markov chain Monte Carlo}
\newacronym{PM}{pm}{pseudo marginal}
\newacronym{CPM}{cpm}{correlated pseudo marginal}

\newacronym{ESS}{ess}{effective sample size}

\newacronym{BREAD}{bread}{bounding divergences with reverse annealing}
\newacronym{AIDE}{aide}{auxiliary inference divergence estimator}

\newacronym{KL}{kl}{Kullback-Leibler}

\newacronym{ELBO}{elbo}{\emph{evidence lower bound}}
\newacronym{VI}{vi}{variational inference}
\newacronym{VSMC}{vsmc}{variational sequential Monte Carlo}

\newcommand{\eg}{e.g.\@\xspace}
\newcommand{\ie}{i.e.\@\xspace}
\newcommand{\etc}{etc}               
\newcommand{\wrt}{w.r.t.\@\xspace}

\newcommand{\cf}{cf.\@\xspace}

\newcommand{\iid}{iid\@\xspace}

\title{\bf Elements of Sequential Monte Carlo}
\author{
Christian A. Naesseth \\
Columbia University \\
christian.a.naesseth@columbia.edu
\and
Fredrik Lindsten \\
Link\"oping University \\
fredrik.lindsten@liu.se
\and
Thomas B. Sch\"on \\
Uppsala University \\
thomas.schon@it.uu.se
}

\date{\textbf{Please cite this version:} {Christian A. Naesseth, Fredrik Lindsten and Thomas B. Sch\"{o}n. Elements of Sequential Monte Carlo. \textit{Foundations and Trends\textregistered ~in Machine Learning}, 12(3): 307--392, 2019.}}

\begin{document}

\maketitle

\begin{abstract}
    A core problem in statistics and probabilistic machine learning is to 
    compute probability distributions and expectations. This is the 
    fundamental problem of Bayesian statistics and machine learning, which frames all inference 
    as expectations with respect to the posterior distribution. The key 
    challenge is to approximate these intractable expectations. 
    In this tutorial, we review \gls{SMC}, a random-sampling-based 
    class of methods for approximate inference. 
    First, we explain the basics of 
    \gls{SMC}, 
    discuss practical issues, and review theoretical results.
    We then examine two of the main user design choices: the \emph{proposal 
    	distributions} and the so called \emph{intermediate target distributions}.
    We review recent results on how variational inference and amortization can be used to learn efficient proposals and target distributions.
    Next, we discuss the \gls{SMC} estimate of the normalizing constant, how this can be used for pseudo-marginal inference and inference evaluation.
    Throughout the tutorial we illustrate the use of \gls{SMC} on various models commonly used in machine learning, such as stochastic recurrent neural networks, probabilistic graphical models, and probabilistic programs.
\end{abstract}

\tableofcontents

\chapter{Introduction}\label{sec:introduction}
A key strategy in machine learning is to break down a problem into smaller and
more manageable parts, then process data or unknown variables recursively. 
Well known examples of this are message passing algorithms for graphical models and
annealing for optimization or sampling. \Acrfull{SMC} is a class of
methods that are tailored to solved statistical inference problems recursively. These methods
have mostly received attention in the signal processing and statistics 
communities. With well over two decades of research in \gls{SMC}, they 
have enabled inference in increasingly complex and challenging models.
Recently, there has been an emergent interest in this class of algorithms
from the machine learning community. We have seen applications to probabilistic programming 
\citep{wood2014new}, \gls{VI} 
\citep{maddison2017,naesseth18a,anh2018autoencoding}, inference 
evaluation \citep{grosse2015sandwiching,towner-nips-2017}, \glspl{PGM} 
\citep{ihler2009particle,naessethLS2014,paige2016inference}, Bayesian nonparametrics \citep{Fearnhead:2004} and many 
other areas.

We provide a unifying view of the \gls{SMC} methods that have been 
developed since their conception in the early 1990s 
\citep{GordonSS:1993,stewart1992,Kitagawa:1993}. In this introduction 
we provide relevant background material, introduce a running 
example, and discuss the use of code snippets throughout the tutorial.

\section{Historical Background}\label{sec:background}
\gls{SMC} methods are generic tools for performing approximate 
(statistical) inference, 
predominantly Bayesian inference. They use a weighted sample set to 
iteratively approximate the posterior distribution of a 
probabilistic model. Ever since the dawn of \acrlong{MC} methods (see 
\eg \citet{metropolis1949monte} for an early discussion), 
random sample-based approximations have been recognized as powerful 
tools for inference in complex probabilistic models. Parallel to the 
development of \gls{MCMC} methods 
\citep{metropolis1953equation,Hastings1970}, \gls{SIS} \citep{handschin1969monte} and 
sampling/importance resampling \citep{Rubin1987} laid the foundations 
for what would one day become \gls{SMC}. 

\gls{SMC} methods where initially known as particle filters
\citep{GordonSS:1993,stewart1992,Kitagawa:1993}. Particle filters where 
conceived as algorithms for online inference in nonlinear \glspl{SSM} 
\citep{cappe2005}. Since then there has been a flurry of work applying 
\gls{SMC} and particle filters to perform approximate inference in 
ever more complex models. While research in \gls{SMC} initially 
focused on \glspl{SSM}, we will see that \gls{SMC} can be a powerful 
tool in a much broader setting.

\section{Probabilistic Models and Target Distributions}\label{sec:seqmodels}
As mentioned above, \gls{SMC} methods were originally developed as an 
approximate solution to the so called filtering problem, which amounts 
to online inference in dynamical models. Several overview and tutorial 
articles focus on particle filters, \ie the \gls{SMC} algorithms 
specifically tailored to solve the online filtering problem 
\citep{arulampalam2002tutorial,doucet2009tutorial,fearnhead2018particle}.
In this tutorial we will take a different view and explain how 
\gls{SMC} can be used to solve more general ``offline'' problems. We 
shall see how this viewpoint opens up for many interesting applications 
of \gls{SMC} in machine learning that do not fall in the traditional 
filtering setup, and furthermore how it gives rise to new and 
interesting design choices. For complementary review and tutorial articles that treat a similar topic see also \citet{delmoral2006tutorial,doucet2018tutorial}.

We consider a generic probabilistic model given by
a joint \gls{PDF} of latent variables~$\xm$ and observed data~$\ym$, 
\begin{align}
    p(\xm,\ym).
    \label{eq:probmodel}
\end{align}
We focus on Bayesian inference, where the key object is the posterior 
distribution
\begin{align}
    p(\xm \given \ym) = \frac{p(\xm,\ym)}{p(\ym)},
    \label{eq:posterior}
\end{align}
where $p(\ym)$ is known as the marginal likelihood.

The \emph{target distributions} are a sequence of 
probability distributions that we recursively approximate using 
\gls{SMC}. We define each target 
distribution $\nmod_\tm(\xp_{1:\tm})$ in the sequence as a joint 
\gls{PDF} of latent variables $\xp_{1:\tm}=(\xp_1,\ldots,\xp_\tm)$, where $t=1,\ldots,T$. 
The \gls{PDF} is denoted by
\begin{align}
    \nmod_\tm(\xp_{1:\tm}) &\eqdef \frac{1}{\Z_\tm}\umod_\tm(\xp_{1:\tm}), 
    \quad \tm = 1,\ldots,\Tm,
    \label{eq:seqmodel}
\end{align}
where $\umod_\tm$ is a positive integrable function and $\Z_\tm$ is the 
normalization constant, ensuring that $\nmod_\tm$ is indeed a \gls{PDF}. 

We connect the target distributions to the probabilistic model through 
a requirement on the final target distribution 
$\nmod_\Tm(\xp_{1:\Tm})$. We enforce the condition that 
$\nmod_\Tm(\xp_{1:\Tm})$ is either equivalent to the posterior 
distribution, or that it contains the posterior distribution as a marginal 
distribution. The \emph{intermediate} target distributions, \ie 
$\nmod_\tm(\xp_{1:\tm})$ for $\tm <\Tm$, are useful only insofar they 
help us approximate the final target $\nmod_\Tm(\xp_{1:\Tm})$. This approach is 
distinct from most previous tutorials on particle filters and \gls{SMC} 
that traditionally focus on the intermediate targets, \ie the 
filtering distributions. We stress that
it is not necessarily the case that $\xp_{1:\Tm}$ (the latent variables of the target distribution) is equal to $\xm$ (the latent variables of the probabilistic model of interest), as the former can include additional auxiliary variables.

Below we introduce a few examples of probabilistic models and some 
straightforward choices of target distributions. We introduce and 
illustrate our running example which will be used throughout. We will 
return to the issue of choosing the sequence of intermediate targets 
in \cref{sec:target}.

\paragraph{State Space Models} The \acrlong{SSM} (or hidden Markov model) is 
a type of probabilistic models where the latent variables and data 
satisfy a Markov property. For this model we typically have $\xm 
= \xp_{1:\Tm}$. Often the measured data can also be split into a 
sequence of the same length ($\Tm$) as the latent variables, \ie $\ym = \y_{1:\Tm}$. The model is 
defined by a transition \gls{PDF} $\fmod$ and an observation \gls{PDF} $\gmod$,
\begin{subequations}
    \begin{align}
    \xp_\tm \given \xp_{\tm-1} &\sim \fmod(\cdot \given\xp_{\tm-1}), \\
    \y_\tm \given \xp_{\tm} &\sim \gmod(\cdot\given\xp_\tm).
    \end{align}
    \label{eq:intro:ssm}%
\end{subequations}
The joint \gls{PDF} is
\begin{align}
    p(\xm,\ym) &= p(\xp_1) \gmod(\y_1\given\xp_1)\prod_{\tm=2}^\Tm 
    \fmod(\xp_\tm\given\xp_{\tm-1}) \gmod(\y_\tm\given\xp_\tm),
    \label{eq:ssm:joint}
\end{align}
where $p(\xp_1)$ is the prior on the initial state~$\xp_1$.
This class of models is especially common for data 
that has an inherent temporal structure such as in the field of signal 
processing.
A common choice is to let the target distributions follow the same 
sequential structure as in \cref{eq:ssm:joint}:
\begin{align}
    \umod_\tm(\xp_{1:\tm}) &= p(\xp_1) 
    \gmod(\y_1\given\xp_1)\prod_{\km=2}^\tm 
    \fmod(\xp_\km\given\xp_{\km-1}) \gmod(\y_\km\given\xp_\km),
    \label{eq:intro:ssm:filtertarget}
\end{align}
which means that the final normalized target distribution satisfies 
$\nmod_\Tm(\xp_{1:\Tm}) = p(\xm\given\ym)$ as required.
This is the model class and target distributions which are studied in 
the classical filtering setup.

\paragraph{Non-Markovian Latent Variable Models} The non-Markovian 
\glspl{LVM} are characterized by either no, or higher order, Markov 
structure between the latent variables $\xm$ and/or data 
$\ym$. This can be seen 
as a non-trivial extension of the \gls{SSM}, see 
\cref{eq:intro:ssm}, which has a Markov structure. Also for this class 
of models it is common to have $\xm = \xp_{1:\Tm}$ and $\ym = 
\y_{1:\Tm}$.

Unlike the \gls{SSM}, the non-Markovian \gls{LVM} in its most general 
setting requires access to all previous latent variables 
$\xp_{1:\tm-1}$ to generate $\xp_\tm,\y_\tm$
\begin{subequations}
    \begin{align}
    \xp_\tm \given \xp_{1:\tm-1} &\sim \fmod_\tm(\cdot \given\xp_{1:\tm-1}), \\
    \y_\tm \given \xp_{1:\tm} &\sim \gmod_\tm(\cdot\given\xp_{1:\tm}),
    \end{align}
    \label{eq:intro:nmlvm}%
\end{subequations}
where we again refer to $\fmod_\tm$ and $\gmod_\tm$ as the transition \gls{PDF}
and observation \gls{PDF}, respectively. The joint \gls{PDF} is given by
\begin{align}
    p(\xm,\ym) &= p(\xp_1) \gmod(\y_1\given\xp_1)\prod_{\tm=2}^\Tm 
    \fmod_\tm(\xp_\tm\given\xp_{1:\tm-1}) 
    \gmod_\tm(\y_\tm\given\xp_{1:\tm}),
    \label{eq:nmlvm:joint}
\end{align}
where $p(\xp_1)$ is again the prior on $\xp_1$. A typical target 
distribution is given by
\begin{align}
    \umod_\tm(\xp_{1:\tm}) &= 
    \umod_{\tm-1}(\xp_{1:\tm-1})\fmod_\tm(\xp_\tm\given\xp_{1:\tm-1}) 
    \gmod_\tm(\y_\tm\given\xp_{1:\tm}), \quad \tm > 1,
    \label{eq:nmlvm:filteringtarget}
\end{align}
with $\umod_1(\xp_1) = p(\xp_1) \gmod_1(\y_1\given\xp_1)$. Another 
option is 
\begin{align*}
    \umod_1(\xp_1) &= p(\xp_1), \\
    \umod_\tm(\xp_{1:\tm}) &= 
    \umod_{\tm-1}(\xp_{1:\tm-1})\fmod_\tm(\xp_\tm\given\xp_{1:\tm-1}), 
    \quad 1< \tm < \Tm, \\
    \umod_\Tm(\xp_{1:\Tm}) &= \umod_{\Tm-1}(\xp_{1:\Tm-1})
    \fmod_\Tm(\xp_\Tm\given\xp_{1:\Tm-1}) \prod_{\tm=1}^\Tm 
    \gmod_\tm(\y_\tm\given\xp_{1:\tm}).
\end{align*}
For both these sequences of target distributions the final iteration $\Tm$ is the posterior 
distribution, \ie $\nmod_\Tm(\xp_{1:\Tm}) = 
\model(\xp_{1:\Tm}\given \y_{1:\Tm}) = p(\xm\given\ym)$. However, the 
former one will often lead to more accurate inferences. This is 
because we introduce information from the data at an earlier stage in 
the \gls{SMC} algorithm.

Throughout the monograph we will exemplify the different methods using 
a Gaussian special case of \cref{eq:intro:nmlvm}, see \cref{ex:running}. 
We let the prior on $\xp_{1:\tm}$, defined by the transition 
\glspl{PDF} $\fmod_1, \ldots \fmod_\tm$, be Markovian and introduce the non-Markov 
property instead through the observation \glspl{PDF} $\gmod_1,\ldots,\gmod_\tm$.
\begin{example}[Non-Markovian Gaussian Sequence Model]
As our running example for illustration purposes we use a non-Markovian 
Gaussian sequence model. It is
\begin{align}
    \xp_\tm \given \xp_{1:\tm-1} &\sim \fmod_\tm(\cdot\given\xp_{\tm-1}),
    &&\y_\tm \given \xp_{1:\tm} \sim \gmod_\tm(\cdot \given \xp_{1:\tm}),
\end{align}
with observed variables $\y_\tm$ (data), and where 
\begin{align*}
    \fmod_\tm(\xp_\tm\given\xp_{\tm-1}) &= 
    \Norm\left(\xp_\tm \given \phi \xp_{\tm-1}, q \right), \\
    \gmod_\tm(\y_\tm\given \xp_{1:\tm}) &= \Norm\left(\y_\tm\given
    \sum_{\km=1}^\tm \beta^{\tm-\km} \xp_{\km}, r\right),
\end{align*}
where $\Norm\left(\xp\given \mu, \Sigma \right)$ denotes a Gaussian distribution on $\xp$ with mean $\mu$ and (co)variance $\Sigma$.
We let the prior at $\tm=1$ be $p(\xp_1) = \Norm(\xp_1\given 0, q)$. 
Artificial data was generated using 
$(\phi, q, \beta, r) = (0.9, 1, 0.5, 1)$. The distribution of interest 
is the posterior distribution $\model(\xp_{1:\Tm}\given \y_{1:\Tm})$. 
We illustrate a few sample paths of $\y_{1:\Tm}$ in \cref{fig:nmlvmex} 
for $\Tm = 100$.
\begin{figure}[tb]
    \centering
    \includegraphics[width=0.7\textwidth]{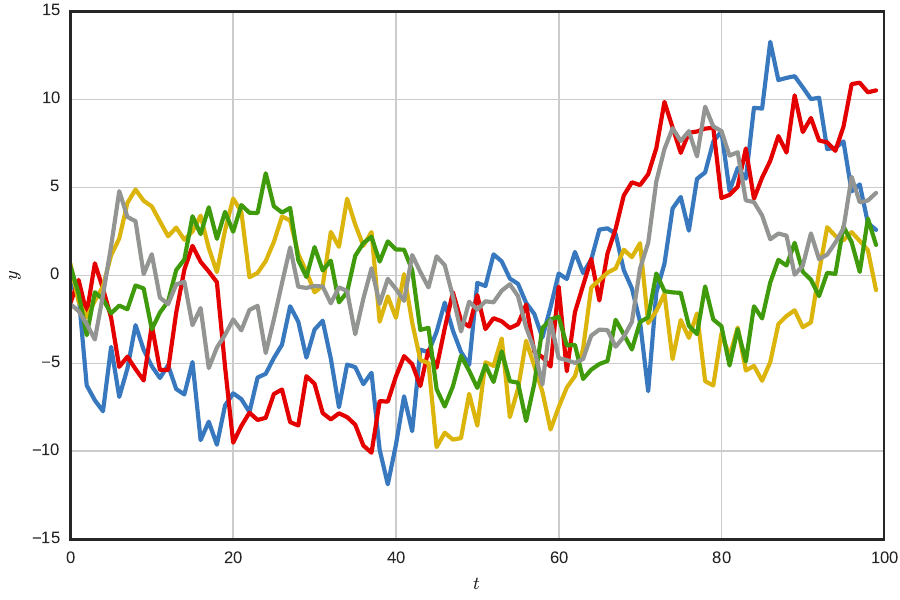}
    \caption{Five sample paths of $\y_{1:\Tm}$ from our running 
    example for $\Tm=100$.}
    \label{fig:nmlvmex}
\end{figure}

We can adjust the strength of the dependence on previous latent 
variables in the observations, $\y_\tm$, through the parameter $\beta \in [0, 1]$. If we 
set $\beta = 0$ we obtain a linear Gaussian \gls{SSM}, since the data 
depends only on the most recent latent $\xp_\tm$. On the other 
hand if we let $\beta = 1$, this signifies that $\xp_\km$ for $\km < 
\tm$ has equally strong effect on $\y_\tm$ as does $\xp_\tm$.
\label{ex:running}
\end{example}

Another example of a non-Markovian \gls{LVM} often encountered in machine learning is the stochastic \gls{RNN} \citep{Chung2015, bayer2014,FraccaroSPW:2016,maddison2017,naesseth18a}. We define the stochastic \gls{RNN} below and will return to it again in \cref{sec:proptwisting}.
\begin{example}[Stochastic Recurrent Neural Network]
A stochastic \gls{RNN} is a non-Markovian \gls{LVM} where the 
parameters of the transition and observation models are defined by 
\gls{RNN}s. A common example is using the conditional Gaussian 
distribution to define the transition \gls{PDF}
\begin{align*}
    \fmod_\tm(\xp_\tm\given\xp_{1:\tm-1}) &= \Norm\left(\xp_\tm\given 
    \mu_\tm(\xp_{1:\tm-1}), \Sigma_\tm(\xp_{1:\tm-1})\right),
\end{align*}
where the functions $\mu_\tm(\cdot), \Sigma_\tm(\cdot)$ are defined by
\gls{RNN}s.
\end{example}

\paragraph{Conditionally independent models} 
A common model in probabilistic machine learning is to assume that the 
datapoints $\y_\km$ in the dataset 
$\ym=\{\y_\km\}_{\km=1}^\Km$ are conditionally independent given the latent $\xm$. 
This means that the joint \gls{PDF} is given by
\begin{align}
    p(\xm,\ym) &= p(\xm) \underbrace{\prod_{\km=1}^{\Km} 
    \gmod_\km(\y_\km\given\xm)}_{p(\ym\given\xm)},
\end{align}
where $p(\ym\given\xm)$ is the likelihood.
For this class of models it might not be immediately apparent that we 
can define a useful sequence of target distributions on latent variables $\xp_{1:\tm}$. There is no obvious sequence structure as in the \gls{SSM} or the non-Markovian \gls{LVM}. However, as we 
shall see, we can make use of auxiliary variables to design target 
distributions that can help with inference.

We will discuss two approaches to design the sequence of target 
distributions: using data tempering and likelihood tempering, 
respectively. Both of these will make use of an auxiliary variable 
technique, where each $\xp_\tm$ is a random variable on the same space 
as $\xm$.

\emph{Data tempering:} Using data tempering we add the data $\y_\km$ 
to the target distribution one by one. In this case the data index 
$\km$ coincides with the target index $\tm$. We define the target 
distribution
\begin{align}
    \umod_\tm(\xp_{1:\tm}) &= p(\xp_\tm) \prod_{\km=1}^\tm 
    \gmod_\km(\y_\km\given\xp_{\tm}) \cdot 
    \prod_{\km=1}^{\tm-1} \smod_\km(\xp_\km \given \xp_{\km+1}),
\end{align}
where the distributions $\smod_\km(\xp_\km \given \xp_{\km+1})$ are a 
design choice, known as \emph{backward kernels} \citep{chopin2002sequential,del2006sequential}. 
With this choice, we have that the marginal distribution of $\xp_\Tm$ at the final iteration is 
exactly the posterior distribution, \ie $\nmod_\Tm(\xp_\Tm) = 
p(\xm\given\ym)$. In fact, at each step we have that the marginal target 
distribution is a partial posterior $\nmod_\tm(\xp_\tm) = p(\xm\given 
\y_{1:\tm})$.

\emph{Likelihood tempering:} With likelihood tempering, instead of 
adding the data one by one, we change the likelihood $p(\ym\given\xm)$ 
through a sequence of positive variables. We define the target 
distribution
\begin{align}
    \umod_\tm(\xp_{1:\tm}) &= p(\xp_\tm) p(\ym\given\xp_\tm)^{\tmp_\tm}\cdot 
    \prod_{\km=1}^{\tm-1} \smod_\km(\xp_\km \given \xp_{\km+1}),
\end{align}
where $0 = \tmp_1 < \ldots < \tmp_\Tm = 1$, and again make use of  
user chosen backward kernels $\smod_\km(\xp_\km \given \xp_{\km+1})$ \citep{chopin2002sequential,del2006sequential}. 
In this setting all data is considered at each iteration. Since 
$\tmp_\Tm=1$, we have that the final marginal target distribution 
is again equal to the posterior $\nmod_\Tm(\xp_\Tm) = p(\xm\given\ym)$.

Applying \gls{SMC} methods to tempered (and similar) target distributions has 
been studied by \eg \citet{chopin2002sequential,del2006sequential}. If the proposal $\prop_{\km+1}(\xp_{\km+1}\given\xp_{\km})$ is a Markov kernel with stationary distribution $\nmod_{\km+1}(\xp_{\km+1})$, then a commonly used backward kernel is $\smod_\km(\xp_\km \given \xp_{\km+1}) = \frac{\nmod_{\km+1}(\xp_\km)\prop_{\km+1}(\xp_{\km+1}\given\xp_{\km})}{\nmod_{\km+1}(\xp_{\km+1})}$. We refer to the 
works of \citet{chopin2002sequential,del2006sequential} for a thorough discussion on the choice of backward kernels
$\smod_\km(\xp_\km \given \xp_{\km+1})$. Another well known 
example is \emph{annealed importance sampling} by \citet{neal2001annealed}, 
which uses the backward kernel example above to define the target distributions, but relies on \gls{SIS} instead of \gls{SMC} for inference.

\paragraph{Models and Targets} We have seen several probabilistic 
models with examples of corresponding target distributions. While not 
limited to these, this illustrates the wide range of the applicability 
of \gls{SMC}. In fact, as long as we can design a sequence of target 
distributions such that $\nmod_\Tm$ coincides with the distribution of 
interest, we can leverage \gls{SMC} for inference.

\section{Applications}\label{sec:applications}
\Acrlong{SMC} and \acrlong{IS} methods have already seen a plethora of applications to machine learning and statistical inference problems. Before we turn to the fundamentals of the various algorithms it can be helpful to understand some of these applications. We present and discuss a few select examples of applications of \gls{SMC} and \gls{IS} to \acrlongpl{PGM},
Bayesian nonparametric models, probabilistic programming, and inference evaluation. 

\paragraph{Probabilistic Graphical Models}
\Acrlongpl{PGM} (\gls{PGM}; see \eg \citet{koller2009probabilistic,wainwright2008graphical}) are probabilistic models where the conditional independencies in the joint \gls{PDF} are described by edges in a graph. The graph structure allows for easy and strong control on the type of prior information that the user can express. The main limitation of the \gls{PGM} is that exact inference is often intractable and approximate inference is challenging.

The \gls{PGM} is a probabilistic model where the \gls{PDF} factorizes according to an underlying graph described by a set of cliques $C\in\mathcal{C}$, \ie fully connected subsets of the vertices $V \in \mathcal{V}$ where $\mathcal{V}$ contains all individual components of $\xm$ and $\ym$. The undirected graphical model can be denoted by
\begin{align}
	p(\xm,\ym) &= \frac{1}{\Z} \prod_{C \in \mathcal{C}} \psi_C(\xp_C),
\end{align}
where $\xp_C$ includes all elements of $\xm$ and $\ym$ in the clique $C$, and $\Z$ is a normalization constant ensuring that the right hand side is a proper \gls{PDF}. 

\Gls{SMC} methods have recently been successfully applied to the task of inference in general \glspl{PGM}, see \eg  \citet{maceachern1999sequential, chopin2002sequential,del2006sequential,ihler2009particle,naessethLS2014,paige2016inference,LindstenJNKSAB:2017,lindsten2018} for representative examples.


\paragraph{Probabilistic Programming}
Probabilistic programming languages are programming languages designed to describe probabilistic models and to automate the process of doing inference in those models. We can think of probabilistic programming as that of automating statistical inference, particularly Bayesian inference, using tools from computer science and computational statistics. Developing a syntax and semantics to describe and denote probabilistic models and the inference problems we are interested in solving is key to the probabilistic programming language. To define what separates a probabilistic programming language from a standard programming language we  quote \citet{gordon2014probabilistic}: ``Probabilistic programs are usual functional or imperative programs with two added constructs: (1) the ability to draw values at random from distributions, and (2) the ability to condition values of variables in a program via observations.'' This aligns very well with the notion of Bayesian inference through the posterior distribution \cref{eq:posterior}; through the syntax of the language we define $\xm$ to be the random values we sample and $\ym$ our observations that we condition on through the use of Bayes rule. The probabilistic program then essentially defines our joint probabilistic model $p(\xm,\ym)$. 

One of the main challenges of probabilistic programming is to develop algorithms that are general enough to enable inference for any model (probabilistic program) that we could conceivably write down using the language. Recently \citet{wood2014new} have shown that \gls{SMC}-based approaches 
can be used as inference back-ends in probabilistic programs.

For a more thorough treatment of probabilistic programming we refer the interested reader to the recent tutorial by \citet{van2018introduction} and the survey by \citet{gordon2014probabilistic}.

\paragraph{Bayesian nonparametric models}
Nonparametric models are characterized by having a complexity which grows with the amount of available data. In a Bayesian context this implies that the usual latent random variables (\ie, parameters) of the model are replaced by latent stochastic processes. 
Examples include Gaussian processes, Dirichlet processes, and Beta processes; see \eg \citet{HjortHMW:2010} for a general introduction.

Sampling from these latent stochastic processes, conditionally on observed data, can be done using \gls{SMC}. To give a concrete example, consider the Dirichlet process mixture model, which is a clustering model that can handle an unknown and conceptually infinite number of mixture components. 
Let $\y_\tm$, $\tm=1,2,\ldots$ be a stream of data. Let 
$\xp_\tm$, $\tm=1,2,\ldots$ be a sequence of latent integer-valued random variables, such that $\xp_\tm$ is the index of the mixture component to which datapoint $\y_\tm$ belongs. 
A generative representation of the mixture assignment variables is given by
\begin{align*}
	p(\xp_{\tm+1} = j \mid \xp_{1:\tm}) &=
	\begin{cases}
		\frac{n_{\tm,j}}{\tm+\alpha} & \text{for } j=1,\ldots,J_{\tm}, \\
		\frac{\alpha}{\tm+\alpha} & \text{for } j=J_{\tm}+1,
	\end{cases}
\end{align*}
where
\(
	J_\tm \eqdef \max\{\xp_{1:\tm}\}
\) is the number of distinct mixture components represented by the first $\tm$ datapoints, and $n_{\tm,j} \eqdef \sum_{\km=1}^\tm \mathbb{I}\{\xp_\km = j\}$
is the number of datapoints among $\y_{1:\tm}$ that belong to the $j$th component.

The model is completed by specifying the distribution of the data, conditionally on the mixture assignment variables:
\begin{align*}
	\theta_k &\sim F(\theta), \quad k=1,2,\ldots \\
	p(\y_\tm \mid \xp_\tm, \{\theta_k\}_{k \geq 1}) &= G(\y_\tm \mid \theta_{\xp_\tm}),
\end{align*}
where $G(\;\cdot \mid \theta)$ is an emission probability distribution parameterized by $\theta$ and 
$F(\theta)$ is a prior distribution over the mixture parameters $\theta$.

Note that the mixture assignment variables $\xp_\tm$, $\tm=1,2,\ldots$ evolve according to a latent stochastic process. Solving the clustering problem amounts to computing the posterior distribution of this stochastic process, conditionally on the observed data. One way to address this problem is to use \gls{SMC}; see \citet{Fearnhead:2004} for an efficient implementation tailored to the discrete nature of the problem.

\paragraph{Inference Evaluation}
An important problem when performing approximate Bayesian inference is to figure out when our approximation is “good enough”? Is it possible to give practical guarantees on the approximation we obtain? We need ways to evaluate how accurate our approximate inference algorithms are when compared to the true target distribution that we are trying to approximate. We will refer to the process of evaluating and validating approximate inference methods as \emph{inference evaluation}.

Inference evaluation is mainly concerned with measuring how close our approximation is to the true object we are trying to estimate, often a posterior distribution. For simulated data, \citet{grosse2015sandwiching,grosse2016reliability} have shown that we can make use of \gls{SMC} and \gls{IS} to bound the symmetrized \gls{KL} divergence between our approximate posterior and the true posterior. In another related work \citet{towner-nips-2017} have shown that \gls{SMC}-based methods show promise in estimating the symmetric \gls{KL} divergence between the approximate posterior and a gold standard algorithm.


\section{Example Code}\label{sec:examplecode}
We will be making use of inline Python code snippets throughout the 
manuscript to illustrate the algorithms and methods. Below we 
summarize the modules that are necessary to import to run the code 
snippets:
\begin{lstlisting}[caption={Necessary imports for Python code examples.},
label=code:imports]
import numpy as np
import numpy.random as npr
from scipy.misc import logsumexp
from scipy.stats import norm
\end{lstlisting}

\section{Outline}\label{sec:outline}
The remainder of this tutorial is organized as follows. In 
\cref{cha:smc}, we first introduce \gls{IS}, a foundational building 
block for \gls{SMC}. Then, we discuss the limitations of \gls{IS} and 
how \gls{SMC} resolves these. Finally, the section concludes with 
discussing some practical issues and theoretical results relevant to 
\gls{SMC} methods.

\cref{sec:proptwisting} is focused on the two key design choices of 
\gls{SMC}: the \emph{proposal} and \emph{target} distributions. 
Initially we focus on the proposal, discussing various ways of adapting 
and learning good proposals that will make the approximation more 
accurate. 
Then we discuss the sequence of target distributions; how we can learn 
intermediate distributions that help us when we try to approximate the 
posterior.

\cref{sec:pm} focuses on \gls{PM} methods and other \gls{SMC} methods that rely on a concept known as \emph{proper weights}. First, we provide a simple and straightforward proof of the unbiasedness property of the \gls{SMC} normalization constant estimate. Then, we describe and illustrate the combination of \gls{MCMC} and \gls{SMC} methods through \gls{PM} algorithms.  We move on to detail properly weighted \gls{SMC}, a concept that unites and extends the random weights and nested \gls{SMC} algorithms. Finally, we conclude the chapter by considering a few approaches for distributed and parallel \gls{SMC}.

In \cref{sec:csmc} we introduce \gls{CSMC}, and related methods for simulating from and computing expectations with respect to a target distribution. First, we introduce the basic \gls{CSMC} algorithm and provide a straightforward proof of the unbiasedness of the inverse normalization constant estimate. Then, we show how \gls{SMC} and \gls{CSMC} can be combined to leverage multi-core and parallel computing architectures in the \gls{IPMCMC} algorithm. Finally, we discuss recent applications of \gls{CSMC} for evaluating approximate inference methods.

The tutorial concludes with a discussion and outlook in  \cref{sec:discussion}.

\chapter{Importance Sampling to Sequential Monte Carlo}\label{cha:smc}
Typical applications require us to be able to evaluate or sample 
from the target distributions $\nmod_t$, as well as compute their 
normalization constants~$\Z_\tm$. For most models and targets this will 
be intractable, and we need approximations based on \eg Monte Carlo 
methods.

In this section, we first review \gls{IS} and some of its 
shortcomings. Then, we introduce the \gls{SMC} method, the key 
algorithm underpinning this monograph. Finally, we discuss some key 
theoretical properties of the \gls{SMC} algorithm.

\section{Importance Sampling}\label{sec:is}
\Acrlong{IS} \citep{kahn1950a,kahn1950b,robert2004monte} is a Monte Carlo method that constructs an approximation 
using samples from a proposal distribution, and corrects for the 
discrepancy between the target and proposal using (importance) weights.

Most applications of Bayesian inference can be formulated as computing 
the expectation of 
some generic function $\fun_\tm$, referred to as a test function, with respect to the target 
distribution $\nmod_\tm$,
\begin{align}
    \nmod_\tm(\fun_\tm) \eqdef \Exp_{\nmod_\tm}
    \left[\fun_\tm(\xp_{1:\tm})\right].
    \label{eq:is:expectation}
\end{align}
Examples include posterior predictive distributions, Bayesian 
p-values, and point estimates such as the posterior mean. Computing 
\cref{eq:is:expectation} is often intractable, but by a clever trick we can 
rewrite it as 
\begin{align}
    \Exp_{\nmod_\tm}\left[\fun_\tm(\xp_{1:\tm})\right] &= \frac{1}{\Z_\tm}
    \Exp_{\prop_\tm}\left[\frac{\umod_\tm(\xp_{1:\tm})}
    {\prop_\tm(\xp_{1:\tm})} \fun_\tm(\xp_{1:\tm})\right] 
    = \frac{\Exp_{\prop_\tm}\left[\frac{\umod_\tm(\xp_{1:\tm})}
    {\prop_\tm(\xp_{1:\tm})} \fun_\tm(\xp_{1:\tm})\right]}
    {\Exp_{\prop_\tm}\left[\frac{\umod_\tm(\xp_{1:\tm})}
    {\prop_\tm(\xp_{1:\tm})}\right]}.
    \label{eq:is:trick}
\end{align}
The \gls{PDF} $\prop_\tm$ is a user chosen proposal distribution, we 
assume it is simple to sample from and evaluate.
We can now estimate the right hand side of \cref{eq:is:trick} using the 
Monte Carlo method,
\begin{align}
    \Exp_{\nmod_\tm}\left[\fun_\tm(\xp_{1:\tm})\right] &\approx 
    \frac{\frac{1}{\Np}\sum_{\ip=1}^\Np \uwp_\tm(\xp_{1:\tm}^\ip)
    \fun_\tm(\xp_{1:\tm}^\ip)}
    {\frac{1}{\Np}\sum_{\jp=1}^\Np \uwp_\tm(\xp_{1:\tm}^\jp)},
    \label{eq:is:verbose_est}
\end{align}
where $\uwp_\tm(\xp_{1:\tm}) \eqdef 
\nicefrac{\umod_\tm(\xp_{1:\tm})}{\prop_\tm(\xp_{1:\tm})}$ and 
$\xp_{1:\tm}^\ip$ are simulated \iid from $\prop_\tm$. We will usually 
write \cref{eq:is:verbose_est} more compactly as
\begin{align}
    \Exp_{\nmod_\tm}\left[\fun_\tm(\xp_{1:\tm})\right] &\approx 
    \sum_{\ip=1}^\Np \nwp_\tm^\ip
    \fun_\tm(\xp_{1:\tm}^\ip), \quad \xp_{1:\tm}^\ip \iidsim \prop_\tm,
    \label{eq:is:est}
\end{align}
where the normalized weights $\nwp_\tm^\ip$ are defined by
\begin{align*}
    \nwp_\tm^\ip &\eqdef \frac{\uwp_\tm^\ip}{\sum_\jp 
    \uwp_\tm^\jp}, 
\end{align*}
with $\uwp_\tm^\ip$ a shorthand for $\uwp_\tm(\xp_{1:\tm}^\ip)$. The 
estimate in \cref{eq:is:est} is strongly consistent, 
converging (almost surely) to the true expectation as the number of 
samples~$\Np$ tend to infinity. An alternate view of \gls{IS} is to consider 
it an (empirical) approximation\footnote{This should be interpreted as an approximation of the underlying probability distribution, and not of the density function itself.} of $\nmod_\tm$,
\begin{align}
    \nmod_\tm(\xp_{1:\tm}) &\approx \sum_{\ip=1}^\Np \nwp_\tm^\ip 
    \dirac_{\xp_{1:\tm}^\ip}(\xp_{1:\tm}) \defeq \widehat 
    \nmod_{\tm}(\xp_{1:\tm}),
    \label{eq:is:nmodapprox}
\end{align}
where $\dirac_X$ denotes the Dirac measure at $X$.
Furthermore, \gls{IS} provides an approximation of the 
normalization constant,
\begin{align}
    \Z_\tm &\approx \frac{1}{\Np} \sum_{\ip=1}^\Np \uwp_\tm^\ip \defeq 
    \widehat\Z_\tm
    \label{eq:is:zhat}
\end{align}
Because the weights depend on the random samples, $\xp_{1:\tm}^\ip$, they are themselves 
random variables. One of the key properties is that it is 
\emph{unbiased},
This can be easily seen by noting that $\xp_{1:\tm}^\ip$ are \iid draws from $\prop_\tm$ and therefore
\begin{align}
	\Exp[\widehat\Z_\tm]
	= \frac{1}{\Np} \sum_{\ip=1}^\Np \Exp\left[\frac{\umod_\tm(\xp_{1:\tm}^\ip)}
	{\prop_\tm(\xp_{1:\tm}^\ip)}\right]
	=  \frac{1}{\Np} \sum_{\ip=1}^\Np \int \frac{\umod_\tm(\xp_{1:\tm}^\ip)}
	{\prop_\tm(\xp_{1:\tm}^\ip)} \prop_\tm(\xp_{1:\tm}^\ip) \dif \xp_{1:\tm}^\ip
	 = \Z_\tm,
\end{align}
since $\Z_\tm = \int \umod_\tm(\xp_{1:\tm})\dif \xp_{1:\tm}$ is nothing by the normalization constant for $\umod_\tm$.
%
This property
will be important for several of the more powerful \gls{IS} and 
\gls{SMC}-based methods considered in this monograph.

We summarize the \acrlong{IS} method in \cref{alg:is}. This algorithm 
is sometimes referred to as \emph{self-normalized} \gls{IS}, because 
we are normalizing each individual weight using all samples.
\begin{algorithm}[tb]
    \SetKwInOut{Input}{input}\SetKwInOut{Output}{output}
    \Input{Unnormalized target distribution $\umod_\tm$, proposal 
    $\prop_\tm$, number of samples $\Np$.}
    \BlankLine
    \For{$\ip = 1$ \KwTo $\Np$}{
        Sample $\xp_{1:\tm}^\ip \sim \prop_\tm$\\
        Set $\uwp_\tm^\ip = \frac{\umod_\tm(\xp_{1:\tm}^\ip)}
        {\prop_\tm(\xp_{1:\tm}^\ip)}$
    }
    Set $\nwp_\tm^\ip = \frac{\uwp_\tm^\ip}{\sum_\jp \uwp_\tm^\jp}$, 
    for $\ip =1,\ldots,\Np$
    
    \caption{\glsreset{IS}\Gls{IS}}\label{alg:is}
\end{algorithm}

The straightforward implementation of the \gls{IS} method we have 
described thus far is impractical for many of the example models and 
targets in \cref{sec:seqmodels}. It is 
challenging to design good proposals for high-dimensional models. A 
good proposal is typically more heavy-tailed than the target; 
if it is not, the weights can have infinite variance. Another favorable 
property of a proposal is that it should cover the bulk of the target 
probability mass, putting high probability on regions of high 
probability under the target distribution.
Even Markovian models, such as the \gls{SSM}, can have a prohibitive computational complexity 
without careful design of the proposal. In the next section we will 
describe how we can alleviate these concerns using \gls{SIS}, a 
special case of \gls{IS}, with a kind of divide-and-conquer approach to 
tackle the high-dimensionality in $\Tm$.

\subsection{Sequential Importance Sampling}\label{sec:sis}
\Acrlong{SIS} \citep{handschin1969monte,robert2004monte} is a variant of \gls{IS} were we select a proposal 
distribution that has an autoregressive structure, and compute 
importance weights recursively. By choosing a proposal defined by
\begin{align*}
    \prop_\tm(\xp_{1:\tm}) &= \prop_{\tm-1}(\xp_{1:\tm-1}) 
    \prop_\tm(\xp_\tm \given \xp_{1:\tm-1})
\end{align*}
we can decompose the proposal design problem into $\Tm$ conditional 
distributions. This means we obtain samples $\xp_{1:\tm}^\ip$ by 
reusing $\xp_{1:\tm-1}^\ip$ from the previous iteration, and append a 
new sample, $\xp_\tm^\ip$, simulated from $\prop_\tm(\xp_\tm \given 
\xp_{1:\tm-1}^\ip)$. The unnormalized weights can be computed 
recursively by noting that
\begin{align*}
    \uwp_\tm(\xp_{1:\tm}) &= 
    \frac{\umod_\tm(\xp_{1:\tm})}{\prop_\tm(\xp_{1:\tm})} =
    \frac{\umod_{\tm-1}(\xp_{1:\tm-1})}{\prop_{\tm-1}(\xp_{1:\tm-1})} 
    \frac{\umod_\tm(\xp_{1:\tm})}{\umod_{\tm-1}(\xp_{1:\tm-1}) 
    \prop_\tm(\xp_\tm \given \xp_{1:\tm-1})} \\
    &=\uwp_{\tm-1}(\xp_{1:\tm-1}) 
    \frac{\umod_\tm(\xp_{1:\tm})}{\umod_{\tm-1}(\xp_{1:\tm-1}) 
    \prop_\tm(\xp_\tm \given \xp_{1:\tm-1})}.
\end{align*}
We summarize the \gls{SIS} method in \cref{alg:sis}, where 
${\prop_1(\xp_1\given\xp_{1:0}) = \prop_1(\xp_1)}$ and $\uwp_0 = 
\umod_0 = 1$.
\begin{algorithm}[tb]
    \SetKwInOut{Input}{input}\SetKwInOut{Output}{output}
    \Input{Unnormalized target distributions $\umod_\tm$, proposals 
    $\prop_\tm$, number of samples $\Np$.}
    \BlankLine
    \For{$\tm=1$ \KwTo $\Tm$}{
        \For{$\ip = 1$ \KwTo $\Np$}{
            Sample $\xp_{\tm}^\ip \sim 
            \prop_\tm(\xp_\tm\given\xp_{1:\tm-1}^\ip)$\\
            Append $\xp_{1:\tm}^\ip = 
            \left(\xp_{1:\tm-1}^\ip,\xp_\tm^\ip\right)$ \\
            Set $\uwp_\tm^\ip = \uwp_{\tm-1}^\ip \frac{\umod_\tm(\xp_{1:\tm}^\ip)}
            {\umod_{\tm-1}(\xp_{1:\tm-1}^\ip) \prop_\tm(\xp_{\tm}^\ip\given \xp_{1:\tm-1}^\ip)}$
        }
        Set $\nwp_\tm^\ip = \frac{\uwp_\tm^\ip}{\sum_\jp \uwp_\tm^\jp}$, 
        for $\ip =1,\ldots,\Np$
    }
    \caption{\glsreset{SIS}\Gls{SIS}}\label{alg:sis}
\end{algorithm}

If we need to evaluate the normalization constant estimate 
$\hatZ_\tm$, analogously to \gls{IS} we make use of 
\cref{eq:is:zhat}. However, we may also obtain a (strongly) 
consistent estimate of the ratio of normalization constants
$\nicefrac{\Z_\tm}{\Z_{\tm-1}}$
\begin{align*}
\frac{\Z_\tm}{\Z_{\tm-1}} &= 
\Exp_{\nmod_\tm (\xp_{1:\tm-1})\prop_\tm(\xp_\tm\given\xp_{1:\tm-1})}
\left[\frac{\umod_\tm(\xp_{1:\tm})}{\umod_{\tm-1}(\xp_{1:\tm-1})
\prop_\tm(\xp_\tm\given \xp_{1:\tm-1})}\right] \\
&\approx \sum_{\ip=1}^\Np \nwp_{\tm-1}^\ip 
\frac{\uwp_\tm^\ip}{\uwp_{\tm-1}^\ip}.
\end{align*}
While the estimate of the \emph{ratio} is consistent, it is in general not unbiased. 
However, \gls{SIS} is a special case of \gls{IS}. This means that the \gls{SIS} 
estimate of the normalization constant for $\umod_\tm$, \ie $\hatZ_\tm$ in 
\cref{eq:is:zhat}, is still unbiased and consistent.

In \cref{ex:sis} we detail a first example proposal $\prop_\tm$ for 
the running example, and derive the corresponding weights $\uwp_\tm$.
Furthermore, we include a code snippet that illustrates how to 
implement the sampler in Python.

\begin{example}[\Acrlong{SIS} for \cref{ex:running}]
We revisit our running non-Markovian Gaussian example. The target 
distribution is
\begin{align*}
    \umod_\tm(\xp_{1:\tm}) &= p(\xp_1)\gmod(\y_1\given\xp_1)\prod_{\km=2}^\tm 
    \fmod(\xp_\km\given\xp_{\km-1}) \gmod(\y_\km\given\xp_{1:\km}),
\end{align*}
with $p(\xp_1) = \Norm(\xp_1\given 0, q)$ and 
\begin{align*}
    \fmod(\xp_\km\given\xp_{\km-1}) &= 
    \Norm\left(\xp_\km \given \phi \xp_{\km-1}, q \right), 
    &&\gmod(\y_\km\given\xp_{1:\km}) = \Norm\left(\y_\tm\given
        \sum_{\lm=1}^\km \beta^{\km-\lm} \xp_{\lm}, r\right).
\end{align*}
A common approach is to set the proposal to 
be the prior (or transition) distribution $\fmod$. A sample from the 
proposal $\prop_\tm(\xp_\tm\given\xp_{\tm-1}) = 
\fmod(\xp_\tm\given\xp_{\tm-1})$ is generated as follows
\begin{align}
    \xp_\tm &= \phi \xp_{\tm-1} + \sqrt{q} \varepsilon_\tm, 
    \quad \varepsilon_\tm \sim \Norm(0,1).
    \label{eq:sis:running:bootstrap_prop}
\end{align}
We refer to this proposal simply as the \emph{prior} 
proposal. The corresponding weight update is
\begin{align}
	\notag
    \uwp_\tm(\xp_{1:\tm}) &= \uwp_{\tm-1}(\xp_{1:\tm-1}) 
    \frac{\umod_\tm(\xp_{1:\tm})}{\umod_{\tm-1}(\xp_{1:\tm-1}) 
    \prop_\tm(\xp_\tm \given \xp_{1:\tm-1})} \\
    &= \uwp_{\tm-1}(\xp_{1:\tm-1}) \Norm\left(\y_\tm\given
        \sum_{\km=1}^\tm \beta^{\tm-\km} \xp_{\km}, r\right),
\end{align}
where $\uwp_0 = 1$. We provide \cref{code:sis:running} to illustrate 
how to implement \gls{SIS} with the prior proposal for this model in 
Python.
\begin{lstlisting}[caption={\Acrlong{SIS} for \cref{ex:running}.},
label=code:sis:running]
x = np.zeros((N,T))
logw = np.zeros(N)
mu = np.zeros(N)
for t in range(T):
    x[:,t]= phi*x[:,t-1]+np.sqrt(q)*npr.randn(N)
    mu = beta*mu + x[:,t]
    logw += norm.logpdf(y[t], mu, np.sqrt(r))
w = np.exp(logw-logsumexp(logw))
\end{lstlisting}
For improved numerical stability we update the log-weights $\log 
\uwp_\tm$ and subtract the logarithm of the sum of weights (the weight normalization) before exponentiating.
\label{ex:sis}
\end{example}
\Gls{SIS} can be implemented efficiently for a large class 
of problems, the computational complexity is usually linear in $\Np$ 
and $\Tm$. Even so, the \gls{IS} methods suffer from severe drawbacks limiting 
their practical use for many high-dimensional problems.

\subsection{Shortcomings of Importance Sampling}
The main drawback of \gls{IS} is that the variance of the estimator 
scales unfavorably with the dimension of the problem; the 
variance generally increases exponentially in $\Tm$ \citep{doucet2009tutorial}. Because 
\gls{SIS} is a special case of \gls{IS} it inherits this 
unfavorable property. To see that this is true we can illustrate using a toy example where the target factorizes completely over time and with identical distributions for each factor $\umod_\tm(\xp_{1:\tm}) = \prod_{\km=1}^\tm \umod(\xp_\km)$. We illustrate in \cref{ex:sis:exponential}.
\begin{example}[\Acrlong{SIS} variance of $\hatZ_\tm$]
Assume that the target distribution of interest is given by $\umod_\tm(\xp_{1:\tm}) = \prod_{\km=1}^\tm \umod(\xp_\km)$, and we pick a proposal distribution that follows the same factorization  $\prop_\tm(\xp_{1:\tm}) = \prod_{\km=1}^\tm \prop(\xp_\km)$. The normalization constant is then given by $\Z_\tm= \Z^\tm$ for $\Z = \int \umod(\xp) \dif \xp$. The normalization constant estimate is
\begin{align*}
	\hatZ_\tm &= \frac{1}{\Np}\sum_{\ip=1}^\Np \prod_{\km=1}^\tm  \frac{\umod(\xp_\km^\ip)}{\prop(\xp_\km^\ip)},
\end{align*}
and its corresponding variance, normalized by $\Z_\tm$, is
\begin{align*}
	\Var\left(\frac{\hatZ_\tm}{\Z_\tm}\right) &= \frac{1}{\Np} \frac{\Var\left( \prod_{\km=1}^\tm \frac{\umod(\xp_\km^\ip)}{\prop(\xp_\km^\ip)} \right)}{\Z_\tm^2} = \frac{1}{\Np} \left(\Exp_{\prop(\xp)}\left[\frac{\nmod(\xp)^2}{\prop(\xp)^2}\right]\right)^\tm - \frac{1}{\Np},
\end{align*}
where $\Exp_{\prop(\xp)}\left[\frac{\nmod(\xp)^2}{\prop(\xp)^2}\right] = \Exp_{\prop(\xp_\km)}\left[\frac{\nmod(\xp_\km)^2}{\prop(\xp_\km)^2}\right]$ for all $\km$. This means that if the proposal is not exactly equal to the target the variance of the normalization constant estimate will in general increase exponentially with $\tm$.
\label{ex:sis:exponential}
\end{example}

Another way in which the unfavorable scaling in $\tm$ manifests itself in practice is through 
the normalized weights $\nwp_\tm^\ip$. The maximum of the weights, 
$\max_\ip \nwp_\tm^\ip$, will quickly approach one as $\tm$ increases, and consequently all other weights approach zero; 
a phenomena known as \emph{weight degeneracy}. 
This means that, effectively, we approximate the target distribution 
using a single sample.

We illustrate weight degeneracy in \cref{ex:sis:weightdegeneracy} using the 
running example.
\begin{example}[\GLS{SIS} weight degeneracy]
We return to our running example, set the length $\Tm = 6$, number of 
particles $\Np = 5$, and $(\phi, q, \beta, r) = (0.9, 1.0, 0.5, 1.0)$. 
\cref{fig:sisweightdegeneracy} shows the particles and the normalized 
weights $\nwp_\tm^\ip$, where the area of the discs correspond to the 
size of the weights. We can see that as $\tm$ increases nearly all mass concentrates 
on the fourth particle $\xp_\tm^4$. This means that the normalized 
weight of the particle is almost one, $\nwp_\tm^4 \approx 1$. The remaining 
particles have normalized weights that are all close to zero, and thus 
have a negligible contribution to the approximation.
\begin{figure}[tb]
    \centering
    \includegraphics[width=0.75\columnwidth]{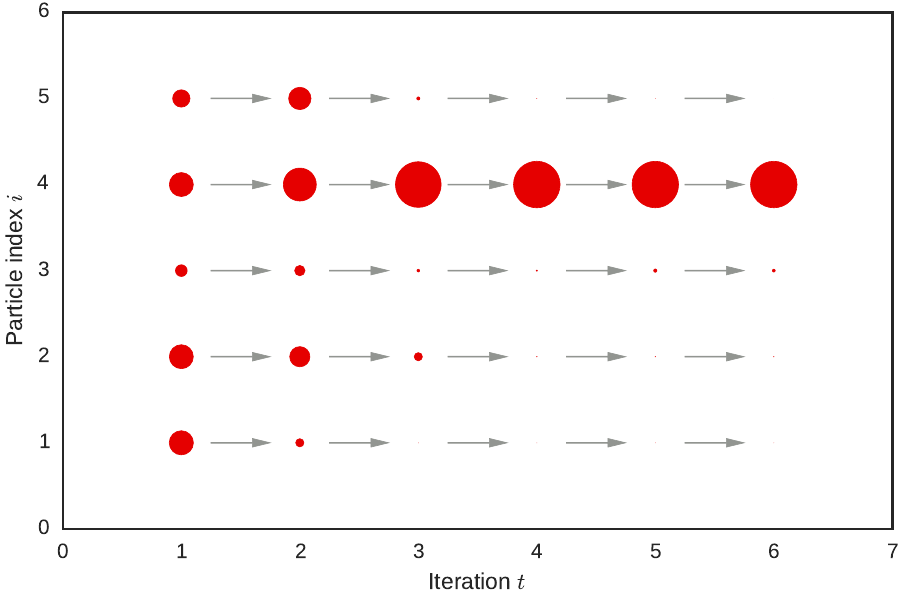}
    \caption{Weight degeneracy of the \gls{SIS} method. The size of the 
    disks represent the size of the corresponding weights 
    $\nwp_\tm^\ip$.}\label{fig:sisweightdegeneracy}
\end{figure}
This concentration of mass for \gls{SIS} to a single particle happens 
very quickly. 
Even for very simple Markovian models 
the variance of \eg our normalization constant estimator can increase 
exponentially fast as a function of $\Tm$.
\label{ex:sis:weightdegeneracy}
\end{example}

\Acrlong{SMC} methods tackle the weight degeneracy issue by 
choosing a proposal that leverages information contained in 
$\widehat\nmod_{\tm-1}$, the previous iteration's 
target distribution approximation.

\section{Sequential Monte Carlo}\label{sec:smc}
\Acrlong{SMC} methods \citep{GordonSS:1993,stewart1992,Kitagawa:1993} improve upon 
\gls{IS} by mitigating the weight degeneracy issue through a clever 
choice of proposal distribution. For certain sequence models the 
weight degeneracy issue can be resolved 
altogether, providing estimators to the final marginal distribution 
$\nmod_\Tm(\xp_\Tm)$ that do not deteriorate for increasing $\Tm$. 
For other sequence models, \gls{SMC} still tends to 
provide more accurate estimates in practice compared to \gls{IS}.

Just like in \gls{SIS} we need a sequence of proposal distributions 
$\prop_\tm(\xp_\tm\given\xp_{1:\tm-1})$ for $\tm=1,\ldots,\Tm$. 
This is a user choice that can significantly impact the 
accuracy of the \gls{SMC} approximation. For now we assume that the 
proposal is given and return to this choice in \cref{sec:proptwisting}. 
Below, we detail the iterations (or steps) of a basic \gls{SMC} 
algorithm. 

\paragraph{Step 1:}
The first iteration of \gls{SMC} boils down to approximating the 
target distribution $\nmod_1$ using standard \gls{IS}. Simulating $\Np$ 
times independently from the first proposal
\begin{align}
    \xp_1^\ip &\iidsim \prop_1(\xp_1), \quad \ip = 1,\ldots, \Np,
    \label{eq:smc:firstprop}
\end{align} 
and assigning corresponding weights
\begin{align}
    &\uwp_1^\ip = \frac{\umod_1(\xp_1^\ip)}{\prop_1(\xp_1^\ip)},
    \quad \nwp_1^\ip = \frac{\uwp_1^\ip}{\sum_{\jp=1}^\Np \uwp_1^\jp},
    \quad \ip = 1,\ldots, \Np,
    \label{eq:smc:firstweight}
\end{align}
lets us approximate $\nmod_1$ (\cf \cref{eq:is:nmodapprox}) by
\begin{align}
    \widehat \nmod_1(\xp_1) &= \sum_{\ip=1}^\Np \nwp_1^\ip 
    \dirac_{\xp_1^\ip}(\xp_1).
    \label{eq:smc:approx1}
\end{align}
The key innovation of the \gls{SMC} algorithm is that it takes 
advantage of the information provided in $\widehat \nmod_1(\xp_1)$, 
\cref{eq:smc:approx1}, when 
constructing a proposal for the next target distribution $\nmod_2$. 

\paragraph{Step 2:} 
In the second iteration of \gls{SMC} we sample $\xp_{1:2}$ from the 
proposal
$\widehat\nmod_1(\xp_1)\prop_2(\xp_2\given\xp_1)$, rather than from 
$\prop_1(\xp_1)\prop_2(\xp_2\given\xp_1)$ like \gls{SIS}.
We sample $\Np$ times independently from
\begin{align}
    \xp_{1:2}^\ip \iidsim  \widehat\nmod_1(\xp_1)
    \prop_2(\xp_2\given\xp_1), \quad \ip = 1,\ldots, \Np,
    \label{eq:smc:secondprop}
\end{align}
and assign weights
\begin{align}
    \uwp_2^\ip &= 
    \frac{\umod_2(\xp_{1:2}^\ip)}{\umod_1(\xp_{1}^\ip)\prop_2(\xp_2^\ip\given\xp_1^\ip)},
    \quad \nwp_2^\ip = \frac{\uwp_2^\ip}{\sum_{\jp=1}^\Np \uwp_2^\jp},
    \quad \ip = 1,\ldots, \Np.
    \label{eq:smc:secondweight}
\end{align}
Simulating $\xp_{1:2}^\ip$, \cref{eq:smc:secondprop}, can be broken 
down into parts: \emph{resampling} 
$\xp_1^\ip\sim\widehat\nmod_1(\xp_1)$, \emph{propagation} 
$\xp_2^\ip\given \xp_1^\ip \sim \prop_2(\xp_2\given\xp_1^\ip)$, and 
concatenation ${\xp_{1:2}^\ip = (\xp_1^\ip, \xp_2^\ip)}$. Note the 
overloaded notation for $\xp_1^\ip$. We replace the 
initial sample set $\{\xp_1^\ip\}_{\ip=1}^\Np$ from Step~1, 
with the resampled set $\xp_1^\ip\sim\widehat\nmod_1(\xp_1)$, 
$\ip=1,\ldots,\Np$.

Resampling can refer to a variety of methods in statistics, for our 
purpose it is simple (weighted) random sampling with replacement from 
$\xp_1^{1:N} = \{\xp_1^\ip\}_{\ip=1}^\Np$ with weights 
$\nwp_1^{1:N}=\{\nwp_1^\ip\}_{\ip=1}^\Np$. Resampling $\Np$ times 
independently means that the number of times each particle is 
selected is multinomially distributed. This resampling algorithm is 
known as \emph{multinomial resampling}, 
see~\cref{code:multinomial_resampling}. In \cref{sec:resampling,sec:ess} 
we revisit resampling and present alternative resampling algorithms, 
increasing efficiency by correlation and adaptation.
\begin{lstlisting}[caption={Sampling $\Np$ times independently from 
$\sum_{i} \nwp^i \dirac_{\xp^i}$.}, 
label=code:multinomial_resampling]
def multinomial_resampling(w, x):
    u = npr.rand(*w.shape)
    bins = np.cumsum(w)
    return x[np.digitize(u,bins)]
\end{lstlisting}

Propagation generates new samples independently from the proposal, \ie 
$\xp_2^\ip \sim \prop_2(\xp_2\given \xp_1^\ip)$ for each $\ip=1,\ldots,\Np$. 
By concatenating $\xp_{1:2}^\ip = (\xp_1^\ip, \xp_2^\ip)$ we obtain a 
complete sample from the proposal 
$\widehat\nmod_1(\xp_1)\prop_2(\xp_2\given\xp_1)$.

Finally, new importance weights are computed according to \eqref{eq:smc:secondweight}. This expression can be understood as a standard importance sampling weight---the target divided by the proposal. Note, however, that we use $\umod_1$ in the denominator: when computing the weights we ``pretend'' that the proposal is given by $\nmod_1(\xp_1)\prop_2(\xp_2\given\xp_1)$ rather than by its approximation $\widehat\nmod_1(\xp_1)\prop_2(\xp_2\given\xp_1)$. This is of course just an interpretation and the motivation for using this particular weight expression is given by the convergence analysis for \gls{SMC}; see \cref{sec:analysisconvergence}.

The resulting approximation of $\nmod_2$ is
\begin{align*}
    \widehat\nmod_2(\xp_{1:2}) &= \sum_{\ip=1}^\Np \nwp_2^\ip 
    \dirac_{\xp_{1:2}^\ip}(\xp_{1:2}).
\end{align*}
\gls{SMC} can in essence be described as a synthesis of \gls{SIS} and 
resampling, which explains its alternate names \gls{SIR} or \gls{SISR}.

\paragraph{Step $\tm$:} The remaining iterations follow the recipe 
outlined in step $2$. First, the proposal is the product of the 
previous empirical distribution approximation and a conditional 
distribution  
\begin{align}
    \prop_\tm(\xp_{1:\tm}) &= \widehat\nmod_{\tm-1}(\xp_{1:\tm-1}) 
    \prop_\tm(\xp_\tm\given\xp_{1:\tm-1}).
    \label{eq:smc:proposal}
\end{align}
Samples for $\ip=1,\ldots, \Np$ are generated as follows
\begin{align*}
    & &  & \textit{resample} 
    & & \xp_{1:\tm-1}^\ip\sim\widehat\nmod_{\tm-1}(\xp_{1:\tm-1}), \\
    & &  & \textit{propagate} 
    & & \xp_\tm^\ip\given\xp_{1:\tm-1}^\ip \sim 
    \prop_\tm(\xp_\tm\given\xp_{1:\tm-1}^\ip), \\
    & &  & \textit{concatenate} 
    & & \xp_{1:\tm}^\ip = \left(\xp_{1:\tm-1}^\ip,\xp_\tm^\ip\right). 
\end{align*}
Finally, we assign the weights
\begin{align}
    \uwp_\tm^\ip &= 
    \frac{\umod_\tm(\xp_{1:\tm}^\ip)}{\umod_{\tm-1}(\xp_{1:\tm-1}^\ip)
    \prop_\tm(\xp_\tm^\ip\given\xp_{1:\tm-1}^\ip)},
    \quad \nwp_\tm^\ip = \frac{\uwp_\tm^\ip}{\sum_{\jp=1}^\Np \uwp_\tm^\jp},
    \quad \ip = 1,\ldots, \Np,
    \label{eq:smc:weights}
\end{align}
and approximate $\nmod_\tm$ by
\begin{align}
    \widehat\nmod_\tm(\xp_{1:\tm}) &= \sum_{\ip=1}^\Np \nwp_\tm^\ip 
    \dirac_{\xp_{1:\tm}^\ip}(\xp_{1:\tm}).
    \label{eq:smc:target_approx}
\end{align}
The normalization constant $\Z_\tm$ can be estimated by
\begin{align}
    \widehat \Z_\tm &= \prod_{\km=1}^\tm \frac{1}{\Np} 
    \sum_{\ip=1}^\Np \uwp_\km^\ip.
    \label{eq:smc:zhat}
\end{align}

We summarize the full \acrlong{SMC} sampler in \cref{alg:esmc:smc}, where 
$\widehat\nmod_0 = \umod_0 = 1$, and 
${\prop_1(\xp_1\given\xp_{1:0}) = \prop_1(\xp_1)}$.
\begin{algorithm}[tb]
    \SetKwInOut{Input}{input}\SetKwInOut{Output}{output}
    \Input{Unnormalized target distributions $\umod_\tm$, proposals 
    $\prop_\tm$, number of samples $\Np$.}
    \BlankLine
    \For{$\tm=1$ \KwTo $\Tm$}{
        \For{$\ip = 1$ \KwTo $\Np$}{
            Sample $\xp_{1:\tm}^\ip \sim 
            \widehat\nmod_{\tm-1}(\xp_{1:\tm-1})
            \prop_\tm(\xp_\tm\given\xp_{1:\tm-1})$ 
            (see \cref{eq:smc:proposal})\\
            Set $\uwp_\tm^\ip = \frac{\umod_\tm(\xp_{1:\tm}^\ip)}
            {\umod_{\tm-1}(\xp_{1:\tm-1}^\ip) \prop_\tm(\xp_{\tm}^\ip\given \xp_{1:\tm-1}^\ip)}$
            (see \cref{eq:smc:weights})
        }
        Set $\nwp_\tm^\ip = \frac{\uwp_\tm^\ip}{\sum_\jp \uwp_\tm^\jp}$, 
        for $\ip =1,\ldots,\Np$
    }
    \caption{\glsreset{SMC}\Gls{SMC}}\label{alg:esmc:smc}
\end{algorithm}

The \gls{SMC} method typically achieves drastic improvements compared 
to \gls{SIS}. In \cref{ex:basic_smc} we return to our running example, 
using the same settings as in \cref{ex:sis:weightdegeneracy}, to study 
the sample diversity and quality of the basic \gls{SMC} method 
compared to \gls{SIS}.

\begin{example}[\GLS{SMC} sample diversity]
We illustrate the weights and resampling dependencies in \cref{fig:smc:weightdegeneracy}. 
The grey arrows represent what sample from iteration $\tm-1$ that generated 
the current sample at iteration $\tm$, referred to as its \emph{ancestor}.
\begin{figure}[tb]
    \centering
    \includegraphics[width=0.75\columnwidth]{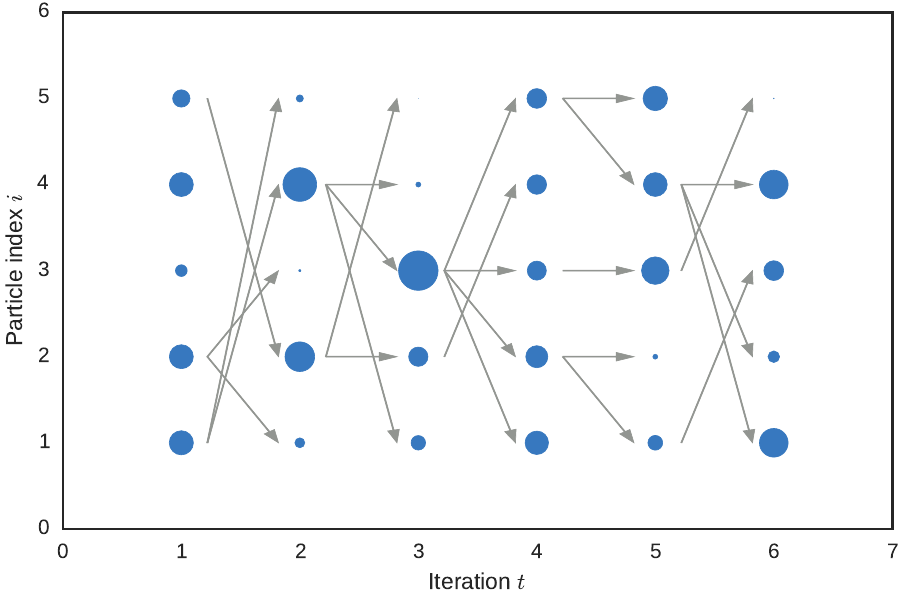}
    \caption{Diversity of samples in the \gls{SMC} method. The size of the 
    disks represent the size of the weights $\nwp_\tm^\ip$, and the 
    grey arrows represent resampling.}
    \label{fig:smc:weightdegeneracy}
\end{figure}
We can see that the weights tend to be more evenly distributed for 
\gls{SMC}. The algorithm dynamically chooses to focus computational 
effort on more promising samples through the resampling step. 
Particles with low weights tend to not be resampled, and particles 
with high weights are resampled more frequently.

\begin{table}[h]
    \centering
    \begin{tabular}{c | c | c | c}
    & $\Exp_{\widehat\nmod_{10}}[\nicefrac{\log \umod_{10}}{10}]$ & 
    $\Exp_{\widehat\nmod_{20}}[\nicefrac{\log \umod_{20}}{20}]$  &
    $\Exp_{\widehat\nmod_{40}}[\nicefrac{\log \umod_{40}}{40}]$  \\
    \hline
    \gls{SIS} & $-2.76$ & $-3.35$ & $-9.86$ \\
    \gls{SMC} & $\mathbf{-2.47}$ & $\mathbf{-2.51}$ & $\mathbf{-2.77}$ 
    \end{tabular}
    \caption{Average log-probability values of the unnormalized target 
    distribution with respect to the sampling distributions of 
    \gls{SIS} and \gls{SMC}. The number of particles $\Np=10$ is fixed 
    for both methods.}\label{tab:avg_logpdfvals}
\end{table}
Not only do we get more diversity in our sample set, \gls{SMC} also 
tends to find areas of higher probability. We illustrate this 
phenomenon in \cref{tab:avg_logpdfvals}. We fix the number of 
particles $\Np=10$, then study the 
average log-probability of our target distribution $\log \umod_\Tm$, 
under the sampling distributions $\widehat \nmod_\Tm$,
normalized by the number of iterations $\Tm$. 
\label{ex:basic_smc}
\end{example}

While \gls{SMC} methods do not suffer from weight degeneracy as $\tm$ increases,
the fundamental difficulty of simulating from a high-dimensional target still remains. For \gls{SMC} this instead manifests itself in
another form of weight degeneracy and something known as \emph{path degeneracy}. We discuss weight degeneracy below and 
illustrate path degeneracy in \cref{ex:pathdegen_smc}.

\paragraph{Weight degeneracy in \GLS{SMC}:} While introducing resampling can alleviate the weight degeneracy stemming from \gls{SIS}'s unfavorable scaling in $\tm$, \gls{SMC} can still suffer from weight degeneracy due to a high-dimensional $\xp_\tm$. If at each iteration of \cref{alg:esmc:smc} the maximum of the normalized importance weight $\max_\ip \nwp_\ip$ is close to one, and the remaining close to zero, we say that the \gls{SMC} sampler suffers from weight degeneracy. This occurs when the variance of the unnormalized importance weights $\uwp_\tm^\ip$ is high. \citet{snyder2008obstacles,bickel2008,bengtsson2008,snyder2015} showed that to avoid weight degeneracy the number of particles $\Np$ must scale at least exponentially with the variance of the log-weights, which in simple examples is proportional to the dimension $\dimX$ of the latent variable $\xp_\tm$. These results rely on a Gaussian target, are only valid as both $\Np$ and $\dimX$ tend to infinity, and focus on the variability of the log-weights. \citet[Proposition 3]{naesseth2016high} showed that asymptotic variance of the resulting estimator of simple test functions can also scale exponentially in the dimension $\dimX$, relying on asymptotics only in $\Np$ and not $\dimX$. We return to this example as we discuss the \gls{CLT} in \cref{sec:analysisconvergence}.

The weight degeneracy in \gls{SMC} caused by high-dimensional latent variables is still an open research problem, see \eg \citet{djuric2013particle,briggs2013data,beskos2014stability,rebeschini2015can,naesseth2016high} and the references therein. In this tutorial we focus on how choosing good proposals and target distributions can ameliorate, but not completely solve, this important issue. 

\begin{example}[\GLS{SMC} path degeneracy]
In \cref{fig:smc:pathdegeneracy} we reiterate the result from our 
previous example, \cref{ex:basic_smc}. However, this time we only 
include the arrows corresponding to samples that have been 
consistently resampled and form our final approximation 
$\widehat\nmod_\Tm$. 
\begin{figure}[tb]
    \centering
    \includegraphics[width=0.75\columnwidth]{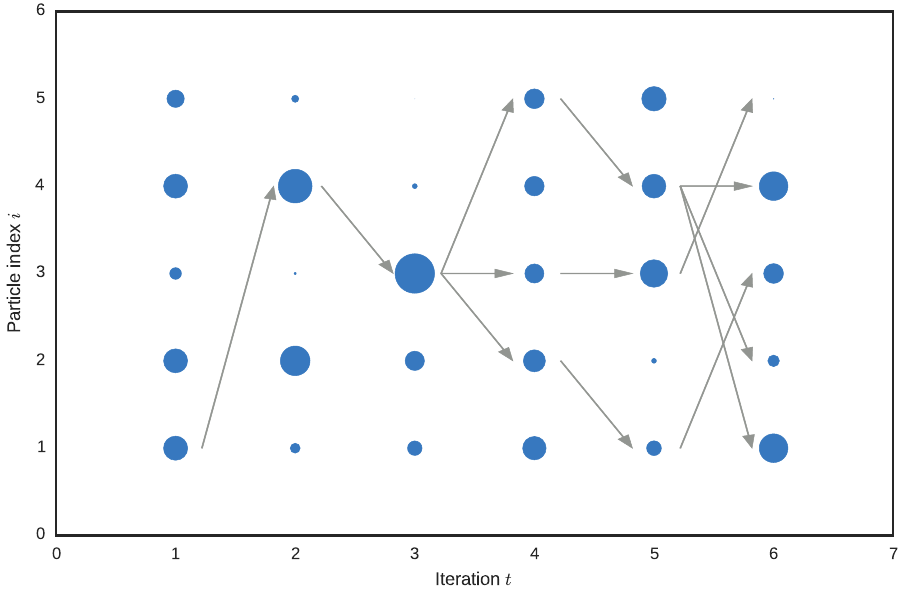}
    \caption{Path degeneracy of the \gls{SMC} method for smoothing 
    approximation. The size of the 
    disks represent the size of the weights $\nwp_\tm^\ip$, and the 
    grey arrows represent resampling.}
    \label{fig:smc:pathdegeneracy}
\end{figure}
We can see that our approximation for early iterations collapses back 
to a single ancestor, \eg we have $N=5$ identical copies of $\xp_1^1$ in 
\cref{fig:smc:pathdegeneracy}. This phenomena is known as path 
degeneracy and occurs because of the resampling mechanism. In 
\citet{jacob2015path} the authors study the impact this has on the 
memory requirements. They show that for state space models, under 
suitable conditions on the observation \gls{PDF} $\gmod$, the expected distance from the current 
iteration to a coalescence of the paths is bounded from above by 
$\ordo(\Np \log \Np)$.

In \cref{fig:smc:pathdegeneracy_logz} we study the impact that 
increasing dependence on earlier iterations has on our \gls{SMC} 
estimate of the log-normalization constant. We let 
$N=20$, $\Tm=100$ be fixed and vary the value of $\beta \in (0,1)$, 
where increasing values of $\beta$ correspond to more long-range 
dependencies in our non-Markovian \gls{LVM}. We can see that for modest 
values of $\beta$ the \gls{SMC} method achieves accurate estimates, 
whereas for higher values the drop-off in efficiency is significant.
\begin{figure}[tb]
    \centering
    \includegraphics[width=0.75\columnwidth]{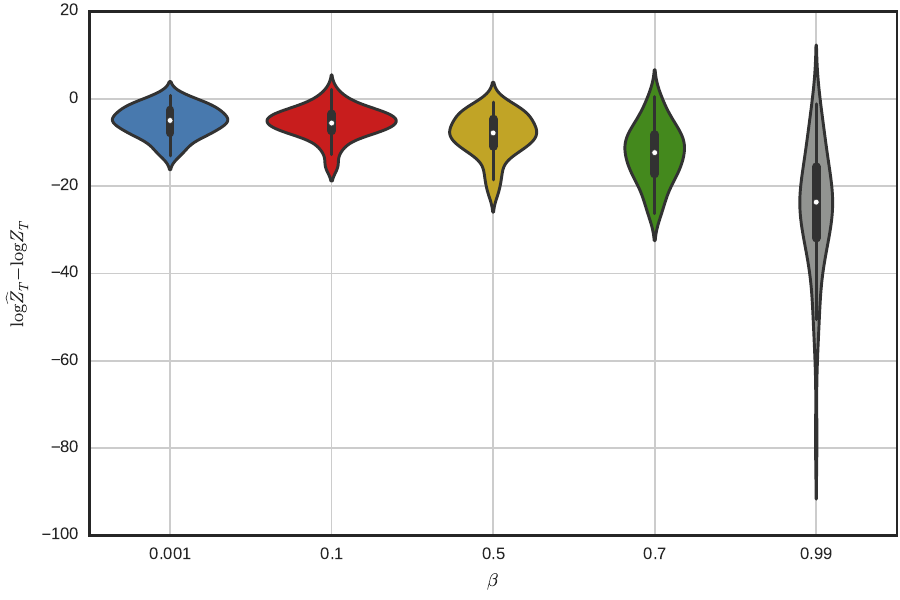}
    \caption{Violinplots for the log of the \gls{SMC} estimate of the 
    normalization constant divided by the true value, \ie
    $\log \widehat\Z_\Tm -\log\Z_\Tm$. The number of particles are 
    $N=20$, length of the data is $T=100$, and we study five different 
    settings of $\beta =(0.001, 0.1, 0.5, 0.7, 0.99)$.}
    \label{fig:smc:pathdegeneracy_logz}
\end{figure}
This manifests itself as an increase in the \gls{MC} variance in the 
estimate (width of bars), as well as a negative bias. As we will discuss 
in \cref{sec:analysisconvergence}, this negative bias in the estimate of 
$\log \Z_\Tm$ is typical for \gls{SMC} and \gls{IS}.
\label{ex:pathdegen_smc}
\end{example}

In \cref{sec:resampling,sec:ess} we will discuss two standard practical
approaches that help alleviate (but not solve) the issue of path degeneracy and improve 
overall estimation accuracy. The first is low-variance resampling --- 
we lower the variance in the resampling step by correlating random 
variables. The second is adaptive resampling, which essentially means 
that we do not resample at each iteration but only on an as-needed basis.
In \cref{sec:proptwisting} we go a step further and show how to 
choose, or learn, proposal and target distributions that lead to even 
further improvements. 

\subsection{Low-variance Resampling}\label{sec:resampling}
The resampling step introduces extra variance in the short-term, using 
$\widehat\nmod_{\tm-1}$ to estimate $\nmod_{\tm-1}$ is better than 
using the resampled particle set. However, discarding improbable particles, and 
putting more emphasis on promising ones is crucial to the long-term 
performance of \gls{SMC}.
To keep long-term performance but minimize the added variance, we can employ 
standard variance reduction techniques based on correlating samples 
drawn from $\widehat\nmod_{\tm-1}$. First, we explain a common 
technique for implementing multinomial resampling based on inverse 
transform sampling. Then, we explain two low-variance 
resampling alternatives, stratified and systematic resampling.

To sample from $\widehat\nmod_{\tm-1}$ we use inverse transform 
sampling based on the weights $\nwp_{\tm-1}^\ip$. We have that
\begin{align*}
\xp_{1:\tm-1} &\sim \widehat\nmod_{\tm-1} \Leftrightarrow 
 \xp_{1:\tm-1} = \xp_{1:\tm-1}^{\ap},
\end{align*}
where 
$\ap$ is an integer random variable 
on $\{1,\ldots,\Np\}$, such that the following is true
\begin{align}
\quad \textstyle 
 \sum_{\ip=1}^{\ap-1}\nwp_{\tm-1}^\ip \leq u < \sum_{\ip=1}^{\ap}
 \nwp_{\tm-1}^\ip,
 \label{eq:ancestor}
\end{align}
for $u\sim\Uni(0,1)$. Repeating the above process independently $\Np$ 
times gives multinomial resampling. That is, we draw $u^\ip \sim 
\Uni(0,1)$ and find the corresponding $\ap^\ip$ for each 
$\ip=1,\ldots,\Np$. In the context of \gls{SMC} the index variables $\{\ap^\ip\}_{\ip=1}^\Np$ determine the ancestry of the particles and they are therefore often referred to as \emph{ancestor indices}.

\paragraph{Stratified Resampling} One way to improve resampling is by 
stratification on the uniform random numbers $u^\ip$ \citep{kitagawa1996monte,fearnhead1998sequential}. This means we 
divide the unit interval into $\Np$ strata, \newline${(0,\nicefrac{1}{\Np}), 
(\nicefrac{1}{\Np},\nicefrac{2}{\Np}), \ldots, 
(1-\nicefrac{1}{\Np},1)}$. Then, we generate $u^\ip \sim 
\Uni(\nicefrac{\ip-1}{\Np}, \nicefrac{\ip}{\Np})$ for each strata 
$\ip=1,\ldots,\Np$. Finally, the corresponding ancestor indices
$\ap^\ip$ are given by studying \cref{eq:ancestor}.

\cref{code:stratified_resampling} shows how this can be implemented in 
Python. The main change compared to multinomial resampling, 
\cref{code:multinomial_resampling}, is the way the $u^\ip$'s are generated.
\begin{lstlisting}[caption={Sampling $\Np$ times from 
$\sum_{i} \nwp^i \dirac_{\xp^i}$ using stratification.}, 
label=code:stratified_resampling]
def stratified_resampling(w, x):
    N = w.shape[0]
    u = (np.arange(N) + npr.rand(N))/N
    bins = np.cumsum(w)
    return x[np.digitize(u,bins)]
\end{lstlisting}

\paragraph{Systematic Resampling} We can take correlation to the 
extreme by only generating a single uniform number $u\sim\Uni(0,1)$ to 
set all the $u^\ip$'s. Systematic resampling \citep{carpenter1999improved,whitley1994genetic} means that we let
\begin{align*}
    u^\ip = \frac{i-1}{\Np} + \frac{u}{\Np},
\end{align*}
where the random number $u\sim\Uni(0,1)$ is identical for all $\ip=1,\ldots,\Np$. Then, 
we find corresponding indices $\ap^\ip$ by again studying
\cref{eq:ancestor}. Note that just like in stratified resampling, 
systematic resampling also generates one $u^\ip$ in each strata 
$(\nicefrac{\ip-1}{\Np},\nicefrac{\ip}{\Np})$. 
However, in this case the $u^\ip$'s are based on a single random value 
$u$. This means that systematic resampling will be more
computationally efficient than stratifies and multinomial resampling.

The code change to implement systematic resampling is simple. We only 
need to change a single line in \cref{code:stratified_resampling}. We 
replace line $3$ by: \lstinline"u = (np.arange(N) + npr.rand())/N".

Both systematic and stratified resampling are heavily used in practice. 
In many cases systematic resampling achieves slightly better results 
\citep{hol2006resampling}.
However, systematic resampling can sometimes lead to non-convergence 
\citep{gerber2017negative}, depending on the ordering of the samples. 

For a more thorough discussion on these, and other, resampling methods
see \eg \citet{douc2005comparison,hol2006resampling}.

\subsection{Effective Sample Size and Adaptive Resampling}\label{sec:ess}
The resampling step introduces extra variance by eliminating 
particles with low weights, and replicating particles with high 
weights. If the variance of the normalized weights is low, this step 
might be unnecessary. By tracking the variability of the normalized 
weights, and trigger a resampling step only when the variability crosses a 
pre-specified threshold, we can alleviate this issue. This is known as 
\emph{adaptive resampling} in the \gls{SMC} literature. Often we study 
the \gls{ESS} \citep{liu2004monte} to gauge the variability of the weights, which for 
iteration $\tm$ is
\begin{align}
    \ess_\tm &= \frac{1}{\sum_{\ip=1}^\Np \left(\nwp_\tm^\ip 
    \right)^2}.
\end{align}
The \gls{ESS} is a positive variable taking values in the 
continuous range between $1$ and $\Np$. For \gls{IS} the \gls{ESS} is 
an approximation of the number of 
exact samples from $\nmod_{\tm}$ that we would need to achieve a comparable 
estimation accuracy \citep{doucet2009tutorial}. A common approach is to resample only at 
iterations when the \gls{ESS} falls below $\nicefrac{\Np}{2}$.

Adaptive resampling can be implemented by slight alterations to the 
proposal and weight updates in \cref{eq:smc:proposal,eq:smc:weights}. 
If $\ess_{\tm-1}$ is above the prespecified threshold we simply omit 
the resampling step in \cref{eq:smc:proposal} and obtain
\begin{align*}
    & &  & \textit{propagate} 
    & & \xp_\tm^\ip\given\xp_{1:\tm-1}^\ip \sim 
    \prop_\tm(\xp_\tm\given\xp_{1:\tm-1}^\ip), \\
    & &  & \textit{concatenate} 
    & & \xp_{1:\tm}^\ip = \left(\xp_{1:\tm-1}^\ip,\xp_\tm^\ip\right). 
\end{align*}
The fact that resampling is omitted at some iterations needs to be accounted for when computing the weights. For iterations where we do not resample the weight update is,
\begin{align}
    \uwp_\tm^\ip &= \frac{\nwp_{\tm-1}^\ip}{1/N} \cdot 
    \frac{\umod_\tm(\xp_{1:\tm}^\ip)}{\umod_{\tm-1}(\xp_{1:\tm-1}^\ip)
    \prop_\tm(\xp_\tm^\ip\given\xp_{1:\tm-1}^\ip)},
    \quad \nwp_\tm^\ip = \frac{\uwp_\tm^\ip}{\sum_{\jp=1}^\Np \uwp_\tm^\jp},
    \quad \ip = 1,\ldots, \Np.
    \label{eq:smc:ess-weights}
\end{align}
This can be thought of as adding an extra importance correction, where the previous weights $\{ \nwp_{\tm-1}^\ip \}_{\ip=1}^\Np$ define a ``target distribution'' for the ancestor indices and the factor $\nicefrac{1}{N}$ is the ``proposal'' corresponding to not resampling.
When the \gls{ESS} falls below our pre-set threshold, we use the standard 
update \cref{eq:smc:proposal,eq:smc:weights} with resampling instead.

Adaptive resampling is usually combined with the low-variance 
resampling techniques explained above for further variance reduction.

\begin{remark}
	The constant factors in the weight expression will cancel when normalizing the weights. However, these constants still need to be included for the expression for the normalizing constant estimate \cref{eq:smc:zhat} to be valid. That is, the extra importance correction added at those iterations when we do not resample should be the ratio of the \emph{normalized weights} from the previous iteration, divided by the constant $\nicefrac{1}{N}$. An alternative approach appearing in the literature is to neglect the constants when computing the weights, but then modify the expression for the normalizing constant estimate \cref{eq:smc:zhat} instead.
\end{remark}


\section{Analysis and Convergence}\label{sec:analysisconvergence}
Since its conception in the 1990s, significant effort has been spent 
on studying the theoretical properties of \gls{SMC} methods. We will 
review and discuss a few select results in this section. For an early 
review of the area, see \eg \citet{delMoral2004}. The theorems we 
discuss below all hold for a number of conditions on the proposal, 
probabilistic model, and test functions. For brevity we have omitted these 
exact conditions and refer to the cited proofs for details.

\paragraph{Unbiasedness}
One of the key properties of \gls{SMC} approximations is that they provide 
\emph{unbiased} approximations of integrals of functions $\fun_\tm$
with respect to the \emph{unnormalized} target distribution $\umod_\tm$. We 
formalize this in \cref{thm:unbiasedness}.
\begin{theorem}[Unbiasedness]
    \begin{align*}
    \Exp\left[\prod_{\km=1}^{\tm} \left( \frac{1}{N}\sum_{\jp=1}^\Np 
    \uwp_{\km}^\jp \right) \cdot \sum_{\ip=1}^{\Np} \nwp_{\tm}^{\ip} \fun_\tm( \xp_{1:\tm}^{\ip} )\right]
    = \int \fun_\tm( \xp_{1:\tm} ) \umod_\tm( \xp_{1:\tm} ) 
    \dif \xp_{1:\tm}
    \end{align*}
    \label{thm:unbiasedness}
\end{theorem}
\begin{proof}
See \citet[Theorem 7.4.2]{delMoral2004}. For the special case 
$\fun_\tm\equiv 1$ see also \cref{sec:unbiasedZ} and Appendix~\ref{sec:smc:unbiasedZ}.
\end{proof}
A particularly important special case is when $\fun_\tm \equiv 1$ and 
we approximate the normalization constant of $\umod_\tm$. We have that 
$\Exp[\widehat\Z_\tm] = \Z_\tm$, where the expectation is taken with 
respect to all the random variables generated by the \gls{SMC} 
algorithm. If we instead consider the more numerically stable 
$\log \widehat\Z_\tm$, we have by 
Jensen's inequality that $\Exp\left[\log \widehat\Z_\tm\right] \leq 
\log\Z_\tm$. This means that the estimator of the log-normalization 
constant is negatively biased. This is illustrated by the violin plot 
discussed in \cref{ex:pathdegen_smc}.

We will delve 
deeper into the applications of the unbiasedness property and its 
consequences in \cref{sec:pm}.

\paragraph{Laws of Large Numbers}
While integration with respect to unnormalized distributions can be 
estimated unbiasedly, this is unfortunately not true when estimating 
expectations with respect to the normalized target distribution $\nmod_\tm$. 
However, \gls{SMC} methods are still strongly consistent, leading to 
exact solutions when the number of particles $\Np$ tend to infinity. 
We formalize this law of large numbers in \cref{thm:lln}.
\begin{theorem}[Law of Large Numbers]
    \begin{align*}
    \widehat\nmod_\tm(\fun_\tm) \eqdef 
    \sum_{\ip=1}^{\Np} \nwp_{\tm}^{\ip} \fun_\tm( \xp_{1:\tm}^{\ip} )
    \convAS \nmod_\tm(\fun_\tm) = \int \fun_\tm( \xp_{1:\tm} ) \nmod_\tm( \xp_{1:\tm} ) 
    \dif \xp_{1:\tm}, ~\Np \to \infty
    \end{align*}
    \label{thm:lln}
\end{theorem}
\begin{proof}
See \citet[Theorem 7.4.3]{delMoral2004}.
\end{proof}

\paragraph{Central Limit Theorem} While the law of large numbers from 
the previous section shows that \gls{SMC} approximations are exact in 
the limit of infinite computation, it tells us nothing about the 
quality of our estimate. The \gls{CLT} in \cref{thm:esmc:clt} tells us about 
the limiting distribution of our \gls{SMC} estimate and its asymptotic 
variance. This gives us a first approach to get an understanding for 
the precision of our \gls{SMC} approximation to $\nmod_\tm(\fun_\tm)$.
\begin{theorem}[Central Limit Theorem]
    \begin{align*}
    &\sqrt{\Np}\left(
    \widehat\nmod_\tm(\fun_\tm)
    - \nmod_\tm(\fun_\tm) 
    \right) \convD 
    \Norm\left(0, \varp_{\tm} \left(\fun_\tm\right) \right), ~\Np \to 
    \infty,
    \end{align*}
    where $\varp_\tm(\cdot)$ is defined recursively for a measurable 
    function $h$,
    \begin{align*}
    \varp_{\tm}\left(\fun \right) &= 
    \widetilde\varp_{\tm}\left(\nwp_\tm'(\xp_{1:\tm})
    \left(\fun(\xp_{1:\tm}) - \nmod_\tm(\fun)\right)\right), ~\tm \geq 
    1,\\
    \widetilde\varp_{\tm}(\fun) &= 
    \widehat\varp_{\tm-1}\left(\Exp_{\prop_\tm(\xp_\tm\given\xp_{1:\tm-1})}
    \left[\fun(\xp_{1:\tm})\right]\right) + 
    \Exp_{\nmod_{\tm-1}}\left[\Var_{\prop_\tm(\xp_\tm\given\xp_{1:\tm-1})}(\fun)\right],
    ~\tm > 1, \\
    \widehat \varp_\tm(\fun) &= \varp_\tm(\fun) + 
    \Var_{\nmod_\tm}(\fun), ~\tm \geq 1,
    \end{align*}
    initialized with $\widetilde\varp_1(\fun) = \Var_{\prop_1}(\fun)$ 
    and where 
    \begin{align*}
    \nwp_\tm'(\xp_{1:\tm}) &= 
    \frac{\nmod_\tm(\xp_{1:\tm})}{\nmod_{\tm-1}(\xp_{1:\tm-1}) 
    \prop_\tm(\xp_\tm\given\xp_{1:\tm-1})}, \\
    \nwp_1'(\xp_{1:\tm}) &= \frac{\nmod_1(\xp_1)}{\prop_1(\xp_1)}.
    \end{align*}
    \label{thm:esmc:clt}
\end{theorem}
\begin{proof}
See \citet{chopin2004}.
\end{proof}

\begin{example}[State Space Model]
We study the following example of a state space model
\begin{align*}
	p(\xm,\ym) &= \prod_{\km=1}^{\dimX} p(\xp_{1,\km}) \gmod(\y_{1,\km}\given\xp_{1,\km}) \cdot \prod_{\lm=2}^\tm \left[ \prod_{\km=1}^{\dimX}  \fmod(\xp_{\lm,\km}\given\xp_{\lm-1,\km}) \gmod(\y_{\lm,\km}\given\xp_{\lm,\km}) \right],
\end{align*}
where $\dimX$ denotes the dimension of $\xp_\lm$ and $\y_\lm$. For simplicity assume that $\y_{\lm,\km}=\y_{\lm,\mm}$, $\forall \km, \mm$ and that $\Exp_{p(\xm\given\ym)}\left[\xm\right] = 0$. Let the function of interest be $\fun_{\tm}(\xp_{1:\tm}) = \sum_{\km=1}^{\dimX} \xp_{\tm,\km}$, \ie the sum of the components of the state at the most recent iteration. Studying the asymptotic variance in \cref{thm:esmc:clt} for the so called \emph{fully adapted} \gls{SMC}, explained in \cref{sec:twist:learn}, we get
\begin{align*}
	 \varp_{\tm} \left(\fun_\tm\right) &= \dimX A_{\tm} + \sum_{\lm=1}^{\tm-1} \dimX B_{\lm}^{\dimX-1} A_{\lm} + \dimX (\dimX-1)B_{\lm}^{\dimX-2} C_{\lm}^2, 
\end{align*}
where the constants $A_{\lm}$, $B_{\lm}$, and $C_{\lm}$ are defined by
\begin{align*}
	 A_\tm &= \int \xp_{\tm,\km}^2 p(\xp_{1:\tm,\km} \given \y_{1:\tm,\km}) \dif \xp_{1:\tm,\km}, \\
	 A_{\lm} &= \int \frac{p(\xp_{1:\lm,\km} \given \y_{1:\tm,\km})^2}{p(\xp_{1:\lm,\km} \given \y_{1:\lm,\km})} \left(\int \xp_{\tm,\km} p(\xp_{\tm,\km}\given \xp_{\lm,\km}, \y_{1:\tm,\km}) \dif \xp_{\tm,\km}\right)^2 \dif \xp_{1:\lm,\km}, \\
	 B_{\lm} &= \int \frac{p(\xp_{1:\lm,\km} \given \y_{1:\tm,\km})^2}{p(\xp_{1:\lm,\km} \given \y_{1:\lm,\km})}\dif \xp_{1:\lm,\km}, \\
	 C_{\lm} &=  \int \frac{p(\xp_{1:\lm,\km} \given \y_{1:\tm,\km})^2}{p(\xp_{1:\lm,\km} \given \y_{1:\lm,\km})} \int \xp_{\tm,\km} p(\xp_{\tm,\km}\given \xp_{\lm,\km}, \y_{1:\tm,\km}) \dif \xp_{\tm,\km} \dif \xp_{1:\lm,\km}.
\end{align*}
For the proof see \citet[Proposition 3]{naesseth2016high}. 

This result tells us that the asymptotic variance is exponential in $\dimX$, the dimension of $\xp_\tm$. Furthermore, it tells us that for the asymptotic variance to be stable in $\tm$, \ie to not be increasing with increasing $\tm$,  we need that our model forgets its initial conditions. Concretely, we need the inner expectations in $A_\lm$ and $C_\lm$ to tend to $0$ (the posterior mean) quickly enough to be able to ensure that we can bound the sum by a (finite) constant that does not depend on $\tm$. This holds more generally (under various assumptions on the model and proposal) for approximations of the filtering distribution, $\nmod_{\tm}(\xp_\tm) = p(\xp_\tm\given\y_{1:\tm})$, of \glspl{SSM}. We formalize the result in \cref{prop:clt:ssm}.
\label{ex:clt:ssm}
\end{example}
\begin{proposition}[Asymptotic Variance for Filtering Approximation in \GLS{SSM}]
Under appropriate forgetting conditions on the model
 	\begin{align*}
 		\varp_{\tm} \left(\fun(\xp_\tm)\right) \leq \Var_{\nmod_\tm}\left(\fun(\xp_\tm)\right) + \|\fun\|^2 \mathrm{c},
 	\end{align*}
 	for some finite constant $\mathrm{c}$.
    \label{prop:clt:ssm}
\end{proposition}
\begin{proof}
See \citet{whiteley2013stability}.
\end{proof}
\cref{prop:clt:ssm} shows that the asymptotic variance can be stable as a function of $\tm$, this is in general only true for test functions that depend only on the most recent latent variable $\xp_\tm$ and under various assumptions on the model and proposal \citep{del2001stability,delMoral2004,whiteley2013stability}. When considering the variance of $\nicefrac{\hatZ_\tm}{\Z_\tm}$ of \glspl{SSM}, it is instead possible to show that the variance increases only linearly with $\tm$ \citep{whiteley2013stability,delMoral2004}. This implies that to obtain a good estimate of the normalization constant, we must choose the number of particles $\Np$ to scale at least linearly in $\tm$.

\paragraph{Sample Bounds} 
Another way of looking at the \gls{SMC} method, disregarding test 
functions, is as a direct approximation to the target distribution 
itself. The weights and samples generated by the \gls{SMC} algorithm provide an approximation
$\widehat{\nmod}_\tm\left(\xp_{1:\tm}\right)$, \cref{eq:smc:target_approx}, to the true target distribution $\nmod_\tm(\xp_{1:\tm})$.

With this point of view we can study bounds on the difference between the
distribution of a sample drawn from the \gls{SMC} approximation compared to
that of a sample drawn from the target distribution. Specifically, assume that we generate 
a sample $\xp_{1:\tm}'$ by first running an \gls{SMC} sampler to 
generate an approximation $\widehat\nmod_\tm$ of $\nmod_\tm$, and then 
simulate $\xp_{1:\tm}' \sim \widehat\nmod_\tm$. Then, the marginal 
distribution of $\xp_{1:\tm}'$ is $\Exp[\widehat\nmod_\tm]$, where the 
expectation is taken with respect to all random variables generated by 
the \gls{SMC} algorithm. It is worth noting that this distribution, 
$\Exp[\widehat\nmod_\tm]$, may be continuous despite the fact that 
$\widehat\nmod_\tm$ is a point-mass distribution by construction. 
In \cref{thm:samplebound} we restate a generic sample 
bound on the \gls{KL} divergence from the expected \gls{SMC} approximation 
$\Exp\left[\widehat \nmod_\tm\right]$ to the target distribution $\nmod_\tm$. The \gls{KL} divergence from a distribution $q$ to another distribution $p$ is defined by $\Kl\,\left(p(\xp) \| q(\xp)\right) \eqdef \int p(\xp) \left(\log p(\xp) - \log q(\xp)\right) \dif \xp$.
\begin{theorem}[Sample Bound]
    \begin{align*}
     \Kl\,\left(\Exp\left[\widehat \nmod_\tm\right] \| \nmod_\tm \right) \leq 
     \frac{\mathcal{C}}{\Np},
    \end{align*}
    for a finite constant $\mathcal{C}$.
    \label{thm:samplebound}
\end{theorem}
\begin{proof}
See \citet[Theorem 8.3.2]{delMoral2004}.
\end{proof}
From this result we can conclude that in fact the \gls{SMC} 
approximation tends to the true target in distribution as the number 
of particles increase. 

Recently there has been an increased interest for using \gls{SMC} as 
an approximation to the target distribution, rather than just as a method 
for estimating expectations with respect to test functions. This point 
of view has found applications in \eg probabilistic programming and 
variational inference \citep{wood2014new,naesseth18a,huggins2015sequential}.


\chapter{Learning Proposals and Twisting Targets}\label{sec:proptwisting}
The two main design choices of \acrlong{SMC} methods are the proposal 
distributions $\prop_\tm$, $\tm \in \{1,\dots,\Tm\}$ and the intermediate (unnormalized) target distributions $\umod_\tm$, $\tm \in \{1,\dots,\Tm-1\}$. 
Carefully choosing these can drastically improve the efficiency of the 
algorithm and the accuracy of our estimation. Adapting both the proposal 
and the target distribution is especially important if the latent 
variables are high-dimensional.

In the first section below we discuss how to choose, or learn, the 
proposal distribution for a fixed target distribution. Then, in the final 
section we discuss how we can design a good sequence of intermediate 
target distributions for a given probabilistic model.

\section{Designing the Proposal Distribution}\label{sec:proposal}
The choice of the proposal distribution is perhaps the most important 
design choice for an 
efficient \acrlong{SMC} algorithm. A common choice is to 
propose samples from the model prior; this is simply known as the 
prior proposal. When using the prior as proposal the algorithm is commonly known as the bootstrap particle filter or bootstrap \gls{SMC} \citep{GordonSS:1993}. However, using the prior can lead to poor 
approximations for a small number of particles, especially if the 
latent space is high-dimensional. 

We will in this section derive the locally (one-step) \emph{optimal} 
proposal distribution, or \emph{optimal} proposal for short. 
Because it is typically intractable, we further 
discuss various ways to either emulate it or approximate it directly. 
Finally, we discuss alternatives to learn efficient proposal 
distributions using an end-to-end variational perspective.

\subsection{Locally Optimal Proposal Distribution}
The locally optimal proposal distribution \citep{doucet2000sequential} is the distribution we obtain 
if we assume that we have a perfect approximation at iteration 
$\tm-1$. 
Then, choose 
the proposal $\prop_{\tm}$ that minimizes the \gls{KL} divergence from 
the joint distribution
$\nmod_{\tm-1}(\xp_{1:\tm-1})\prop_{\tm}(\xp_\tm \given \xp_{1:\tm-1})$ 
to $\nmod_{\tm}(\xp_{1:\tm})$. We formalize this result in 
\cref{thm:optimalprop}. Equivalently, we can view this as the proposal minimizing the variance 
of the incremental weights, \ie 
$\nicefrac{\uwp_\tm^\ip}{\uwp_{\tm-1}^\ip}$, 
with respect to the newly generated samples $\xp_{\tm}^\ip$. 
\begin{proposition}[Locally optimal proposal distribution]
The optimal proposal $\prop_\tm^\star(\xp_\tm\given\xp_{1:\tm-1})$ 
minimizing $\Kl\left(\nmod_{\tm-1}(\xp_{1:\tm-1})\prop_{\tm}(\xp_\tm 
\given \xp_{1:\tm-1}) \| \nmod_{\tm}(\xp_{1:\tm})\right)$ is given by
\begin{align}
    \prop_\tm^\star(\xp_\tm \given \xp_{1:\tm-1}) &= 
    \nmod_{\tm}(\xp_{\tm} \given \xp_{1:\tm-1}) =
    \frac{\umod_\tm(\xp_{1:\tm})}{\umod_\tm(\xp_{1:\tm-1})},
    \label{eq:prop:localoptimal}
\end{align}
where $\umod_\tm(\xp_{1:\tm-1}) = \int \umod_\tm(\xp_{1:\tm}) \dif 
\xp_\tm$.
\label{thm:optimalprop}
\end{proposition}
\begin{proof}
If we let ``$\text{const}$'' denote terms constant with respect to 
the proposal distribution $\prop_{\tm}(\xp_\tm \given \xp_{1:\tm-1})$, we get
    \begin{align*}
        &\Kl\left(\nmod_{\tm-1}(\xp_{1:\tm-1})\prop_{\tm}(\xp_\tm \given 
        \xp_{1:\tm-1}) \| \nmod_{\tm}(\xp_{1:\tm})\right)  \\
        &= \Exp_{\nmod_{\tm-1} \prop_\tm}\left[\log \prop_{\tm}(\xp_\tm \given 
        \xp_{1:\tm-1}) - \log \nmod_\tm(\xp_{1:\tm})\right] + 
        \text{const} \\
        &= \Exp_{\nmod_{\tm-1} \prop_\tm}\left[\log \prop_{\tm}(\xp_\tm \given 
        \xp_{1:\tm-1}) - \log \nmod_\tm(\xp_\tm \given \xp_{1:\tm-1})\right] + 
        \text{const} \\
        &=\Exp_{\nmod_{\tm-1}(\xp_{1:\tm-1})}\left[\Kl\left(\prop_{\tm}(\xp_\tm \given 
        \xp_{1:\tm-1}) \| \nmod_\tm(\xp_\tm \given \xp_{1:\tm-1})\right)\right] + 
        \text{const},
    \end{align*}
where the inner (conditional) \acrlong{KL} divergence is zero if and only 
if $\prop_{\tm}(\xp_\tm \given \xp_{1:\tm-1}) \equiv \nmod_\tm(\xp_\tm 
\given \xp_{1:\tm-1})$.
\end{proof}

In \cref{ex:optprop:running} we show that the optimal proposal is 
analytically tractable for our running non-Markovian Gaussian example.
\begin{example}[Optimal proposal for \cref{ex:running}]
If we let $\umod_\tm(\xp_{1:\tm}) = 
\umod_{\tm-1}(\xp_{1:\tm-1}) \fmod(\xp_{\tm}\given\xp_{\tm-1}) 
\gmod(\y_{\tm} \given \xp_{1:\tm})$, then the (locally) optimal proposal for our 
running example is analytically tractable. It is
\begin{align*}
    \prop_\tm^\star(\xp_\tm \given \xp_{1:\tm-1}) &= \frac{
    \umod_\tm(\xp_{1:\tm})}{\umod_\tm(\xp_{1:\tm-1})} \propto
    \fmod(\xp_{\tm}\given\xp_{\tm-1}) \gmod(\y_{\tm} \given 
    \xp_{1:\tm})\\
    &\propto \Norm\left(\xp_\tm \given \frac{r\phi \xp_{\tm-1}+q 
    \y_\tm- q \sum_{\km=1}^{\tm-1}\beta^{\tm-\km} \xp_{\km}}{q+r}, 
    \frac{qr}{q+r} \right).
\end{align*}
\label{ex:optprop:running}
\end{example}

In most practical cases the optimal proposal distribution is not a 
feasible alternative. The resulting importance weights 
are intractable, or simulating random variables from the optimal proposal is too computationally 
costly. Below we discuss various common approaches for approximating it.

\subsection{Approximations to the Optimal Proposal Distribution}
There have been a number of suggestions over the years on how to
approximate the optimal distribution. We will review three analytic 
approximations based on a Gaussian assumption, as well as briefly 
describe an \emph{exact 
approximation}. 

\paragraph{Laplace Approximation}
We obtain the Laplace approximation to the optimal proposal by a 
second-order Taylor approximation of the log-\gls{PDF} around a point 
$\bar\xp$ \citep{doucet2000sequential}. If we let $l_\tm(\xp_{\tm}) 
\eqdef \log \nmod_{\tm} (\xp_{\tm} \given \xp_{1:\tm-1} )$, supressing 
the dependence on $\xp_{1:\tm-1}$, then
\begin{align*}
    &l_\tm(\xp_{\tm}) \approx l_\tm(\bar\xp)+ \grad l_\tm(\bar\xp)^{\top}
    (\xp_{\tm} -\bar\xp)
    + \frac{1}{2} (\xp_{\tm} -\bar\xp)^\top \hessian l_\tm(\bar\xp)
    (\xp_{\tm} -\bar\xp) \\
    &= \const \,+ \\
    &\frac{1}{2} \left(\xp_\tm-\bar\xp+\left(\hessian 
    l_\tm(\bar\xp)\right)^{-1} \grad l_\tm(\bar\xp) \right)^\top \hessian 
    l_\tm(\bar\xp)\left(\xp_\tm-\bar\xp+\left(\hessian 
    l_\tm(\bar\xp)\right)^{-1} \grad l_\tm(\bar\xp) \right),
\end{align*}
where $\grad l_\tm$ and $\hessian l_\tm$ are the gradient and the Hessian 
of the log-\gls{PDF} with respect to $\xp_\tm$, respectively. A natural 
approximation to the optimal proposal is then
\begin{align}
    \prop_\tm(\xp_{\tm}\given \xp_{1:\tm-1}) &= \Norm\left(\xp_\tm \given 
    \bar\xp - \left(\hessian 
    l_\tm(\bar\xp)\right)^{-1}\grad l_\tm(\bar\xp), 
    - \hessian l_\tm(\bar\xp)^{-1} \right).
    \label{eq:optprop:laplace}
\end{align}
The mode of the distribution can be a good choice for the 
linearization point $\bar\xp$ if the distribution is unimodal. With this 
choice the mean simplifies to just $\bar\xp$. However, the mode is 
usually unknown and will depend on the value of $\xp_{1:\tm-1}$. This 
means that we are required to run a separate 
optimization for each particle $\xp_{1:\tm-1}^\ip$ and iteration $\tm$ 
to find the mode and Hessian. This can outweigh the benefits of the improved proposal 
distribution.

\paragraph{Extended and Unscented Kalman Filter Approximations}
The \gls{EKF} and the \gls{UKF} \citep{anderson1979,julier1997new}
are by now standard solutions to non-linear and non-Gaussian filtering 
problems. We can leverage these ideas to derive another class of 
Gaussian approximations to the optimal proposal distribution
\citep{doucet2000sequential,van2001unscented}. 

The two methods depend on a structural equation representation of our 
probabilistic model,
\begin{subequations}
    \begin{align}
    \xp_\tm &= \amod(\xp_{1:\tm-1}, \vn_\tm), \\
    \y_\tm &= \cmod(\xp_{1:\tm}, \en_\tm),
    \end{align}\label{eq:structural:ac}%
\end{subequations}
where $\vn_\tm$ and $\en_\tm$ are random variables. The representation 
implies a joint distribution on $\xp_\tm$ and $\y_\tm$ conditional on 
$\xp_{1:\tm-1}$, \ie $p(\xp_\tm, \y_\tm \given \xp_{1:\tm-1})$. The 
locally optimal proposal distribution corresponding to this 
representation is given by $\prop_\tm^\star(\xp_\tm\given\xp_{1:\tm-1}) = p(\xp_\tm\given 
\xp_{1:\tm-1}, \y_\tm)$. 

The \gls{EKF} and \gls{UKF} approximations rely on a Gaussian 
approximation to the joint conditional distribution of $\xp_\tm$ and $\y_\tm$,
\begin{align}
    p(\xp_\tm, \y_\tm \given \xp_{1:\tm-1}) &\approx 
    \widehat{p}(\xp_\tm, \y_\tm \given \xp_{1:\tm-1}) =
    \Norm\left(
    \left(\begin{array}{c}
        \xp_\tm \\
        \y_\tm
    \end{array}\right)\,
    \Big|\,
    \widehat\mu, \,
    \widehat\Sigma
    \right),
    \label{eq:proptwist:ekfukf}
\end{align}
where we have suppressed the dependence on $\xp_{1:\tm-1}$ for 
clarity. We block the mean $\widehat\mu$ and covariance 
$\widehat\Sigma$ as follows
\begin{align}
    \widehat\mu &= \left(\begin{array}{c}
        \widehat\mu_\xp \\
        \widehat\mu_\y
    \end{array}\right), \quad
    \widehat\Sigma = \left(\begin{array}{cc}
        \widehat\Sigma_{\xp\xp} & \widehat\Sigma_{\xp\y}\\
        \widehat\Sigma_{\y\xp} &\widehat\Sigma_{\y\y}
    \end{array}\right).
\end{align}
Under the assumptions of the approximation in \cref{eq:proptwist:ekfukf}, 
the distribution of $\xp_\tm \given \xp_{1:\tm-1}, \y_\tm$ is 
tractable:
\begin{align}
    \widehat{p}(\xp_\tm \given \xp_{1:\tm-1}, \y_\tm) &=
    \Norm\left(\xp_\tm\given \mu_\tm, \, \Sigma_\tm\right),
\end{align}
with
\begin{subequations}
    \begin{align}
    \mu_\tm &= \widehat\mu_\xp + \widehat\Sigma_{\xp\y} 
    \widehat\Sigma_{\y\y}^{-1}\left(\y_\tm -\widehat\mu_\y\right),\\
    \Sigma_\tm &=  \widehat\Sigma_{\xp\xp} -  \widehat\Sigma_{\xp\y}  
    \widehat\Sigma_{\y\y}^{-1}  \widehat\Sigma_{\y\xp}.
    \end{align}\label{eq:proptwist:ekfukf:prop}%
\end{subequations}
The key difference between the \gls{EKF}- and \gls{UKF}-based 
approximations is how to compute the estimates $\widehat\mu$ and 
$\widehat\Sigma$. We leave the details of these procedures to 
Appendix~\ref{sec:proptwisting:tut}, and focus here on the intuition behind 
them.

The \gls{EKF} uses a first-order Taylor approximation to the 
non-linear functions $\amod, \cmod$ to derive an approximation to the 
full posterior distribution based on 
exact (analytical) updates of the approximate model. Following this line of thinking 
we can similarly linearize $\amod, \cmod$ locally. Then under a Gaussian 
assumption on the noise $\vn_\tm, \en_\tm$, we compute the distribution 
$\widehat{p}(\xp_\tm, \y_{\tm} \given \xp_{1:\tm-1})$ exactly for the 
linearized model. The \gls{EKF}-based approximations 
$\widehat\mu$ and $\widehat\Sigma$ can be found in 
\cref{eq:ekf_proposal} in the appendix.

The \gls{UKF} on the other hand uses so-called \emph{sigma points} to 
reach the Gaussian approximation $\widehat{p}(\xp_\tm, \y_{\tm} \given \xp_{1:\tm-1})$. 
The key idea is to choose a set of points, pass them through the 
nonlinear functions $\amod$ and $\cmod$, and then estimate the mean and variance of the 
transformed point-set. Unlike the \gls{EKF}-based approximation, the 
\gls{UKF}-based approximation does 
not require that the noises $\vn_\tm$ and $\en_\tm$ are Gaussian distributed. 
However, we do require that the mean and variance for these random 
variables are available. The \gls{UKF}-based approximations 
$\widehat\mu$ and $\widehat\Sigma$ can be found in 
\cref{eq:ukf_based_musig} in the appendix.

\paragraph{Analytic Gaussian Approximations}
We refer to the Laplace-, \gls{EKF}- and \gls{UKF}-based 
approximations as \emph{analytic Gaussian} approximations to the locally 
optimal proposal distribution. We summarize these three approaches 
in \cref{alg:analytic_gaussian_proposal}. Which one works best depends 
heavily on the model being studied. 
\begin{algorithm}[tb]
    Proposal: $\prop_\tm(\xp_\tm \given \xp_{1:t-1}) = \Norm(\xp_\tm \given \mu_\tm, \Sigma_\tm)$
    
    \begin{itemize}
    \item The \textrm{Laplace} approximation around the point $\bar\xp$ is
    \begin{align*}
        \mu_\tm &= \bar\xp - \left(\hessian l_\tm(\bar\xp)\right)^{-1} \grad 
        l_\tm(\bar\xp),\\
        \Sigma_\tm &= - \hessian l_\tm(\bar\xp)^{-1},
    \end{align*}
    where $l_\tm(\xp_\tm) = \log \umod_{\tm}(\xp_{1:\tm})$ and 
    derivatives are \wrt $\xp_\tm$.
    
    \item The \gls{EKF}- and \gls{UKF}-based approximations are
    \begin{align*}
        \mu_\tm &= \widehat\mu_\xp + \widehat\Sigma_{\xp\y} 
        \widehat\Sigma_{\y\y}^{-1}\left(\y_\tm -\widehat\mu_\y\right),\\
        \Sigma_\tm &=  \widehat\Sigma_{\xp\xp} -  \widehat\Sigma_{\xp\y}  
        \widehat\Sigma_{\y\y}^{-1}  \widehat\Sigma_{\y\xp},
    \end{align*}
    where the variables are defined in \cref{eq:ekf_proposal}~(\gls{EKF}) 
    and in \cref{eq:ukf_based_musig}~(\gls{UKF}), respectively.
    \end{itemize}
    \caption{Analytic Gaussian Proposal Approximations}
    \label{alg:analytic_gaussian_proposal}
\end{algorithm}
In \cref{ex:analyticgaussian} we study a simple example when $\xp$ and 
$\y$ are one-dimensional random variables.

\begin{example}[Gaussian Proposal Approximations]
We illustrate the analytic Gaussian approximations from 
\cref{alg:analytic_gaussian_proposal} based on a simple scalar example
\begin{subequations}
    \begin{align}
        \xp &= \vn, \quad \vn \sim \Norm(0,1),\\
        \y &= \frac{(\xp + 1)(\xp - 1)(\xp - 3)}{6} + \en,  \quad \en \sim 
        \Norm(0,0.5).
    \end{align}
    \label{eq:analytic_example}%
\end{subequations}
The prior on the latent variable $\xp$ is a standard normal, but the 
measurement model $\cmod$ is a polynomial in $\xp$ with additive 
Gaussian noise. In \cref{fig:proptwisting:analytic_poly} we show the 
true (bimodal) posterior $p(\xp\given\y)$, and the corresponding 
approximations based on the Laplace, \gls{EKF}, and \gls{UKF} methods. 
Neither of these methods can capture the bimodality of the true 
normalized distribution, since they are all based on the basic 
simplifying assumption that the posterior is Gaussian. However, as 
previously discussed a good proposal will cover the bulk of 
probability mass of the target. This means it is likely that the 
\gls{EKF} and \gls{UKF} proposals will outperform the Laplace proposal 
in this case.
\begin{figure}[tb]
    \centering
    \includegraphics[width=0.75\columnwidth]{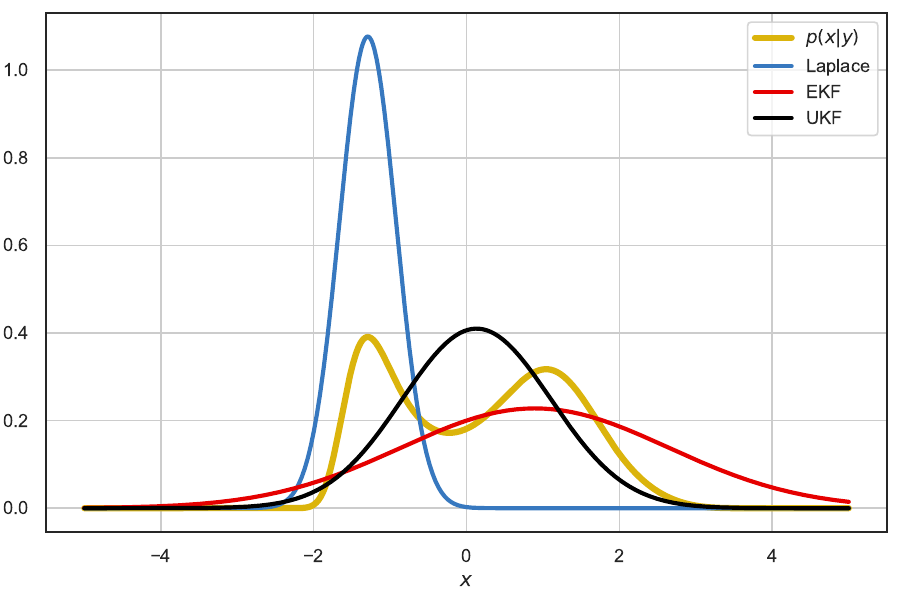}
    \caption{Analytic Gaussian approximations of $p(\xp\given\y)$ for 
    the model in \cref{eq:analytic_example}.}
    \label{fig:proptwisting:analytic_poly}
\end{figure}
\label{ex:analyticgaussian}
\end{example}

\paragraph{Exact Approximations}
In \emph{exact approximations} we use another level of 
\gls{MC}. Instead of drawing exact samples from 
$\prop_\tm^\star$, we approximate draws from it using a nested \gls{MC} 
algorithm \citep{naesseth2015nested,naesseth2016high}. This nested 
algorithm is chosen such that it does not alter the asymptotic 
exactness of the outer \gls{SMC} method. We have discussed 
how \gls{SMC} and \gls{IS} can be used as distribution approximators.
Nested \gls{MC} methods takes this one step further, using a separate 
\gls{SMC} (or \gls{IS}) approximation $\widehat\prop_\tm^\star$ for each sample from 
$\prop_\tm^\star$ we would like. 
This can lead to substantial benefits for \eg high-dimensional latent variables 
$\xp_\tm$. One of the key differences between the analytic Gaussian 
approximations and the exact approximations is the fixed form 
distribution assumption on the analytic approximation. The exact 
approximation is not limited to a standard parametric distribution, 
and we can obtain arbitrarily accurate approximations with enough 
compute. The trade-off is the additional effort needed to generate each 
sample $\xp_\tm^\ip$, which requires us to run independent \gls{SMC} 
algorithms for each particle $\ip$ at each iteration $\tm$. 
While this might seem wasteful, it can in fact improve accuracy 
compared to \gls{SMC} methods with standard proposals 
\citep{naesseth2015nested}, even for equal compute.

We are interested in approximating the locally optimal proposal 
distribution, \ie $\prop_\tm^\star(\xp_\tm \given\xp_{1:\tm-1})
\propto \umod_\tm\left(\xp_{1:\tm}\right)$, 
separately for each particle $\ip$ using an \gls{MC} method. A first 
approach is to use a nested \gls{IS} sampler with a proposal 
$\nprop_\tm(\xp_\tm\given\xp_{1:\tm-1})$ targeting 
$\prop_\tm^\star(\xp_\tm \given\xp_{1:\tm-1})$. We construct an 
approximation of the optimal proposal independently for each particle 
$\xp_{1:\tm-1}^\ip$ as follows
\begin{align}
    \widehat\prop_\tm^\star(\xp_{\tm}\given\xp_{1:\tm-1}^\ip) &= 
    \sum_{\jp=1}^\Mp \frac{\nuwp_\tm^{\jp,\ip}}{\sum_\lm \nuwp_\tm^{\lm,\ip}} 
    \dirac_{\nxp^{\jp,\ip}_\tm}(\xp_\tm), \quad {\nxp^{\jp,\ip}_\tm}\sim 
    \nprop_\tm(\xp_\tm\given\xp_{1:\tm-1}^\ip),
    \label{eq:nis:prop}
\end{align}
where the weights $\nuwp_\tm$ are given by
\begin{align*}
    \nuwp_\tm^{\jp,\ip} &= 
    \frac{\umod_\tm\left((\xp_{1:\tm-1}^\ip,\nxp^{\jp,\ip}_\tm)\right)}
    {\umod_{\tm-1}(\xp_{1:\tm-1}^\ip)\nprop_\tm(\nxp^{\jp,\ip}_\tm\given\xp_{1:\tm-1}^\ip)}.
\end{align*}
We replace $\prop_\tm(\xp_\tm\given\xp_{1:\tm-1})$ in 
\cref{eq:smc:proposal} with samples from the distribution defined in 
\cref{eq:nis:prop}. The corresponding weight update, under the 
assumption that we have resampled $\xp_{1:\tm-1}^\ip$, is given by
\begin{align}
    \uwp_\tm^\ip &= \frac{1}{\Mp}\sum_{\jp=1}^\Mp 
    \nuwp_\tm^{\jp,\ip}.
    \label{eq:nis:weights}
\end{align}
We summarize the nested \gls{IS} approach to approximate the optimal 
proposal in \cref{alg:nis}. The algorithm generates a single sample, 
and must be repeated for each $\ip = 1,\ldots,\Np$.
\begin{algorithm}[tb]
    \SetKwInOut{Input}{input}\SetKwInOut{Output}{output}
    \Input{Unnormalized target distributions $\umod_\tm$, $
    \umod_{\tm-1}$, resampled particle $\xp_{1:\tm-1}^\ip$, nested proposal 
    $\nprop_\tm$, number of samples $\Mp$.}
    
    Sample $\nxp^{\jp,\ip}_\tm \sim \nprop_\tm(\xp_\tm\given\xp_{1:\tm-1}^\ip), 
    ~j=1,\ldots,\Mp$.
    
    Set $\nuwp_\tm^{\jp,\ip} = \frac{\umod_\tm\left((\xp_{1:\tm-1}^\ip,\nxp^{\jp,\ip}_\tm)\right)}
    {\umod_{\tm-1}(\xp_{1:\tm-1}^\ip)\nprop_\tm(\nxp^{\jp,\ip}_\tm\given\xp_{1:\tm-1}^\ip)}$.
    
    Sample $\xp_\tm^\ip \sim \sum_{\jp=1}^\Mp 
    \frac{\nuwp_\tm^{\jp,\ip}}{\sum_\lm \nuwp_\tm^{\lm,\ip}} 
    \dirac_{\nxp^{\jp,\ip}_\tm}(\xp_\tm)$. 
    
    Set $\uwp_\tm^\ip = \frac{1}{\Mp}\sum_{\jp=1}^\Mp 
    \nuwp_\tm^{\jp,\ip}$. 
    
    \caption{Nested \gls{IS} Approximation}
    \label{alg:nis}
\end{algorithm}

By what we know from the theory of \gls{SMC}, see 
\cref{sec:analysisconvergence}, we would expect 
that $\widehat\prop_\tm^\star \convD \prop_\tm^\star$ if we let 
$\Mp\to\infty$. This is true, and as shown by \citet{naesseth2016high} 
the asymptotic variance in the limit of $\Mp\to\infty$ is equal to that of 
using the locally optimal proposal distribution. 
The main implication is that we can obtain arbitrarily accurate optimal 
proposal approximations at the cost of increasing the number of samples 
$\Mp$ for the nested \gls{MC} method. We can also combine the analytic 
Gaussian approximations above with exact approximations by choosing 
$\nprop_\tm$ as either the Laplace, \gls{EKF}, \gls{UKF}, or any other 
analytic approximation. 

The computational complexity of the nested \gls{IS} method is 
increased by a factor of $\Mp$ compared to standard \gls{SMC}. 
A natural question is therefore whether or not we are better of running standard \gls{SMC} with $\nprop_\tm$ as a proposal and using $\Np\Mp$ particles.
However, the memory requirement is lower for nested \gls{SMC} and there are opportunities for 
speeding up through parallelization. There are even variants of 
nested \gls{MC} that can outperform on the same computational 
budget \citep{naesseth2015nested}, \ie when compared to standard 
\gls{SMC} with $\Np\Mp$ particles without taking parallelization into account.
We return to this class of methods in \cref{sec:pm:nsmc}.

\subsection{Learning a Proposal Distribution}
Rather than relying on approximating the locally optimal proposal 
distribution, we can instead directly learn a good proposal 
distribution that can take into account \emph{global} information of the 
target distribution. By parameterizing a suitable class of distributions and 
choosing a cost function to optimize, we can use standard optimization 
tools (such as stochastic gradient descent) to adapt our proposal to the 
problem we are trying to solve. When combined with the \gls{SMC} 
approximation, we refer to these types of approaches as \emph{adaptive} or 
\emph{variational} \gls{SMC} methods. 

We will in this section focus on proposal distributions, $\prop_\tm$, 
parameterized by $\pparams$, we denote this as 
$\prop_\tm(\xp_\tm\given\xp_{1:\tm-1} \idxby \pparams)$. A common 
choice is to use a conditional Gaussian distribution
\begin{align*}
    \prop_\tm(\xp_\tm\given\xp_{1:\tm-1} \idxby \pparams) &= 
    \Norm\left(\xp_\tm\given \mu_\pparams(\xp_{1:\tm-1}), 
    \sigma_{\pparams}^2(\xp_{1:\tm-1})\right),
\end{align*}
where $\mu_\pparams(\cdot), \sigma_{\pparams}^2(\cdot)$ are neural 
networks parameterized by $\pparams$. The proposal is also 
a function of the data, either explicitly as a part of the input to the 
functions or implicitly through the cost function for the parameters. 
These types of conditional 
distributions have recently been used with success in everything from 
image and speech generation \citep{gulrajani2017pixelvae,maddison2017} 
to causal inference \citep{louizos2017,krishnan2017structured}.

Below we discuss two classes of methods to learn the parameters 
$\pparams$, adaptive and variational methods, respectively.

\paragraph{Adaptive Sequential Monte Carlo} Adaptive methods are 
characterized by choosing a cost function that tries to fit directly 
the proposal distribution $\prop_\tm(\xp_{1:\Tm})$ to the target 
distribution $\nmod_\tm(\xp_{1:\Tm})$. This can either be done locally 
at each iteration for $\prop_\tm(\xp_\tm\given\xp_{1:\tm-1} \idxby 
\pparams)$ as in \eg \citet{cornebise2008adaptive}, or globally for 
$\prop_\Tm(\xp_{1:\Tm} \idxby \pparams)$ as in \eg
\citet{gu2015neural,paige2016inference}. Because we are mainly 
interested in the approximation of the final target distribution 
$\nmod_\Tm$, we will here focus our attention to methods that take a 
global approach to learning proposals. Below we focus on recent methods
developed in the machine learning literature. However, there is also an important body of work on adaptive \gls{SMC} samplers \citep{del2006sequential} stemming from the statistics community. For a comprehensive description of those methods we refer to \citet{jasra2011inference, fearnhead2013adaptive, schafer2013sequential,buchholz2018adaptive}.

Since we generally need the proposal distribution to cover areas of 
high probability under the target distribution, we need the proposal 
to be more diffuse than the target. This will ensure that the weights 
we assign will have finite variance and that our approximation gets more 
accurate. This means that using the cost function 
$\Kl(\prop_\Tm(\xp_{1:\Tm}\idxby \pparams)\|\nmod_\Tm(\xp_{1:\Tm}))$, like in standard 
variational inference \citep[\eg][]{blei2017variational}, would not result in 
a good proposal distribution. This cost function will encourage a 
proposal that is more concentrated than the target, and not less. 
Instead we focus our attention on the inclusive \gls{KL} divergence 
$\Kl(\nmod_\Tm(\xp_{1:\Tm})\|\prop_\Tm(\xp_{1:\Tm}\idxby \pparams))$ 
like in \eg \citet{gu2015neural} or \citet{paige2016inference}. The 
inclusive \gls{KL} divergence encourages a $\prop_\Tm$ that covers the 
high-probability regions of $\nmod_\Tm$.

We focus our exposition on adaptive methods to the case when 
$\nmod_\Tm(\xp_{1:\Tm})$ is the posterior 
distribution $p(\xp_{1:\Tm}\given \y_{1:\Tm})$. The cost function to 
minimize is the inclusive \gls{KL} divergence
\begin{align}
    &\Kl(p(\xp_{1:\Tm}\given \y_{1:\Tm})\|\prop_\Tm(\xp_{1:\Tm}\idxby \pparams)) 
    = \int p(\xp_{1:\Tm}\given \y_{1:\Tm}) \log 
    \frac{p(\xp_{1:\Tm}\given \y_{1:\Tm}) }{\prop_\Tm(\xp_{1:\Tm}\idxby 
    \pparams)} \dif \xp_{1:\Tm}\nonumber \\
    &=- \int p(\xp_{1:\Tm}\given \y_{1:\Tm}) \log \prop_\Tm(\xp_{1:\Tm}\idxby 
    \pparams)\dif \xp_{1:\Tm}+\text{const},
    \label{eq:adaptive:kl}
\end{align}
where ``$\text{const}$'' includes all terms constant with respect to 
the proposal distribution.
If we could compute the gradients of \cref{eq:adaptive:kl} with 
respect to $\pparams$ we could employ a standard gradient descent 
method to optimize it. Unfortunately the gradients, given by
\begin{align}
    &\gadapt = \grad_\pparams\Kl(p(\xp_{1:\Tm}\given 
    \y_{1:\Tm})\|\prop_\Tm(\xp_{1:\Tm}\idxby \pparams)) = \nonumber\\
    &-\Exp_{p(\xp_{1:\Tm}\given \y_{1:\Tm})}\left[\sum_{\tm=1}^\Tm
    \grad_\pparams \log \prop_\tm(\xp_\tm\given\xp_{1:\tm-1}\idxby 
    \pparams) 
    \right],
    \label{eq:adaptive:gradkl}
\end{align}
requires us to compute expectations with respect to the posterior 
distribution. Computing these types of expectations is the problem we 
try to solve with \gls{SMC} in the first place! Below we will detail two 
approaches to solve this problem. First, we consider a stochastic 
gradient method that uses \gls{SMC} to estimate the gradients. Then, 
we consider an alternative solution that uses \emph{inference amortization} 
combined with stochastic gradient methods.

In the first approach, we use our current best guess for the parameters $\pparams$ 
and then the \gls{SMC} procedure itself to approximate the gradient 
in \cref{eq:adaptive:gradkl}. Given this estimate of the gradient we 
can do an update step based on standard stochastic gradient descent 
methods. 
Suppose we have the iterate $\pparams^{n-1}$, we estimate the gradient in 
\cref{eq:adaptive:gradkl} using \cref{alg:esmc:smc} with proposals 
$\prop_\tm(\xp_\tm \given\xp_{1:\tm-1}\idxby \pparams^{n-1})$, and get
\begin{align}
    \gadapt^n \approx \widehat\gadapt^n \eqdef
    \sum_{\ip=1}^\Np \nwp_\Tm^\ip
    \grad_\pparams \log \prop_\Tm(\xp_{1:\Tm}^\ip\idxby 
    \pparams) \Big|_{\pparams = \pparams^{n-1}}.
    \label{eq:adaptive:gradest}
\end{align}
With stochastic gradient descent we update our iterate by
\begin{align}
    \pparams^n &= \pparams^{n-1} - \alpha_n \widehat\gadapt^n,
    \label{eq:adaptive:sgdupdate}
\end{align}
where $\alpha_n$ are a set of positive step-sizes, typically chosen 
such that $\sum_n \alpha_n = \infty$ and $\sum_n \alpha_n^2 < \infty$. 
Note that unlike standard stochastic optimization methods, the 
gradient estimate is \emph{biased} and convergence to a local minima 
of the cost function is not guaranteed. However, this approach has 
still shown to deliver useful proposal adaptation in practice \citep{gu2015neural}.

The other approach relies on amortizing inference 
\citep{gershman2014amortized,paige2016inference}. One view of 
inference amortization is as a procedure that does not
just learn a single optimal $\pparams^\star$ for the observed 
dataset $\y_{1:\Tm}$, but rather learns a \emph{mapping} $\lambda^\star(\cdot)$ 
from the data space to the parameter space. 
This mapping is optimized to make sure $\prop_\Tm(\xp_{1:\Tm}\idxby 
\lambda^\star(\y_{1:\Tm}))$ is a good approximation to $p(\xp_{1:\Tm}\given 
\y_{1:\Tm})$ for any $\y_{1:\Tm}$ that is likely under 
the probabilistic model $p(\y_{1:\Tm})$. Note that the way we parameterize $\prop_\Tm$ may differ from the above approach even though we use the same notation for the parameters $\lambda$.

In this setting instead of the \gls{KL} divergence in \cref{eq:adaptive:kl}, we consider minimizing
\begin{align}
    &\Kl\left(p(\xp_{1:\Tm},\y_{1:\Tm})\| p(\y_{1:\Tm}) \prop_\Tm(\xp_{1:\Tm}\idxby 
    \lambda(\y_{1:\Tm})) \right) \nonumber \\
    &= -\Exp_{p(\xp_{1:\Tm},\y_{1:\Tm})}\left[\log \prop_\Tm(\xp_{1:\Tm}\idxby 
    \lambda(\y_{1:\Tm}))\right] + 
    \text{const}.
    \label{eq:adaptive:amortized}
\end{align}

If we let $\lambda_\pfun(\cdot)$ be a parametric function with 
parameters $\pfun$, we can compute stochastic gradients of 
\cref{eq:adaptive:amortized} \wrt $\pfun$ with no need to resort to 
\gls{SMC}. A common choice is to let $\pparams_\pfun$ be defined by a neural 
network where $\pfun$ are the weights and biases. We define the gradient
\begin{align}
    \gadapt &= -\Exp_{p(\xp_{1:\Tm},\y_{1:\Tm})}\left[\grad_\pfun\log \prop_\Tm(\xp_{1:\Tm}\idxby 
    \lambda_\pfun(\y_{1:\Tm}))\right],
\end{align}
and estimate it, for the current iterate $\pfun^n$, using \gls{MC}
\begin{align}
&\gadapt^n \approx \widehat\gadapt^n \eqdef \nonumber \\
&-\grad_\pfun \log \prop_\Tm(\bar\xp_{1:\Tm}\idxby 
    \lambda_\pfun(\bar\y_{1:\Tm})) \Big|_{\pfun=\pfun^{n-1}}, \quad 
    (\bar\xp_{1:\Tm},\bar\y_{1:\Tm}) \sim p(\xp_{1:\Tm},\y_{1:\Tm}).
    \label{eq:adaptive:amortizedgradest}
\end{align}
Previous methods we have considered focus on proposals that try to 
emulate the posterior or the locally optimal proposal, both 
conditionally on the observed data $\y_{1:\Tm}$. However, the 
amortized inference approach, in this setting, learns proposals based on 
simulated data from the model $p(\xp_{1:\Tm},\y_{1:\Tm})$. This is 
performed offline before using the learned proposal and the real dataset 
for inference.

The amortized inference method follows the same procedure as above 
in \cref{eq:adaptive:sgdupdate} when updating $\pfun^n$, replacing the 
gradient with the expression from \cref{eq:adaptive:amortizedgradest}
\begin{align}
    \pfun^n &= \pfun^{n-1} - \alpha_n \widehat\gadapt^n.
    \label{eq:adaptive:amortizedup}
\end{align}
Unlike the above approach that uses \gls{SMC} to estimate the 
gradients, the amortized inference approach results in an unbiased 
approximation of the gradient. This means that using the update 
\cref{eq:adaptive:amortizedup} will ensure convergence to a local 
minima of its cost function \cref{eq:adaptive:amortized} by standard 
stochastic approximation results \citep{robbins1951}. On the other hand, this approach 
requires us to learn a proposal that works well for any dataset that 
could be generated by our model. This puts extra stress on the model 
to be accurate for the actual observed dataset to be able to learn a 
good proposal. For more thorough discussion of these topics, see 
\citet{paige2016inference}.

We summarize these two approaches to optimize the proposal in 
\cref{alg:adaptive}.
\begin{algorithm}[tb]
    \begin{description}
    \item[Stochastic Gradient] $\prop_\tm(\xp_\tm\given\xp_{1:\tm-1}\idxby 
    \pparams)$
    \begin{align*}
        \pparams^n &= \pparams^{n-1} - \alpha_n \widehat\gadapt^n,
    \end{align*}
    where $\widehat\gadapt^n$ is given in \cref{eq:adaptive:gradest}.
    
    \item[Amortized Inference] 
    $\prop_\tm(\xp_\tm\given\xp_{1:\tm-1}\idxby 
    \pparams_\pfun(\y_{1:\Tm}))$
    \begin{align*}
        \pfun^n &= \pfun^{n-1} - \alpha_n \widehat\gadapt^n,
    \end{align*}
    where $\widehat\gadapt^n$ is given in \cref{eq:adaptive:amortizedgradest}.
    \end{description}
    
    (The stepsizes satisfy $\alpha_n > 0$, 
    $\sum_n \alpha_n = \infty$, and $\sum_n \alpha_n^2 < \infty$.)
    \caption{Adaptive \gls{SMC}}
    \label{alg:adaptive}
\end{algorithm}

\paragraph{Variational Sequential Monte Carlo} The key idea in 
\gls{VSMC} \citep{naesseth18a, anh2018autoencoding, maddison2017} is to 
use the parametric distribution for our proposals, $\prop_\tm(\xp_\tm\given\xp_{1:\tm-1} \idxby 
\pparams)$, and 
optimize the fit in \gls{KL} divergence from the expected \gls{SMC} 
approximation $\Exp\left[\widehat\nmod_\Tm\right]$ to the target distribution 
$\nmod_\Tm$. The dependence on $\pparams$ enters implicitly in the 
expected 
\gls{SMC} approximation $\Exp\left[\widehat\nmod_\Tm\right]$ through the proposed 
samples, weights, and resampling step. In contrast to mimicking the 
locally optimal proposal 
distribution, this cost function takes into account the complete 
\gls{SMC} approximation and defines a coherent global objective 
function for it. 
Compared to the previously described adaptive \gls{SMC} methods, 
\gls{VSMC} optimizes the fit of the final \gls{SMC} distribution 
approximation to the true target, rather than the fit between the proposal to the target 
distribution. This means that we explicitly take into account the 
resampling steps that are a key part of the \gls{SMC} algorithm. 

Studying the marginal distribution of a single sample from 
$\widehat\nmod_\Tm$, \ie $\Exp\left[\widehat\nmod_\Tm\right]$, it is possible to show that 
\citep{naesseth18a} the \gls{KL} divergence from this distribution to 
the target distribution $\nmod_\Tm$ can be bounded from above
\begin{align}
    \Kl \left(\Exp\left[\widehat\nmod_\Tm(\xp_{1:\Tm}\idxby \pparams) 
    \right]\| 
    \nmod_\Tm(\xp_{1:\Tm}) \right) \leq -\Exp\left[\log 
    \frac{\widehat\Z_\Tm}{\Z_\Tm}\right],
    \label{eq:vsmc:ubound}
\end{align}
where the log-normalization constant estimate, \cf \cref{eq:smc:zhat}, is
\begin{align}
    \log\widehat\Z_\Tm &= \sum_{\tm=1}^\Tm \log\left( \frac{1}{\Np}
    \sum_{\ip=1}^\Np \uwp_\tm\left(\xp_{1:\tm}^\ip\idxby\pparams\right)
    \right).
    \label{eq:vsmc:lognorm}
\end{align}
The expectation is taken with respect to all random variables 
generated by the \gls{SMC} algorithm. Because $\log\Z_\Tm$ does not 
depend on the parameters $\pparams$, minimizing \cref{eq:vsmc:ubound} 
is equivalent to maximizing
\begin{align}
    \Exp\left[\log \widehat \Z_\Tm\right] &=
    \Exp\left[ \sum_{\tm=1}^\Tm \log\left( \frac{1}{\Np}
    \sum_{\ip=1}^\Np \uwp_\tm\left(\xp_{1:\tm}^\ip\idxby\pparams\right)
    \right)\right].
    \label{eq:vsmc:lbound}
\end{align}
The \gls{KL} divergence $\Kl(\Exp[\widehat\nmod_\Tm]\|\nmod_\Tm)$ 
is non-negative, which means that \cref{eq:vsmc:lbound} is a lower bound to the 
log-normalization constant $\log \Z_\Tm$. This is why 
\cref{eq:vsmc:lbound} is typically known as an \acrlong{ELBO} in the 
variational inference literature. By maximizing the cost function in 
\cref{eq:vsmc:lbound}, with respect to $\pparams$, we can find a 
proposal that fits the complete \gls{SMC} distribution to the target 
distribution. The main issue is that evaluating and computing the 
gradient of this cost function is intractable, so we need to resort to 
approximations.

We assume that the proposal distributions $\prop_\tm$ are 
\emph{reparameterizable} \citep{Kingma2014, rezende2014, ruiz2016generalized,naesseth2017reparameterization}. 
This means we can replace 
simulating $\xp_\tm\given\xp_{1:\tm-1} \sim 
\prop_\tm(\cdot\given\xp_{1:\tm-1})$ by 
\begin{align}
    \xp_\tm &=\fun_\tm(\xp_{1:\tm-1}, \eps_\tm \idxby\pparams), \quad 
    \eps_\tm \sim p(\eps),
    \label{eq:vsmc:reparam}
\end{align}
where the distribution of the random variable $\eps_\tm$ is independent 
of the parameters $\pparams$. If we 
further assume that $\fun_\tm$ is differentiable we can use 
\begin{align}
    \grad_\pparams \Exp\left[\log \widehat \Z_\Tm\right] &\approx
    \Exp\left[ \sum_{\tm=1}^\Tm \sum_{\ip=1}^\Np \nwp_\tm^\ip \grad_\pparams
    \log\uwp_\tm\left(\xp_{1:\tm}^\ip\idxby\pparams\right)\right] \defeq 
    \gvsmc,
    \label{eq:vsmc:gvsmc}
\end{align}
where $\grad_\pparams\log\uwp_\tm(\cdot)$ can be computed using \eg automatic 
differentiation, replacing $\xp_{1:\tm}^\ip$ with its defintion 
through \cref{eq:vsmc:reparam}. The approximation follows from 
ignoring the gradient part that results from the resampling step, 
which has been shown to work well in practice \citep{naesseth18a, anh2018autoencoding, maddison2017}.

Just like for the adaptive \gls{SMC} methods, \cref{eq:vsmc:gvsmc} 
suggests a stochastic gradient method to optimize $\pparams$. Given an 
iterate $\pparams^{n-1}$, we compute the gradient estimate by running 
\gls{SMC} with proposals 
$\prop_\tm(\xp_\tm \given\xp_{1:\tm-1}\idxby \pparams^{n-1})$ and 
evaluate
\begin{align}
    \gvsmc &\approx \widehat\gvsmc^n = \sum_{\tm=1}^\Tm \sum_{\ip=1}^\Np \nwp_\tm^\ip \grad_\pparams
    \log\uwp_\tm\left(\xp_{1:\tm}^\ip\idxby\pparams\right)
    \Big|_{\pparams=\pparams^{n-1}}.
    \label{eq:vsmc:gradest}
\end{align}
With stochastic gradient ascent we update our iterate by
\begin{align}
    \pparams^n &= \pparams^{n-1} + \alpha_n \widehat\gvsmc^n,
    \label{eq:vsmc:sgdupdate}
\end{align}
where $\alpha_n$ are positive step-sizes, chosen such that 
$\sum_n \alpha_n = \infty$ and $\sum_n \alpha_n^2 < \infty$. 

\begin{algorithm}[tb]
    \emph{Until convergence:}
    
    Run \cref{alg:esmc:smc} with proposals $\prop_\tm(\xp_\tm\given\xp_{1:\tm-1}\idxby 
    \pparams^{n-1})$.
    
    Update parameters $\pparams^n = \pparams^{n-1} + \alpha_n \widehat\gvsmc^n$, 
    where $\widehat\gvsmc^n$ is given in \cref{eq:vsmc:gradest}.
    
    (The stepsizes satisfy $\alpha_n > 0$, 
    $\sum_n \alpha_n = \infty$, and $\sum_n \alpha_n^2 < \infty$.)
    \caption{Variational \gls{SMC}}
    \label{alg:esmc:vsmc}
\end{algorithm}
We summarize the \gls{VSMC} algorithm to adapt the proposals in 
\cref{alg:esmc:vsmc}.

\section{Adapting the Target Distribution}\label{sec:target}
A less well-known design aspect of \gls{SMC} algorithms is the target 
distribution itself. If we are mainly interested in the final 
distribution $\nmod_{\Tm}$, we are free to choose the intermediate 
target distributions $\nmod_{\tm}$ to maximize the accuracy of our 
estimate $\widehat\nmod_{\Tm}$. By making use of information 
such as future observations, we can change the 
target distribution to take this into account. This leads to so-called 
\emph{auxiliary} or \emph{twisted} \gls{SMC} methods 
\citep{guarniero2017iterated,heng2017controlled}.

\subsection{Twisting the Target}
Even if we could simulate exactly from the locally optimal proposal 
distribution \cref{eq:prop:localoptimal}, we still would not be 
getting exact samples from our target distribution. The reason for 
this is that when sampling and weighting our particles at iteration 
$\tm$ we do not take into account potential future iterations. These 
future iterations can add new information regarding earlier latent variables. 
For instance, in an \gls{LVM} (\cf \cref{eq:nmlvm:filteringtarget}) a 
natural choice of target distributions is
\begin{align*}
    \umod_\tm(\xp_{1:\tm}) &= p(\xp_1) 
    \gmod_1(\y_1\given\xp_1)\prod_{\km=2}^\tm 
    \fmod_\km(\xp_\km\given\xp_{1:\km-1}) 
    \gmod_\km(\y_\km\given\xp_{1:\km}).
\end{align*}
However, with this choice we do not take future observations 
$\y_{\tm+1}, \ldots, \y_\Tm$ into account at iteration $\tm$.
The \gls{SMC} approximation for earlier iterations has finite 
support (represented by the particles $\xp_{1:\tm-1}^\ip)$ and so the 
only way we can incorporate this is by reweighting and resampling. These 
two operations will typically impoverish our approximation for 
the full states $\xp_{1:\tm}$, a phenomena related to the path 
degeneracy discussed in \cref{cha:smc}. The key idea in twisting (or 
\emph{tilting}) the target distribution is to change our intermediate target 
distributions to take into account information from future iterations 
already at the current step. These types of approaches also have a 
connection to so-called lookahead strategies \citep{lin2013lookahead} 
and block sampling \citep{doucet2006efficient} for \gls{SMC}. 

The optimal target distribution, under the assumption that we are using the locally 
optimal proposal, is to let the target at each iteration be the 
marginal distribution of the final iteration's target, \ie
\begin{align}
    \nmod_\tm^\star(\xp_{1:\tm}) &= \nmod_{\Tm}(\xp_{1:\tm}).
    \label{eq:optimal:target}
\end{align}
With this choice all samples from $\xp_{1:\Tm}\sim\widehat\nmod_\Tm$ 
are perfect samples from $\nmod_\Tm$. We can easily see this if we write 
out the optimal proposal for this sequence of targets
\begin{align*}
    \prop_\tm^\star(\xp_\tm \given \xp_{1:\tm-1}) &= 
    \nmod_\tm^\star(\xp_\tm\given 
    \xp_{1:\tm-1}) = \nmod_\Tm(\xp_\tm\given 
    \xp_{1:\tm-1}) = p(\xp_{\tm} \given \xp_{1:\tm-1}, \y_{1:\Tm}),
\end{align*}
where the final equality corresponds to the case when
$\nmod_\Tm$ is the posterior distribution 
$p(\xp_{1:\Tm}\given\y_{1:\Tm})$ for an \gls{LVM}. This means that the locally optimal 
proposal distribution is no longer only locally optimal, it is in fact 
optimal in a global sense. 

In the following discussion we will assume that the final unnormalized 
target distribution $\umod_\Tm(\xp_{1:\Tm})$ 
can be split into a product of factors:
\begin{align}
	\umod_\Tm(\xp_{1:\Tm}) &= \prod_{\tm=1}^\Tm  \ufmod_\tm(\xp_{1:\tm}).
	\label{eq:twist:target-factorizes}
\end{align}
With this form of the joint distribution of data and latent variables, 
we choose the following structure for our unnormalized target 
distributions:
\begin{align}
    \umod_\tm(\xp_{1:\tm}) &= \umod_{\tm-1}(\xp_{1:\tm-1}) 
    \ufmod_\tm(\xp_{1:\tm}) 
    \frac{\tpot_\tm(\xp_{1:\tm})}{\tpot_{\tm-1}(\xp_{1:\tm-1})},
    \label{eq:twist:targetstruct}
\end{align}
where $\tpot_\tm \geq 0$ are our so-called \emph{twisting potentials}, with 
$\tpot_0\equiv \tpot_\Tm \equiv 1$. 
It is easy to check that this recursive definition is consistent with our assumed factorization of the final target $\umod_\Tm(\xp_{1:\Tm})$ in \cref{eq:twist:target-factorizes}.
The  target distribution structure postulated in \cref{eq:twist:targetstruct} 
can be directly applied in our basic \gls{SMC} method in \cref{alg:esmc:smc} 
without any changes.
 
To deduce the optimal twisting potentials $\tpot_\tm^\star$ we can 
make use of the property that the optimal targets must fulfill
\begin{align}
    \umod_\tm^\star(\xp_{1:\tm}) &= \int\umod_{\tm+1}^\star(\xp_{1:\tm+1}) 
    \dif \xp_{\tm+1}, \quad \tm = 1,\ldots,\Tm-1.
    \label{eq:twist:equalmarginals}
\end{align}
This follows from \cref{eq:optimal:target}, each target 
$\nmod_\tm^\star$ 
should be the marginal distribution of $\xp_{1:\tm}$ under the final 
target distribution $\nmod_\Tm$. By replacing $\umod_\tm^\star$ in 
\cref{eq:twist:equalmarginals} with the definition
from \cref{eq:twist:targetstruct} we get
\begin{align}
    \tpot_\tm^\star(\xp_{1:\tm}) &= \int \ufmod_{\tm+1}(\xp_{1:\tm+1})
    \tpot_{\tm+1}^\star(\xp_{1:\tm+1}) \dif \xp_{\tm+1}.
    \label{eq:twist:opttwist}
\end{align}
While \cref{eq:twist:opttwist} is typically not available analytically for practical 
applications, we will see in \cref{sec:twist:learn} that it can serve as 
a guideline for designing tractable twisting potentials.

Below we give a few concrete examples on what the optimal potentials 
might look like for both non-Markovian \gls{LVM}s and conditionally 
independent models with tempering. 

\paragraph{Non-Markovian Latent Variable Model}
The non-Markovian \gls{LVM} can be described by a transition \gls{PDF} 
$\fmod_\tm$ and an observation \gls{PDF} $\gmod_\tm$
and follows the structure in \cref{eq:twist:target-factorizes} by defining
\begin{align}
	\ufmod_\tm(\xp_{1:\tm}) = \fmod_\tm(\xp_\tm\given\xp_{1:\tm-1})
	\gmod_\tm(\y_\tm\given\xp_{1:\tm},\y_{1:\tm-1}).
    \label{eq:twist:nmlvm}%
\end{align}
For this model we can rewrite \cref{eq:twist:opttwist} as follows
\begin{align}
    &\tpot_\tm^\star(\xp_{1:\tm}) = 
    \Exp_{\fmod_{\tm+1}(\xp_{\tm+1}\given\xp_{1:\tm})}\left[
    \gmod_{\tm+1}(\y_{\tm+1}\given\xp_{1:\tm+1},\y_{1:\tm})
    \tpot_{\tm+1}^\star(\xp_{1:\tm+1})
    \right] \nonumber \\
    &= \Exp_{\prod_{\km=\tm+1}^\Tm\fmod_\km(\xp_\km\given\xp_{1:\km-1})}\left[
    \prod_{\lm=\tm+1}^\Tm \gmod_\lm(\y_\lm\given\xp_{1:\lm},\y_{1:\lm-1})
    \right] = p(\y_{\tm+1:\Tm}\given \xp_{1:\tm}),
    \label{eq:twist:nllvm}
\end{align}
where the final expression is the predictive likelihood of 
$\y_{\tm+1:\Tm}$ given the latent variables $\xp_{1:\tm}$. In 
\cref{ex:twist:analytical} we derive the optimal twisting potentials 
for our Gaussian non-Markovian \gls{LVM}.

\begin{example}[Analytical Twisting Potential]
Our running example is a non-Markovian \gls{LVM} with an 
analytical optimal twisting potential $\tpot_\tm^\star$. This is because 
the joint distribution is Gaussian, and thus the integrals we need to 
compute to construct it (\cf \cref{eq:twist:nmlvm}) are tractable. 
We can rewrite the equation for the observations $\y_\tm$ for each 
$\tm$ only as a function of $\xp_{1:\lm}$, $\en_\tm$ and 
$\vn_{\lm+1:\tm}$, denoted by $\y_{\tm\given\lm}$, as follows
\begin{align}
    \y_{\tm\given\lm} &= \left(\sum_{\km=\lm+1}^\tm 
    \beta^{\tm-\km} \phi^{\km-\lm}\right)\xp_\lm +
    \sum_{\km=1}^\lm \beta^{\tm-\km} \xp_\km + \en_\tm + 
    \sum_{\km=\lm+1}^\tm \gamma_\km \vn_\km,
    \label{eq:twist:running:opt}
\end{align}
where $\gamma_\km$ is given by 
\begin{align*}
    \gamma_\km &= \sum_{\mm=\km}^{\tm}\beta^{\tm-\mm} \phi^{\mm-\km}.
\end{align*}
Because we know that 
$\tpot_\tm^\star(\xp_{1:\tm})=p(\y_{\tm+1:\Tm}\given \xp_{1:\tm})$ we 
get the optimal value by considering the distribution of 
$(\y_{\tm+1\given\tm},\ldots,\y_{\Tm\given\tm})^\top$. Since $\en_\tm$ 
and $\vn_\tm$ are independent Gaussian, by \cref{eq:twist:running:opt} 
the predictive likelihood is a correlated multivariate Gaussian. 
\label{ex:twist:analytical}
\end{example}

\paragraph{Conditionally Independent Models}
Another class of probabilistic models we discussed in \cref{sec:seqmodels} 
was conditionally independent models. In the notation of 
\cref{eq:twist:targetstruct} we can express this class of models as 
\begin{align}
    \ufmod_\tm(\xp_{1:\tm}) &= \smod_\tm(\xp_{\tm-1}\given\xp_{\tm})
    \frac{p(\xp_\tm) \gmod_\tm(\xp_\tm,\ym)}
    {p(\xp_{\tm-1}) \gmod_{\tm-1}(\xp_{\tm-1},\ym)}, 
    \quad\tm = 2,\ldots,\Tm, 
    \label{eq:twist:annealing}%
\end{align}
where $p(\xp)$ is the prior distribution on the (static) latent variable, $\smod_\tm$ is a backward kernel,
and $\gmod_\tm$ is a \emph{potential} depending on the observed data $\ym$. 
 We initialize by $\ufmod_1(\xp_1) = p(\xp_1) \gmod_1(\xp_1,\ym)$.
 The potential $\gmod_\tm$ usually takes 
either of the two following forms:
\begin{align*}
    \gmod_\tm(\xp_\tm,\ym) &= \prod_{\km=1}^\tm p(\y_\km \given 
    \xp_\tm), \quad \text{(data tempering)} \\
    \gmod_\tm(\xp_\tm,\ym) &= \left(\prod_{\km=1}^{\Km} p(\y_\km \given 
    \xp_\tm)\right)^{\tmp_\tm}, \quad \text{(likelihood tempering)} 
\end{align*}
where $0=\tmp_1 < \tmp_2 < \ldots < \tmp_\Tm = 1$. The 
backward kernel $\smod_\tm$ is an artificially introduced distribution to 
enable the use of \gls{SMC} methods, requiring approximating a space 
of increasing dimension. See \cref{sec:seqmodels} for more examples of 
useful sequence models in annealing methods, or 
\citet{del2006sequential} for a more thorough treatment of the subject. 

When we replace $\ufmod_\tm$ 
in  \cref{eq:twist:opttwist} with its 
definition from \cref{eq:twist:annealing} we get
\begin{align}
    &\tpot_\tm^\star(\xp_{1:\tm}) = 
    \int \frac{p(\xp_{\tm+1}) \gmod_{\tm+1}(\xp_{\tm+1},\ym)}
    {p(\xp_\tm)\gmod_\tm(\xp_\tm,\ym)} 
    \smod_{\tm+1}(\xp_\tm \given \xp_{\tm+1})
    \tpot_{\tm+1}^\star(\xp_{1:\tm+1}) \dif \xp_{\tm+1} \nonumber \\
    &=  
    \frac{1}{p(\xp_\tm)\gmod_\tm(\xp_\tm,\ym)}
    \int p(\xp_\Tm) g_\Tm(\xp_\Tm, \ym)
    \prod_{\km=\tm+1}^{\Tm} \smod_\km(\xp_{\km-1}\given\xp_{\km})
    \dif \xp_{\tm+1:\Tm},
\end{align}
where we can see that because of the (reverse) Markov structure in 
the annealing model from \cref{eq:twist:annealing}, the optimal twisting 
potential only depends on the current value of $\xp_\tm$, \ie 
$\tpot_\tm^\star(\xp_{1:\tm}) = \tpot_\tm^\star(\xp_{\tm})$. This will 
hold true also for the \gls{SSM}, the Markovian special case of 
\cref{eq:twist:nmlvm} where $\fmod_\tm(\xp_\tm\given\xp_{1:\tm-1}) = 
\fmod_\tm(\xp_\tm\given\xp_{\tm-1})$ and 
$\gmod_\tm(\y_\tm\given\xp_{1:\tm},\y_{1:\tm-1}) = 
\gmod_\tm(\y_\tm\given\xp_{\tm})$.

\subsection{Designing the Twisting Potentials}\label{sec:twist:learn}
The optimal twisting potentials are typically not tractable, requiring 
us to solve integration problems that might be as difficult to compute 
as finding the posterior itself. However, they can give insight into how to 
design and parameterize approximate twisting functions. First, we study a 
certain class of models that admit a \emph{locally optimal} twisting potential, 
which gives rise to the so-called \emph{fully adapted} \gls{SMC} 
\citep{pitt1999filtering,johansen2008note}. After discussing the 
locally optimal choice, we then move on to discuss a general approach 
to learning twisting potentials for \gls{SMC} algorithms.

\paragraph{Locally Optimal Twisting Potential} If we assume that the 
twisting potentials only satisfy the optimal twisting equation, 
\cref{eq:twist:opttwist}, for one iteration we get a locally optimal 
twisting potential
\begin{align}
    \tpot_\tm(\xp_{1:\tm}) &= \int 
    \ufmod_{\tm+1}(\xp_{1:\tm+1})
    \dif \xp_{\tm+1}.
    \label{eq:twisting:localopt}
\end{align}
The target distribution based on these potentials will allow us to 
adapt for information from one step into the future.

Furthermore, we combine the locally optimal twisting potentials with the 
proposals based only on factors up until iteration $\tm$, \ie
\begin{align}
    \prop_\tm(\xp_\tm\given\xp_{1:\tm-1}) \propto \prod_{\km=1}^\tm \ufmod_{\km}(\xp_{1:\km})
	\quad    \Longrightarrow 	\quad
    \prop_\tm(\xp_\tm\given\xp_{1:\tm-1}) = \frac{\ufmod_{\tm}(\xp_{1:\tm})}{\tpot_{\tm-1}(\xp_{1:\tm-1})}.
    \label{eq:proptwist:localopt:proposal}
\end{align}
With these choices of twisting functions and proposals, based on the twisted target structure 
\cref{eq:twist:targetstruct}, we obtain the fully adapted \gls{SMC} \citep{pitt1999filtering,johansen2008note}
with corresponding weights
\begin{align}
    \uwp_\tm(\xp_{1:\tm}) &= 
    \frac{\umod_\tm(\xp_{1:\tm})}{\umod_{\tm}(\xp_{1:\tm-1}) 
    q_\tm(\xp_\tm\given \xp_{1:\tm-1})} =
    \tpot_\tm(\xp_{1:\tm}).
    \label{eq:proptwist:localopt:weights}
\end{align}
Note that this perspective on fully adaptive \gls{SMC} approximates 
the ``predictive target''
\begin{align*}
	\umod_\tm(\xp_{1:\tm}) &=
	\underbrace{\int 
		\ufmod_{\tm+1}(\xp_{1:\tm+1})\dif \xp_{\tm+1}}_{\tpot_\tm(\xp_{1:\tm})}
	\cdot
	\prod_{\km=1}^\tm  \ufmod_\km(\xp_{1:\km}), 
\end{align*}
including information from iteration $\tm+1$ already at iteration 
$\tm$. In some situations we might be interested in the ``filtering 
target'' $\prod_{\km=1}^\tm  \ufmod_\km(\xp_{1:\km}) $. If so, we can make use of the 
approach described above to generate the particles $\xp_{1:\tm}^\ip$. 
Then to estimate the filtering target we simply average the particles 
$\xp_{1:\tm}^\ip$, ignoring the weights in 
\cref{eq:proptwist:localopt:weights}, 
accounting for the discrepancy between the filtering and predictive targets.

\begin{example}[Locally Optimal Twisting for Non-Markovian LVM]
	Recall the non-Markovian \gls{LVM} with factors given by (see \cref{eq:twist:nmlvm})
	\(
		\ufmod_\tm(\xp_{1:\tm}) = \fmod_\tm(\xp_\tm\given\xp_{1:\tm-1})
		\gmod_\tm(\y_\tm\given\xp_{1:\tm},\y_{1:\tm-1})
		= p(\y_\tm, \xp_\tm \given \xp_{1:\tm-1},\y_{1:\tm-1}).
	\)
	For this model the locally optimal twisting potential \cref{eq:twisting:localopt} is given by
	\begin{align*}
		\tpot_\tm(\xp_{1:\tm}) = 
		\Exp_{\fmod_{\tm+1}(\xp_{\tm+1}\given\xp_{1:\tm})}\left[
		\gmod_{\tm+1}(\y_{\tm+1}\given\xp_{1:\tm+1},\y_{1:\tm})
		\right] = p(\y_{\tm+1}\given \xp_{1:\tm}).
	\end{align*}
	The proposal in \cref{eq:proptwist:localopt:proposal} becomes
	\begin{align*}		
		\prop_\tm(\xp_\tm\given\xp_{1:\tm-1}) = 
		\frac{p(\y_\tm, \xp_\tm \given \xp_{1:\tm-1},\y_{1:\tm-1})}{p(\y_{\tm}\given \xp_{1:\tm-1})}
		= p(\xp_\tm \given \xp_{1:\tm-1},\y_{1:\tm})
	\end{align*}	
	with weights according to \cref{eq:proptwist:localopt:weights} given by
	\(
		\uwp_\tm(\xp_{1:\tm}) = p(\y_{\tm+1}\given \xp_{1:\tm}).
	\)
	This can be recognized as a one-step look-ahead strategy, incorporating the observation $\y_{\tm+1}$ already at iteration $\tm$. This is indeed how the fully adaptive \gls{SMC} is often presented for \glspl{SSM} and non-Markovian \glspl{LVM}.
\end{example}

\begin{remark}
The proposal used above is not the locally optimal proposal 
distribution for our (predictive) target distribution. The locally optimal proposal (see \cref{eq:prop:localoptimal}) for the target \cref{eq:twist:targetstruct} is
\begin{align*}
	\prop_\tm(\xp_\tm \given \xp_{1:\tm-1}) &\propto 
	\ufmod_{\tm}(\xp_{1:\tm}) 
	\tpot_\tm(\xp_{1:\tm}),
\end{align*}
which would result in weights
\begin{align}
	\uwp_\tm(\xp_{1:\tm}) &= 
	\frac{\umod_\tm(\xp_{1:\tm})}{\umod_{\tm}(\xp_{1:\tm-1}) 
		q_\tm(\xp_\tm\given \xp_{1:\tm-1})} =
	\frac{\int 
		\ufmod_{\tm}(\xp_{1:\tm}) 
		\tpot_\tm(\xp_{1:\tm}) \dif \xp_\tm}
	{\tpot_{\tm-1}(\xp_{1:\tm-1})}.
	\label{eq:twisted:localoptlocalopt}
\end{align}
Just like in the non-twisted case the weights for the locally optimal 
proposal are independent of the 
actual samples we generate at iteration $\tm$ of the \gls{SMC} 
algorithm. However, even if the potentials in 
\cref{eq:twisting:localopt} are available, the integration in 
\cref{eq:twisted:localoptlocalopt} is not necessarily tractable.
\end{remark}

Except for a few special cases the locally optimal twisting potential is not 
available, therefore we resort to approximating it. We can extend the 
ideas from \cref{sec:proposal} and either make an approximation to the 
locally optimal potential, or we can learn a full twisting potential.

\paragraph{Learning Twisting Potentials} By parameterizing a class of 
positive functions and a cost function, we can effectively learn 
useful twisting potentials for the application at hand. From 
\cref{eq:twist:opttwist} we know that the optimal twisting potentials 
must satisfy a recursive relationship. We can make use of this 
property as we define the cost function and the twisting potentials.

We explain the approach by \citet{guarniero2017iterated}, but 
generalize slightly to allow for non-Markovian \gls{LVM}s. This means that
our model is defined by
$\ufmod_\tm = \fmod_\tm(\xp_\tm\given\xp_{1:\tm-1})\gmod_\tm(\y_\tm\given \xp_{1:\tm})$,
\cf \cref{eq:twist:nmlvm}. 
Furthermore, we define the twisting potentials $\tpot_\tm(\xp_{1:\tm})$ implicitly 
as expectations with respect to $\fmod_{\tm+1}$ of positive functions $\tpsi_{\tm+1}(\xp_{1:\tm+1} 
\idxby \tparams_{\tm+1})$ with parameters $\tparams_{\tm+1}$, \ie
\begin{align}
    \tpot_\tm(\xp_{1:\tm} \idxby \tparams_{\tm+1}) &\eqdef 
    \Exp_{\fmod_\tm(\xp_{\tm+1}\given\xp_{1:\tm})}\left[
    \tpsi_{\tm+1}(\xp_{1:\tm+1} \idxby \tparams_{\tm+1})
    \right], \quad \tm = 1,\ldots,\Tm-1,
    \label{eq:twisted:iapf:potential}
\end{align}
We further assume that this expectation can be computed analytically. With this setup we focus 
on learning the $\tpsi_\tm$'s instead. Using the optimality condition 
from \cref{eq:twist:opttwist} with twisting potentials defined by 
\cref{eq:twisted:iapf:potential}, we get the recursive approximation 
criteria for the $\tpsi_\tm$'s
\begin{subequations}
    \begin{align}
        \tpsi_\Tm(\xp_{1:\Tm} \idxby \tparams_\Tm) &\approx \gmod_\Tm(\y_\Tm\given \xp_{1:\Tm}), \\
        \tpsi_\tm(\xp_{1:\tm} \idxby \tparams_\tm) &\approx \gmod_\tm(\y_\tm\given \xp_{1:\tm}) 
        \Exp_{\fmod_\tm(\xp_{\tm+1}\given\xp_{1:\tm})}\left[
        \tpsi_{\tm+1}(\xp_{1:\tm+1} \idxby \tparams_{\tm+1})\right].
    \end{align}
    \label{eq:twisted:iapf:opttwist}%
\end{subequations}
Because we assumed that 
$\Exp_{\fmod_\tm(\xp_{\tm}\given\xp_{1:\tm-1})}\left[\tpsi_{\tm}(\xp_{1:\tm} 
\idxby \tparams_\tm)\right]$ is available analytically, and 
\cref{eq:twisted:iapf:opttwist} is telling us to view it as an 
approximation of $p(\y_{\tm:\Tm}\given\xp_{1:\tm-1})$, it makes sense 
to use as the proposal
\begin{align}
    \prop_\tm(\xp_\tm\given \xp_{1:\tm-1}\idxby\tparams_\tm) &= 
    \frac{\fmod_\tm(\xp_{\tm}\given\xp_{1:\tm-1})\tpsi_{\tm}(\xp_{1:\tm} \idxby \tparams_\tm)}
    {\Exp_{\fmod_\tm(\xp_{\tm}\given\xp_{1:\tm-1})}\left[\tpsi_{\tm}(\xp_{1:\tm} \idxby \tparams_\tm)\right]}.
    \label{eq:twisted:iapf:prop}
\end{align}
With this choice of proposal we tie together the
parameters for the proposal and the
potentials. The corresponding weights in the \gls{SMC} 
algorithm are
\begin{align}
    \uwp_\tm(\xp_{1:\tm}\idxby\tparams_{\tm:\tm+1}) &= \frac{\gmod_\tm(\y_\tm\given\xp_{1:\tm}) 
    \tpot_\tm(\xp_{1:\tm} \idxby \tparams_{\tm+1})}{\tpsi_{\tm}(\xp_{1:\tm} 
    \idxby \tparams_\tm)}.
    \label{eq:twisted:iapf:weights}
\end{align}
What remains is how to use \cref{eq:twisted:iapf:opttwist} to actually 
learn a useful set of potentials. To do this we follow the approach by 
\citet{guarniero2017iterated} which uses an iterative refinement 
process: initialize parameters, run \gls{SMC} with the current 
parameters, update parameters based on the current \gls{SMC} 
approximation $\widehat\nmod_\tm$, repeat. The update step solves for 
$\tparams_\tm^n$ by recursively minimizing
\begin{subequations}
    \begin{align}
        &\tparams_\Tm^n = \argmin_{\tparams_\Tm} \sum_{\ip=1}^\Np \nwp_\Tm^\ip 
        \left(\tpsi_\Tm(\xp_{1:\Tm}^\ip\idxby 
        \tparams_{\Tm})-\gmod_\Tm(\y_\Tm\given\xp_{1:\Tm}^\ip)
        \right)^2, \\
        &\tparams_\tm^n = \argmin_{\tparams_\tm} \sum_{\ip=1}^\Np \nwp_\tm^\ip 
        \left(\tpsi_\tm(\xp_{1:\tm}^\ip\idxby \tparams_{\tm})-\gmod_\tm(\y_\tm\given\xp_{1:\tm}^\ip)
        \tpot_\tm(\xp_{1:\tm}^\ip\idxby \tparams_{\tm+1}^n)
        \right)^2,
    \end{align}
    \label{eq:iapf:update}%
\end{subequations}
where the last step is for $\tm =\Tm-1, \ldots, 1$. With the updated set 
of parameters, we re-run \gls{SMC} and repeat the updates in \cref{eq:iapf:update}
if a stopping criteria has not been met. 
\citet{guarniero2017iterated} proposes a stopping criteria based on 
the normalization constant estimate $\widehat\Z_\Tm$.

We summarize the iterated auxiliary \gls{SMC} method explained above in 
\cref{alg:iteratedAPF}, and give an example setup for a stochastic 
\gls{RNN} in \cref{ex:twisting:srnn}.
\begin{algorithm}[tb]
    Initialize $\tparams_{1}^0, \ldots, \tparams_{\Tm}^0$.
    
    \While{(stopping criteria not met)}{
        Run \cref{alg:esmc:smc} with current parameters $\tparams_{1}^{n-1}, 
        \ldots, \tparams_{\Tm}^{n-1}$. The proposal and weights are 
        given by \cref{eq:twisted:iapf:prop} and 
        \cref{eq:twisted:iapf:weights}, respectively.
        
        Update the parameters to $\tparams_{1}^{n}, 
        \ldots, \tparams_{\Tm}^{n}$ by solving \cref{eq:iapf:update}.
    }
    \caption{Iterated auxiliary \gls{SMC}}
    \label{alg:iteratedAPF}
\end{algorithm}

\begin{example}[Stochastic Recurrent Neural Network]
A stochastic \gls{RNN} is a non-Markovian \gls{LVM} where the 
parameters of the transition and observation models are defined by 
\gls{RNN}s. A common example is using the conditional Gaussian 
distribution to define the transition \gls{PDF}
\begin{align*}
    \fmod_\tm(\xp_\tm\given\xp_{1:\tm-1}) &= \Norm\left(\xp_\tm\given 
    \mu_\tm(\xp_{1:\tm-1}), \Sigma_\tm(\xp_{1:\tm-1})\right),
\end{align*}
where the functions $\mu_\tm(\cdot), \Sigma_\tm(\cdot)$ are defined by
\gls{RNN}s. This model together with a Gaussian-like definition for 
the potentials satisfies the criteria for using 
\cref{alg:iteratedAPF}. We define $\tpsi_\tm$ as follows
\begin{align*}
    \tpsi_\tm(\xp_{1:\tm}\idxby\tparams_\tm) &= 
    \exp\left(-\frac{1}{2}\xp_\tm^\top 
    \Lambda_\tm(\xp_{1:\tm-1})\xp_\tm + \iota_\tm(\xp_{1:\tm-1})^\top 
    \xp_\tm + c_\tm(\xp_{1:\tm-1})\right),
\end{align*}
where the functions $\Lambda_\tm(\cdot), \iota_\tm(\cdot), c_\tm(\cdot)$ 
depend on the parameters $\tparams_\tm$---for notational brevity we have not made this dependence explicit.
The proposal (\cf \cref{eq:twisted:iapf:prop}) is given by
\begin{align*}
    \prop_\tm(\xp_\tm\given\xp_{1:\tm-1}\idxby\tparams_\tm) &= \Norm\left(\xp_\tm\given 
    \widehat\mu_\tm, \widehat\Sigma_\tm\right), \\
    \widehat\Sigma_\tm  &= 
    \left(\Sigma_\tm(\xp_{1:\tm-1})^{-1}+\Lambda_\tm(\xp_{1:\tm-1})\right)^{-1}, \\
    \widehat\mu_\tm &= \widehat\Sigma_\tm 
    \left(\Sigma_\tm(\xp_{1:\tm-1})^{-1}\mu_\tm(\xp_{1:\tm-1}) + 
    \iota_\tm(\xp_{1:\tm-1})\right),
\end{align*}
and the twisting potential
\begin{align*}
    &\tpot_{\tm-1}(\xp_{1:\tm-1}\idxby\tparams_{\tm}) = 
    \Exp_{\fmod_\tm(\xp_\tm\given\xp_{1:\tm-1})}
    \left[\tpsi_\tm(\xp_{1:\tm}\idxby\tparams_\tm)\right] = \sqrt{\frac{\det\left(\widehat\Sigma_\tm\right)}
    {\det\left(\Sigma_\tm(\xp_{1:\tm-1})\right)}}\\
    &\cdot \exp\left(\frac{1}{2}\widehat\mu_\tm^\top 
    \widehat\Sigma_\tm^{-1} \widehat\mu_\tm - 
    \frac{1}{2}\mu_\tm(\xp_{1:\tm-1})^\top \Sigma_\tm(\xp_{1:\tm-1})^{-1}\mu_\tm(\xp_{1:\tm-1})
    +c_\tm(\xp_{1:\tm-1})\right).
\end{align*}
\label{ex:twisting:srnn}
\end{example}

Twisting the target distribution is an area of active research 
\citep{guarniero2017iterated,heng2017controlled,lawson2018twisted,lindsten2018}.
The iterated auxiliary \gls{SMC} \citep{guarniero2017iterated} makes 
strong assumptions on the model for easier optimization. 
End-to-end optimization of a priori independent twisting potentials 
$\tpot_\tm$ and proposals $\prop_\tm$ might lead to even more accurate 
algorithms.



\newpage
\begin{subappendices}
\section{Taylor and Unscented Transforms}\label{sec:proptwisting:tut}
In this appendix we detail the \gls{EKF}- and \gls{UKF}-based 
approximations to the distribution $p(\xp_\tm, \y_\tm \given \xp_{1:\tm-1})$. We refer 
to the two approaches as Taylor and Unscented transforms, respectively.

\paragraph{Taylor Transform}
If we assume that $\vn_\tm$ and $\en_\tm$ are zero-mean and linearize $\amod$ around $\vn_\tm = 0$ and 
$\cmod$ around $\xp_\tm = \amod(\xp_{1:\tm-1},0),  \en_\tm = 0$ we get
\begin{subequations}
    \begin{align}
    \xp_\tm &\approx  \bar\amod + 
    \jacob_{\vn}^{\amod}(\bar\amod)\vn_\tm,
    \label{eq:linearized_a}\\
    \y_\tm &\approx \bar\cmod + 
    \jacob_{\xp_\tm}^{\cmod}(\bar\cmod)(\xp_\tm-\bar\amod)+
    \jacob_{\en_\tm}^{\cmod}(\bar\cmod)\en_\tm,
    \label{eq:linearized_c}
    \end{align}\label{eq:linearized_ac}%
\end{subequations}
where $\bar \amod = \amod(\xp_{1:t-1}, 0)$, $\bar \cmod =\cmod\left((\xp_{1:t-1},\bar\amod),0\right)$, and 
$\jacob_{\vn_\tm}^{\amod}(\bar\amod)$ denotes the Jacobian of the 
function $\amod$ with respect to $\vn_\tm$ evaluated at the point 
$\bar\amod$. We rewrite \cref{eq:linearized_ac}:
\begin{subequations}
    \begin{align}
    \xp_\tm &\approx  \bar\amod + 
    \jacob_{\vn}^{\amod}(\bar\amod)\vn_\tm,
    \label{eq:linearized_rw_a}\\
    \y_\tm &\approx \bar\cmod +
    \jacob_{\xp_\tm}^{\cmod}(\bar\cmod)\jacob_{\vn}^{\amod}(\bar\amod)\vn_\tm+
    \jacob_{\en_\tm}^{\cmod}(\bar\cmod)\en_\tm.
    \label{eq:linearized_rw_c}
    \end{align}\label{eq:linearized_rw_ac}%
\end{subequations}
For the approximation in \cref{eq:linearized_rw_ac} $\xp_\tm,\y_\tm \given \xp_{1:\tm-1}$ 
is indeed Gaussian, and with $\vn_\tm \sim \Norm(0,Q), \en_\tm ~\sim 
\Norm(0,R)$, we can identify the blocks of $\widehat\mu$ and 
$\widehat\Sigma$:
\begin{subequations}
    \begin{align}
    \widehat\mu_{\xp} &= \bar\amod, \\
    \widehat\mu_\y &= \bar\cmod, \\
    \widehat\Sigma_{\xp\xp} &= \jacob_{\vn}^{\amod}(\bar\amod)\, Q\, 
    \jacob_{\vn}^{\amod}(\bar\amod)^\top, \\
    \widehat\Sigma_{\xp\y} &= \widehat\Sigma_{\y\xp}^\top = \jacob_{\vn}^{\amod}(\bar\amod)\, Q\, 
    \jacob_{\vn}^{\amod}(\bar\amod)^\top 
    \jacob_{\xp_\tm}^{\cmod}(\bar\cmod)^\top,\\
    \widehat\Sigma_{\y\y} &= 
    \jacob_{\xp_\tm}^{\cmod}(\bar\cmod)\jacob_{\vn}^{\amod}(\bar\amod) \, Q\, 
    \jacob_{\vn}^{\amod}(\bar\amod)^\top 
    \jacob_{\xp_\tm}^{\cmod}(\bar\cmod)^\top + 
    \jacob_{\en_\tm}^{\cmod}(\bar\cmod)\, R\, 
    \jacob_{\en_\tm}^{\cmod}(\bar\cmod)^\top.
    \end{align}\label{eq:ekf_proposal}%
\end{subequations}
Using \eg automatic differentiation we can easily compute the Jacobians 
needed for the \gls{EKF}-based proposal. However, we still have to be 
able to evaluate the densities defined by 
$\amod(\xp_{1:\tm-1},\vn_\tm)$ and $\cmod(\xp_{1:\tm},\en_\tm)$ for 
the weight updates $\uwp_\tm$. For notational convenience 
we have omitted the dependence on previous latent states in the 
expressions above. However, in general because of this dependence
we will have to compute the proposals separately for each particle, which 
includes inverse matrix computations that can be costly if 
$\xp_\tm$ is high-dimensional.

\paragraph{Unscented Transform}
We rewrite \cref{eq:structural:ac} as a single joint random variable, 
and such that the functions $\amod,\cmod$ only depend on $\xp_{1:\tm-1}, \vn_\tm$, and $\en_\tm$,
\begin{align}
    \left(\begin{array}{c}
    \xp_\tm \\
    \y_\tm
    \end{array}\right)
    &=
    \left(\begin{array}{c}
    \amod(\xp_{1:\tm-1}, \vn_\tm) \\
    \cmod\left(\left(\xp_{1:\tm-1},\amod(\xp_{1:\tm-1}, \vn_\tm)\right), \en_\tm\right)
    \end{array}\right)
    \defeq \zmod(\xp_{1:\tm-1},\zn_\tm),
    \label{eq:ukf_joint_ac}
\end{align}
where $\zn_\tm = (\vn_\tm, \en_\tm)^\top$. By approximating the 
conditional distribution of $(\xp_\tm, 
\y_\tm)^\top$ given $\xp_{1:\tm-1}$ as Gaussian, we can again compute 
the distribution of $\xp_\tm \given \xp_{1:\tm-1}, \y_{\tm}$ for 
the approximation. 
We choose sigma points based on the two first moments of $\zn_\tm$. 
The mean and variance of $\zn_\tm$ is
\begin{align*}
    \mu_\zn &= \Exp[\zn_\tm], \quad
    \Sigma_\zn =  \Var(\zn_\tm) = \sum_{\lm=1}^{\dimX+\dimY} 
    \sigma_\lm^2 
    \uv_\lm \uv_\lm^\top,
\end{align*}
where the scalars $\sigma_\lm$ and vectors $\uv_\lm$ correspond to the 
singular value decomposition of $\Sigma_\zn$. We then choose the 
$2\dimX+2\dimY + 1$ sigma points for $\zn$ as follows,
\begin{align*}
    \sigp^0 &= \mu_\zn,  \quad
    \sigp^{\pm\lm} = \mu_\zn \pm \sigma_\lm \sqrt{\dimX+\dimY+\lambda} 
    \cdot \uv_\lm,
\end{align*}
where $\lambda=\alpha^2(\dimX+\dimY+\kappa)-\dimX-\dimY$ \citep{julier1997new}. We will 
discuss the choice of design parameters $\alpha$ and $\kappa$ below. 
The mean and variance of $(\xp_\tm, \y_\tm)^\top$ is estimated by
\begin{subequations}
    \begin{align}
        &\Exp\left[(\xp_\tm, \y_\tm)^\top\right] \approx \left(\begin{array}{c} \widehat \mu_{\xp} \\ \widehat 
        \mu_{\y}\end{array}\right) =
        \sum_{\lm=-(\dimX+\dimY)}^{\dimX+\dimY} \sigw^\lm \zmod(\xp_{1:\tm-1}, 
        \sigp^\lm) , \\
        &\Var\left((\xp_\tm, \y_\tm)^\top\right) \approx 
        \left(\begin{array}{cc}
        \widehat \Sigma^{\xp\xp} & \widehat \Sigma^{\xp\y}  \\
        \widehat \Sigma^{\y\xp} & \widehat \Sigma^{\y\y}  
        \end{array}\right) =  \nonumber\\
        &(1-\alpha^2+\beta)\left(\zmod(\xp_{1:\tm-1},\sigp^0)- 
        \left(\begin{array}{c} \widehat \mu_{\xp} \\ \widehat 
        \mu_{\y}\end{array}\right)\right)\left(\zmod(\xp_{1:\tm-1},\sigp^0)- 
        \left(\begin{array}{c} \widehat \mu_{\xp} \\ \widehat 
        \mu_{\y}\end{array}\right)\right)^\top+
        \nonumber\\
        &\sum_{\lm=-(\dimX+\dimY)}^{\dimX+\dimY} \sigw^\lm \left( \zmod(\xp_{1:\tm-1}, 
        \sigp^\lm) - \left(\begin{array}{c} \widehat \mu_{\xp} \\ \widehat 
        \mu_{\y}\end{array}\right)\right) \left( \zmod(\xp_{1:\tm-1}, 
        \sigp^\lm) - \left(\begin{array}{c} \widehat \mu_{\xp} \\ \widehat 
        \mu_{\y}\end{array}\right)\right)^\top,
    \end{align}
    \label{eq:ukf_based_musig}%
\end{subequations}
where the coefficients $\sigw^\lm$ are
\begin{align*}
    \sigw^0 &= \frac{\lambda}{\dimX+\dimY+\lambda}, \quad \sigw^{\pm\lm} = 
    \frac{1}{2(\dimX+\dimY+\lambda)},
\end{align*}
and $\beta$ is another design parameter.

Like in the \gls{EKF}-based 
approximation we still need access to the densities corresponding to 
\cref{eq:structural:ac}. Unlike in the \gls{EKF} case the, 
\gls{UKF}-based approximation does not need access to derivatives. 
However, we need to choose the three design parameters: $\alpha, \beta, \kappa$. 
Both $\alpha$ and $\kappa$ 
control the spread of the sigma points, and $\beta$ is related to the 
distribution of $\zn_\tm$ \citep{julier2002scaled}. Typical values of 
these parameters are $(\alpha, \kappa, \beta) = (10^{-3}, 0, 2)$.
\end{subappendices}

\chapter{Nested Monte Carlo: Algorithms and Applications}\label{sec:pm}
In this chapter we will discuss nesting \acrfull{MC} algorithms, using several layers of \gls{MC} to construct composite algorithms. Usually, this is done to approximate some intractable idealized algorithm. 
The intractability can stem from \eg a likelihood 
or weight that we cannot evaluate in closed form, or a proposal distribution that 
we cannot sample exactly from. Either way, the key idea of these
methods is to introduce additional auxiliary 
variables to tackle the intractability in the idealized method. 
The approximate method is constructed in such a way that it retains 
some of the favorable properties of the idealized \gls{MC} method that 
it tries to emulate. We will in this section study particle-based \acrfull{PM} methods for 
several idealized \gls{MC} methods, ranging from marginal \gls{MH} 
\citep{andrieu2009pseudo,andrieu2010particle} to distributed \gls{SMC} 
\citep{verge2015parallel} and various nested \gls{SMC} methods 
\citep{fearnhead2010random,chopin2013smc2,naesseth2015nested}. 

In the first section below we discuss the unbiasedness property of the 
\gls{SMC} normalization constant estimate, a key property underpinning 
several nested \gls{MC} methods. Then, in the next section we discuss 
\gls{PMH}. Next, we introduce the concept of 
\emph{proper weights} and discuss various \gls{SMC} methods relying on 
random (but proper) weights. 
Finally, we introduce a few perspectives on distributed \gls{SMC}.

\section{The Unbiasedness of the Normalization Constant Estimate}\label{sec:unbiasedZ}
We briefly touched upon one of the key theoretical properties of 
\gls{SMC} in \cref{sec:analysisconvergence}, the unbiasedness of the 
normalization constant estimate $\hatZ_\tm$. This will be important 
for several of the \gls{PM} methods that we discuss in this chapter. 
We repeat the definition of $\hatZ_\tm$ again for clarity:
\begin{align}
    \hatZ_\tm &= \prod_{\km = 1}^\tm \frac{1}{\Np} \sum_{\ip=1}^\Np 
    \uwp_\km^\ip.
    \label{eq:pm:zhat}
\end{align}
In \cref{prop:pm:unbiased} we provide a result that formalizes 
the unbiasedness property. Furthermore, we provide a straightforward self-contained 
proof of this important property of the \gls{SMC} algorithm.
\begin{proposition}
The \gls{SMC} normalization constant estimate $\hatZ_\tm$ is non-negative 
and unbiased
    \begin{align}
        \Exp \left[\hatZ_\tm \right] = \Z_\tm, \quad
        \hatZ_\tm &\geq 0.
    \end{align}
    \label{prop:pm:unbiased}
\end{proposition}
\begin{proof}
See Appendix~\ref{sec:smc:unbiasedZ}.
\end{proof}
Often it will prove useful to think of the $\hatZ_\tm$ simply as a 
function of the random seed used to generate all the random variables in 
the \gls{SMC} algorithm. Below we formalize this point-of-view, which 
will be useful for \eg \gls{PMH} and \gls{NSMC} in \cref{sec:pmh} 
and \cref{sec:pwsmc}, respectively.

\paragraph{Auxiliary Variables and Random Seeds}
The complete set of random variables generated in the \gls{SMC} 
algorithm consists of each of the particles $\xp_\tm^\ip$, for 
$\ip \in \{1,\ldots,\Np\}$ and $\tm\in\{1,\ldots,\Tm\}$, and ancestor 
variables $\ap_\tm^\ip$ for $\ip \in \{1,\ldots,\Np\}$ and 
$\tm\in\{1,\ldots,\Tm-1\}$. Both $\xp_\tm^\ip$ and $\ap_\tm^\ip$ are 
often generated by sampling base auxiliary random numbers such as Gaussian or 
uniforms, then the particles and ancestor variables are simply a 
function of these auxiliary variables. We will denote the collection 
of all base auxiliary random variables, also referred to as the 
\emph{random seed}, by $\up$. Each individual particle 
$\xp_\tm^\ip$ and ancestor $\ap_\tm^\ip$ are then a transformation from 
the random seed. For example we might have $\up = u_{\xp,1:\Tm}^{1:\Np} 
\cup u_{\ap,1:\Tm-1}^{1:\Np}$, where each $u_{\xp,\tm}^\ip, 
u_{\ap,\tm}^\ip\sim \Uni(0,1)$. The function that connects 
$u_{\xp,\tm}^\ip$ to $\xp_\tm^\ip$ and $u_{\ap,\tm}^\ip$ to 
$\ap_\tm^\ip$ are then the inverse \glspl{CDF} of the 
proposal $\prop_\tm$ and resampling mechanism from \cref{sec:resampling}, 
respectively. 
A common tool in \gls{PM} is to think of the normalization constant 
estimate in \cref{eq:pm:zhat} as a function of all the random 
variables $\up$, effectively the random seed of the \gls{SMC} 
algorithm. Then we can write \cref{eq:pm:zhat} as 
$\hatZ_\tm(\up)$, \ie
\begin{align}
    \hatZ_\tm(\up) \eqdef \prod_{\km = 1}^\tm \frac{1}{\Np} \sum_{\ip=1}^\Np 
    \uwp_\km^\ip,
    \label{eq:pm:zhatseed}
\end{align}
where the weights $\uwp_\km^\ip$ are functions of the random seed 
$\up$. We denote the distribution of the random seed $\up$ by 
$\rsdist(\up)$. The main result in \cref{prop:pm:unbiased} can then be 
written in terms of the random seed $\up$: $\Exp_{\rsdist(\up)}\left[\hatZ_\tm(\up)\right] = 
\Z_\tm$.

\begin{example}[$ \hatZ_\tm(\up)$ for \cref{ex:running}]
We revisit our running non-Markovian Gaussian example and consider a standard \gls{SMC} estimator with the prior proposal and no twisting of the target. This means that for $\up = \varepsilon_{1:\Tm}^{1:\Np} \cup u_{1:\Tm-1}^{1:\Np}$, where $\varepsilon_\tm^{\ip}\sim \Norm(0,1)$ and $u_\tm^{\ip}\sim\Uni(0,1)$,  we have
\begin{align*}
	\hatZ_\tm(\up) =\prod_{\km = 1}^\tm \frac{1}{\Np} \sum_{\ip=1}^\Np 
    \uwp_\km^\ip = \prod_{\km = 1}^\tm \frac{1}{\Np} \sum_{\ip=1}^\Np 
   \gmod_\km\left(\y_\km\given \xp_{1:\km}^\ip\right),
\end{align*}
where $\xp_{1:\km}^\ip = (\xp_{1:\km-1}^{\ap_{\km-1}^\ip}, \xp_\km^\ip)$, the ancestor variable $\ap_{\km-1}^\ip$ is defined by the resampling mechanism \cref{sec:resampling} using $u_{\km-1}^\ip$, and 
\begin{align*}
	\xp_\km^\ip = \phi \xp_{\km-1}^{\ap_{\km-1}^\ip} + \sqrt{q} \varepsilon_\km^\ip.
\end{align*}
In \cref{code:rs:running} we provide a code snippet, definitions left out for clarity, that illustrates how $\hatZ_\tm(\up)$ can be computed. For numerical stability purposes we compute the log-normalization constant estimate, $\log\hatZ_\tm(\up)$, instead.
\begin{lstlisting}[caption={Computing $\log\hatZ_\tm(\up)$ for \cref{ex:running} with $\up=\varepsilon_{1:\Tm}^{1:\Np} \cup u_{1:\Tm-1}^{1:\Np}$.},
label=code:rs:running]
def multinomial_resampling(w, u):
	bins = np.cumsum(w)
    return np.digitize(u,bins).astype(int)
...
for t in range(T):
	if t > 0:
		a=multinomial_resampling(w,u[:,t-1])
		mu = mu[a]
	x[:,t]=phi*x[a,t-1]+np.sqrt(q)*eps[:,t]
	mu = beta*mu + x[:,t]
	logw = norm.logpdf(y[t], mu, np.sqrt(r))
	w = np.exp(logw-logsumexp(logw))
	logZ += logsumexp(logw)-np.log(N)
\end{lstlisting}
\label{ex:rs:hatZ}
\end{example}

\section{Particle Metropolis-Hastings}\label{sec:pmh}
\Gls{PMH} melds particle-based methods, such as \gls{IS} and \gls{SMC}, 
with \acrlong{MH}, providing the practitioners with powerful 
approximate Bayesian inference for a large class of intractable 
likelihood problems. The particle \gls{MCMC} methods discussed in this section were mainly developed in \citet{andrieu2010particle}. 

\Gls{MH} is a \gls{MCMC} method 
\citep{robert2004monte} for approximate simulation from a \gls{PDF}. 
The typical example is the posterior distribution of unknown model 
parameters $\mparams$ given the observed data $\ym$, 
\begin{align}
    p(\mparams \given \ym) &= \frac{p(\mparams)p(\ym\given\mparams)}{p(\ym)} = 
    \frac{p(\mparams) \int p(\xm, \ym\given \mparams) \dif 
    \xm}{p(\ym)},
    \label{eq:pm:posterior}
\end{align}
where here $\xm$ are latent (unknown) variables that we wish to 
integrate out. The likelihood $p(\ym\given\mparams) = \int 
p(\xm,\ym\given\mparams) \dif\xm$ is often intractable. A common 
solution is to use data augmentation, sampling instead (approximately) 
from the posterior distribution of both $\mparams$ and $\xm$ given the 
data $\ym$, \ie $p(\mparams,\xm\given\ym)$. Marginal \gls{MH}, see 
\cref{alg:mh}, on the 
other hand generates approximate samples directly from the marginal 
distribution $p(\mparams \given \ym)$, provided that the likelihood 
can be evaluated exactly. The key idea is to generate a Markov chain $\mparams_1, \mparams_2, \ldots$ with a stationary distribution equal to the target distribution of interest \cref{eq:pm:posterior}. The accept-reject step in \cref{alg:mh}, accepting the proposed value $\mparams'$ as the new iterate with probability $A$, ensures that the stationary distribution is equal to \cref{eq:pm:posterior}. The acceptance probability 
\begin{align*}
A =  \min\left(1,\frac{p(\ym\given\mparams')\,p(\mparams')\, \prop(\mparams^{\jp-1}\given\mparams')}
        {p(\ym\given\mparams^{\jp-1})\,p(\mparams^{\jp-1})\,\prop(\mparams'\given\mparams^{\jp-1})}\right)
\end{align*}
is the minimum of one and 
the ratio of the posterior and the proposal evaluated at the proposed value, divided by the same ratio evaluated at the previous iterate.
For a more thorough treatment of the \gls{MH} algorithm see \eg \citet{robert2004monte}.
\begin{algorithm}[tb]
    \SetKwInOut{Input}{input}\SetKwInOut{Output}{output}
    \Input{\Gls{PDF} $p(\ym,\mparams)$, proposal 
    $\prop(\mparams'\given\mparams)$, initial iterate
    $\mparams^0$, number of iterations $\Jp$.}
    \BlankLine
    \For{$\jp=1$ \KwTo $\Jp$}{
        Sample $\mparams' \sim 
        \prop(\mparams'\given\mparams^{\jp-1})$\\
        Compute $A = \min\left(1,\frac{p(\ym\given\mparams')\,p(\mparams')\, \prop(\mparams^{\jp-1}\given\mparams')}
        {p(\ym\given\mparams^{\jp-1})\,p(\mparams^{\jp-1})\,\prop(\mparams'\given\mparams^{\jp-1})}
        \right)$\\
        Sample $u \sim \Uni(0,1)$\\
        
        \eIf{$u < A$}{Set $\mparams^\jp = \mparams'$}
        {Set $\mparams^\jp = \mparams^{\jp-1}$}
    }
    \caption{\Acrlong{MH}}\label{alg:mh}
\end{algorithm}

\subsection{Particle Marginal Metropolis-Hastings}\label{sec:pm:pmmh}
The \acrlong{PM} method replaces 
the exact likelihood with an estimate of it, usually constructed by 
\gls{IS} or \gls{SMC} like in \cref{eq:pm:zhat}. For a fixed value of 
the parameter $\mparams$ we can use the tools from the previous 
chapters to construct a non-negative and unbiased approximation to 
$p(\ym\given\mparams)$, we will use the notation from 
\eqref{eq:pm:zhatseed} and label this estimate 
$\hatZ(\up,\mparams)$. From \cref{prop:pm:unbiased} we have that
\begin{align}
    \Exp_{\rsdist(\up)}\left[\hatZ(\up,\mparams)\right] &= 
    p(\ym\given\mparams), & \hatZ(\up,\mparams) \geq 0.
\end{align}
\begin{algorithm}[tb]
    \SetKwInOut{Input}{input}\SetKwInOut{Output}{output}
    \Input{Proposals 
    $\prop(\mparams'\given\mparams), \rsdist(\up)$, likelihood approximation $\hatZ(\up,\mparams)$, initial iterates
    $\mparams^0, \up^0$, number of iterations $\Jp$.}
    \BlankLine
    \For{$\jp=1$ \KwTo $\Jp$}{
        Sample $\mparams' \sim 
        \prop(\mparams'\given\mparams^{\jp-1})$, $\up' \sim \rsdist(\up')$ \\
        Compute $\hatZ(\up',\mparams')$ \\
        Compute $A = \min\left(1,\frac{\hatZ(\up',\mparams') \,p(\mparams')\, \prop(\mparams^{\jp-1}\given\mparams')}
        {\hatZ(\up^{\jp-1},\mparams^{\jp-1})\,p(\mparams^{\jp-1})\,\prop(\mparams'\given\mparams^{\jp-1})}
        \right)$\\
        Sample $u \sim \Uni(0,1)$\\
        
        \eIf{$u < A$}{Set $\mparams^\jp = \mparams'$, $\up^\jp = \up'$}
        {Set $\mparams^\jp = \mparams^{\jp-1}$, $\up^\jp = \up^{\jp-1}$}
    }
    \caption{\Acrlong{PMMH}}\label{alg:pmmh}
\end{algorithm}
Simply replacing $p(\ym\given\mparams)$ with this unbiased and 
non-negative approximation, $\hatZ(\up,\mparams)$, results in the 
\gls{PMMH} algorithm in \cref{alg:pmmh} \citep{andrieu2010particle}. In a practical implementation we do not need to save the entire seed  $\up$ for each iteration, it suffices to save the scalar value for our normalization constant estimate $\hatZ(\mparams,\up)$. We illustrate the \gls{MH} and \gls{PMMH} algorithms applied to our running example in \cref{ex:pmmh}.
    
\begin{example}[\Acrlong{PMMH} for \cref{ex:running}]
We revisit our running non-Markovian Gaussian example. To keep it simple we focus on the \gls{IS} special case, \ie with $T=1$, which means we are interested in inference for $\mparams = (q, r)$. We assign a log-normal prior to the variance parameters $q$ and $r$, \ie $\log q \sim \Norm(0,1)$ and $\log r \sim \Norm(0,1)$. The corresponding marginal likelihood is then given by $p(\ym\given\mparams) = \Norm(\ym\given 0, q+r)$. We use a normal distribution random walk proposal with variance $0.5$ for the logarithm of the parameters. In \cref{fig:pm:mhvspmmh} we can see the kernel density estimate based on the samples of a $20\, 000$ iterations long run of the two respective algorithms. The two algorithms obtain the same marginal distribution as expected.
\begin{figure}[ht]
    \centering
    \begin{subfigure}[b]{0.45\textwidth}
        \includegraphics[width=\textwidth]{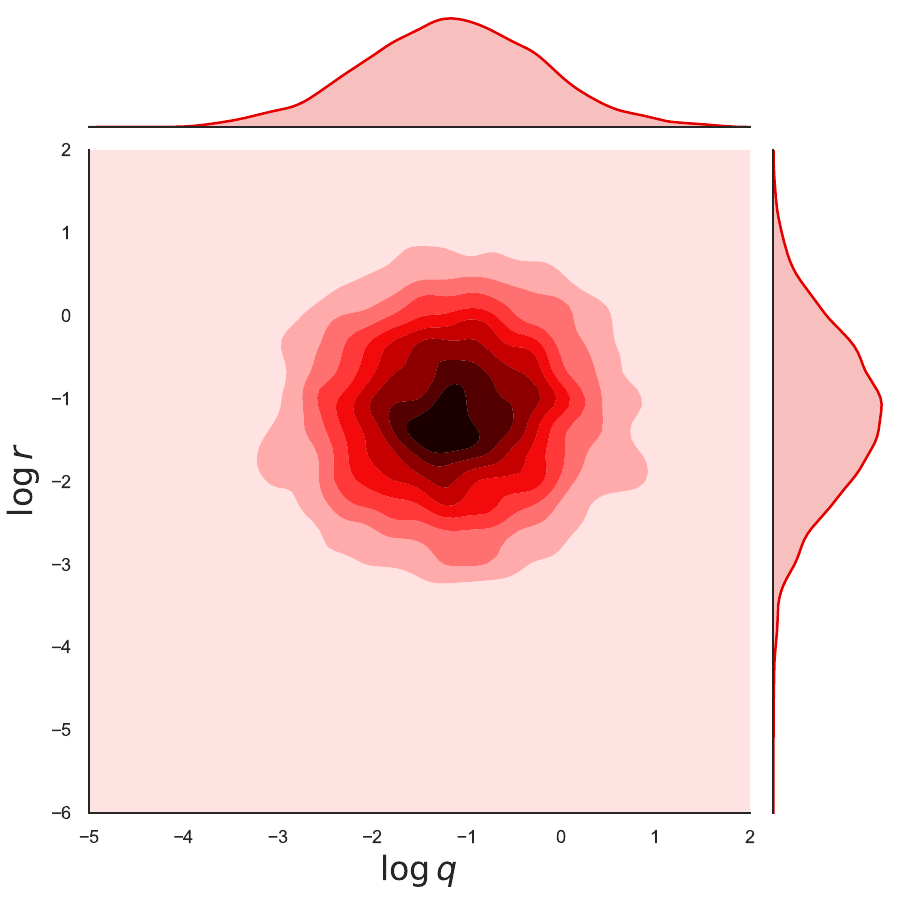}
        \caption{\gls{MH}}
        \label{fig:pm:mh}
    \end{subfigure}
    ~ 
    \begin{subfigure}[b]{0.45\textwidth}
        \includegraphics[width=\textwidth]{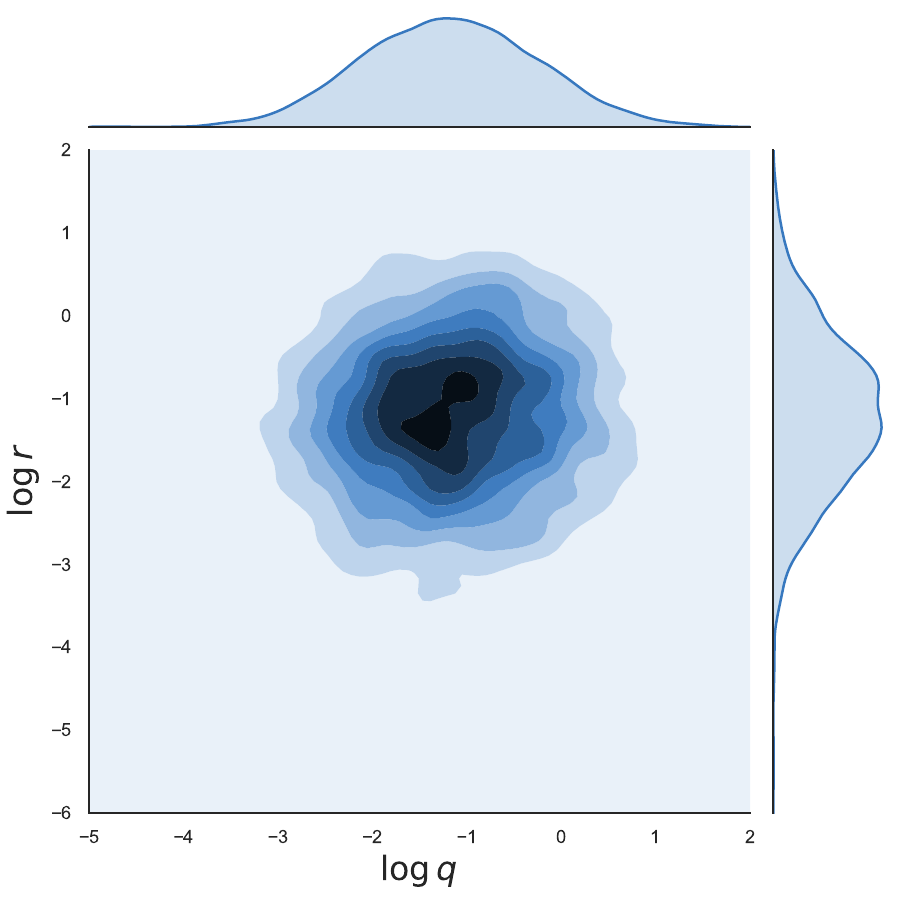}
        \caption{\gls{PMMH} with $\Np=10$}
        \label{fig:pm:pmmh}
    \end{subfigure}
    \caption{Kernel density estimates of the samples for the parameters $\log q$ and $\log r$ using $20\, 000$ \gls{MCMC} iterations.}
    \label{fig:pm:mhvspmmh}
\end{figure}
\label{ex:pmmh}
\end{example}

The motivation for why this works is fairly straightforward. We can think of the \gls{PMMH} algorithm as a standard \gls{MH} on the extended space of $(\mparams,\up)$, \ie adding the random seed as an auxiliary variable. The proposal on this extended space is
\begin{align}
	\prop(\mparams'\given\mparams) \rsdist(\up'),
\end{align}
and the distribution we are trying to draw samples from, \ie the target distribution, is 
\begin{align}
	\frac{\hatZ(\up,\mparams) p(\mparams) \rsdist(\up)}{p(\ym)},
	\label{eq:pm:pmmh:extended}
\end{align}
which by the unbiasedness of $\hatZ(\up,\mparams)$ has the posterior distribution as a marginal 
\begin{align*}
	p(\mparams\given\ym) = \int\frac{\hatZ(\up,\mparams) p(\mparams) \rsdist(\up)}{p(\ym)}\dif \up.
\end{align*}
This means that if we can generate samples approximately distributed as the distribution in \cref{eq:pm:pmmh:extended}, we get samples approximately distributed as $p(\mparams\given\ym)$ as a by-product by considering only the samples for $\mparams$.

To show that \gls{PMMH} is a standard \gls{MH} on an extended space we first note that the proposal $\prop(\mparams'\given\mparams) \rsdist(\up')$ is the distribution of the proposed value $(\mparams',\up')$. Then studying the corresponding acceptance probability $A$ for the above choice of proposal and target distributions,
\begin{align*}
	A &= \min\left( 1 , 
	\frac{\hatZ(\up',\mparams') p(\mparams') \rsdist(\up')}{\hatZ(\up^\jp,\mparams^\jp) p(\mparams^\jp) \rsdist(\up^\jp)}
	\frac{\prop(\mparams^{\jp-1}\given\mparams') \rsdist(\up^\jp)}{ \prop(\mparams'\given\mparams^{\jp-1})\rsdist(\up')}
	\right) \\
	&= \min\left( 1 , 
	\frac{\hatZ(\up',\mparams') p(\mparams')}{\hatZ(\up^\jp,\mparams^\jp) p(\mparams^\jp)}
	\frac{\prop(\mparams^{\jp-1}\given\mparams')}{ \prop(\mparams'\given\mparams^{\jp-1})}
	\right),
\end{align*}
we get exactly the expression from \cref{alg:pmmh}. These two properties ensure that \gls{PMMH} is an \gls{MH} algorithm on the extended space with \cref{eq:pm:pmmh:extended} as its stationary distribution.


\subsection{Sampling the Latent Variables}
\label{sec:pm:pimh}
If we are interested not just in the parameters but in the latent variables also, we can adapt the particle \gls{MH} framework with a simple extension of the \gls{PMMH} algorithm \citep{andrieu2010particle}. We can use particle methods to generate approximate samples from the complete target with both parameters and latent variables, $p(\mparams,\xm\given \ym)$. We simply use the final \gls{SMC} approximation for $p(\xm\given\mparams,\ym)$ that we obtain as a by-product of computing $\hatZ(\up,\mparams)$ as a proposal in the \gls{MH} algorithm. The \gls{SMC} approximation to the conditional distribution $\xm\given\mparams,\ym$ is
\begin{align}
	p(\xm\given\mparams,\ym) &\approx \widehat p(\xm\given\up) \eqdef \sum_{\ip=1}^\Np \nwp_\Tm^\ip \dirac_{\xp_{1:\Tm}^\ip}(\dif \xp_{1:\Tm}),
\end{align}
where we suppress the dependence on the parameters $\mparams$ and data $\ym$ for clarity. Just like $\hatZ(\up,\mparams)$ the conditional distribution estimator is also completely defined by the random seed $\up$.
Using the above distribution as our proposal for $\xm$ lets us derive the \acrlong{PMH} algorithm, see \cref{alg:pimh}. 

Clearly, this algorithm can be applied for sampling the latent variables also when the model
parameters $\mparams$ are all known. In this case we simply drop all $\mparams$-dependent terms and variables in \cref{alg:pimh}. 
The proposal on the extended space is then
\begin{align}
	\rsdist(\up) \, \widehat{p}(\xm\given\up),
\end{align}
which corresponds to first running the \gls{SMC} algorithm (sample $\up'\sim \rsdist(\up)$) and then sample one particle trajectory from the resulting \gls{SMC} approximation (sample $\xm' \sim \widehat{p}(\xm\given\up).$ Since this is done independently from the current state of the Markov chain, it is referred to as an \emph{independent proposal} and the resulting algorithm is know as \gls{PIMH} \citep{andrieu2010particle}.
%
The acceptance probability for \gls{PIMH} is simply
the minimum of one and the ratio of normalization constant estimators, \ie
\begin{align*}
	A &= \min\left(1,\frac{\hatZ(\up')}
        {\hatZ(\up^{\jp-1})}
        \right).
\end{align*}
\begin{algorithm}[tb]
    \SetKwInOut{Input}{input}\SetKwInOut{Output}{output}
    \Input{Proposals 
    $\prop(\mparams'\given\mparams), \rsdist(\up)$, $\widehat{p}(\xm\given\up)$, likelihood approximation $\hatZ(\up,\mparams)$, initial iterates
    $\mparams^0, \up^0, \xm^0$, number of iterations $\Jp$.}
    \BlankLine
    \For{$\jp=1$ \KwTo $\Jp$}{
        Sample $\mparams' \sim 
        \prop(\mparams'\given\mparams^{\jp-1})$, $\up' \sim \rsdist(\up')$, $\xm' \sim \widehat{p}(\xm'\given\up')$ \\
        Compute $\hatZ(\up',\mparams')$ \\
        Compute $A = \min\left(1,\frac{\hatZ(\up',\mparams') \,p(\mparams')\, \prop(\mparams^{\jp-1}\given\mparams')}
        {\hatZ(\up^{\jp-1},\mparams^{\jp-1})\,p(\mparams^{\jp-1})\,\prop(\mparams'\given\mparams^{\jp-1})}
        \right)$\\
        Sample $u \sim \Uni(0,1)$\\
        
        \eIf{$u < A$}{Set $\mparams^\jp = \mparams'$, $\up^\jp = \up'$, $\xm^\jp = \xm'$}
        {Set $\mparams^\jp = \mparams^{\jp-1}$, $\up^\jp = \up^{\jp-1}$, $\xm^\jp = \xm^{\jp-1}$}
    }
    \caption{\Acrlong{PMH}}\label{alg:pimh}
\end{algorithm}

\begin{example}[\Acrlong{PIMH} for \cref{ex:running}]
We revisit our running non-Markovian Gaussian example. We set $T=20$ and run \gls{PIMH} for $\Np=5$, $10$ and $20$. The kernel density estimates of the marginal posterior distributions of $\xp_1$ and $\xp_\Tm$ based on the $10,000$ samples generated can be seen in \cref{fig:illustration:pimh}. We can see that, while in theory exact in the limit of infinite number of samples, the approximations for the lower number of particles $\Np=5$, $10$, are distinctly non-Gaussian. It is only when we use a large enough number of particles, \eg $\Np=20$, that we obtain more satisfactory results for a finite number of \gls{MCMC} iterations.
\begin{figure}[tbp]
\centering
\begin{subfigure}[b]{0.48\textwidth}
\resizebox{1.0\textwidth}{!}{
\includegraphics{{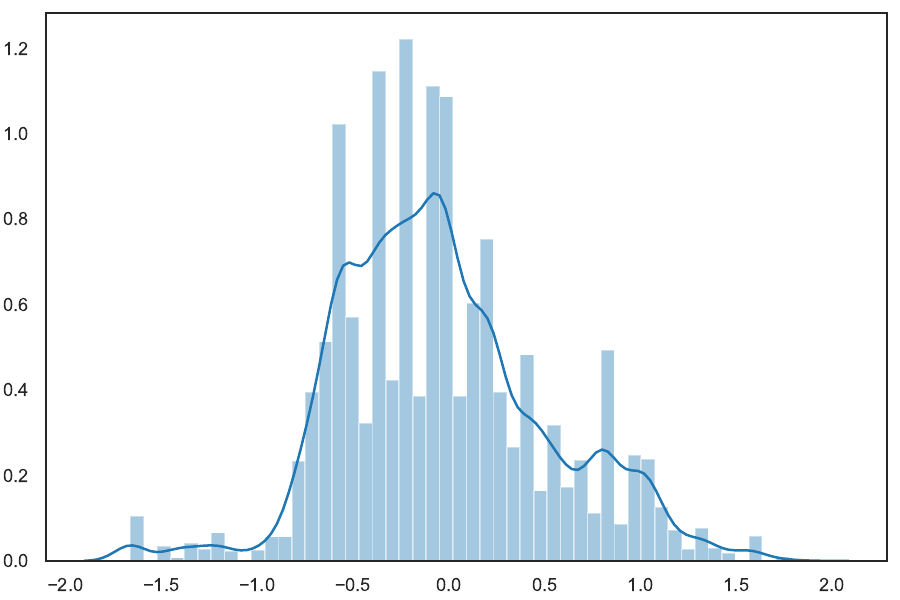}}
}
\caption{$\tm=1$,  $\Np=5$}
\end{subfigure}
\begin{subfigure}[b]{0.48\textwidth}
\resizebox{1.0\textwidth}{!}{
\includegraphics{{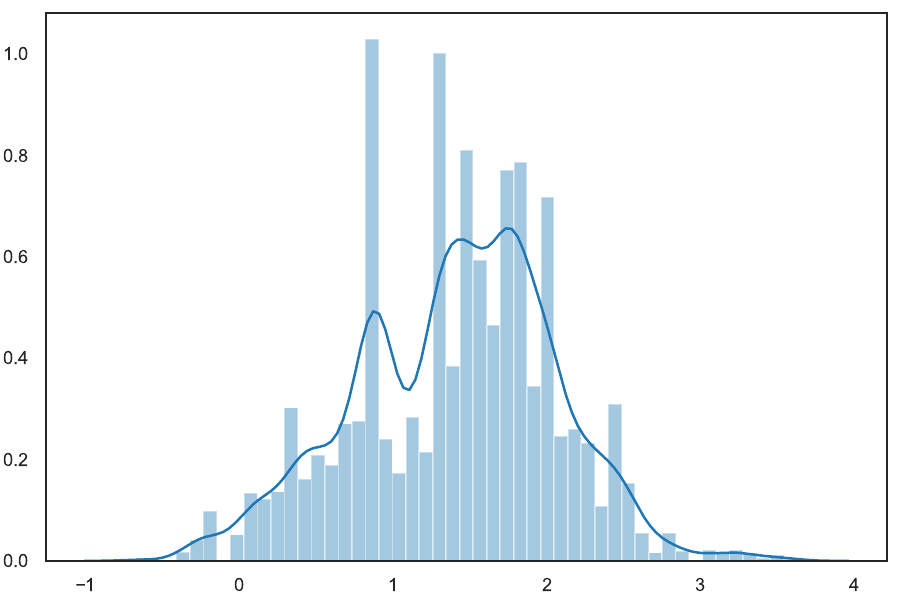}}
}
\caption{$\tm=20$,  $\Np=5$}
\end{subfigure}

\begin{subfigure}[b]{0.48\textwidth}
\resizebox{1.0\textwidth}{!}{
\includegraphics{{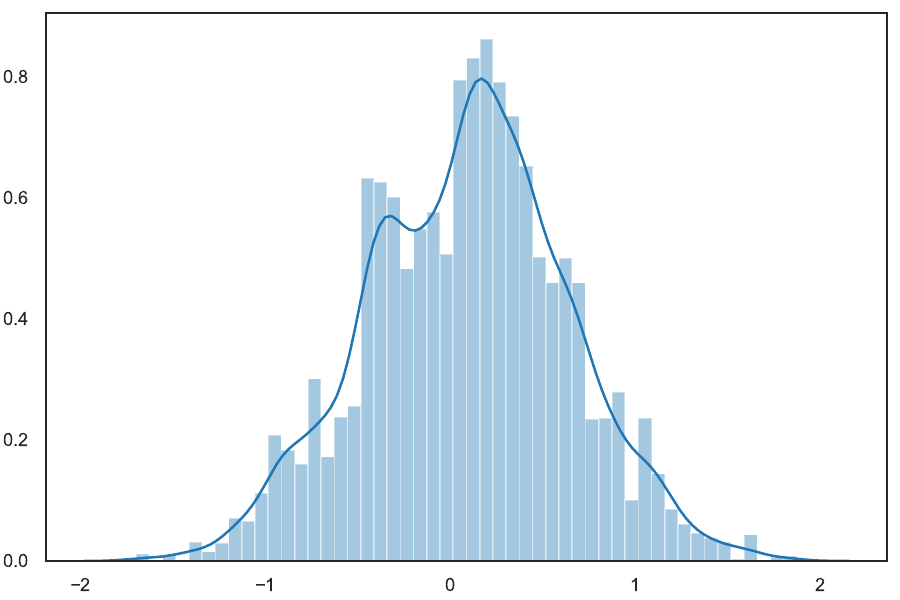}}
}
\caption{$\tm=1$,  $\Np=10$}
\end{subfigure}
\begin{subfigure}[b]{0.48\textwidth}
\resizebox{1.0\textwidth}{!}{
\includegraphics{{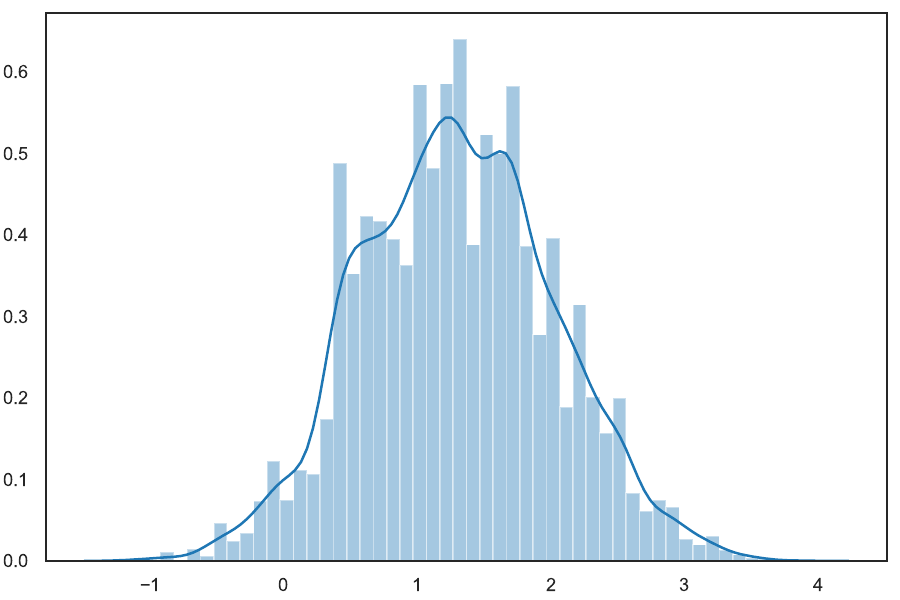}}
}
\caption{$\tm=20$,  $\Np=20$}
\end{subfigure}

\begin{subfigure}[b]{0.48\textwidth}
\resizebox{1.0\textwidth}{!}{
\includegraphics{{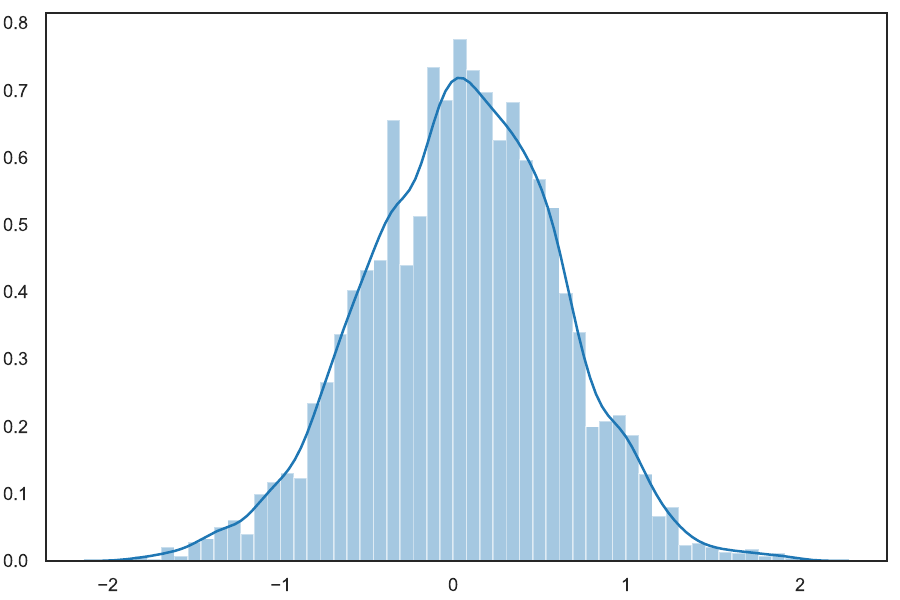}}
}
\caption{$\tm=1$,  $\Np=20$}
\end{subfigure}
\begin{subfigure}[b]{0.48\textwidth}
\resizebox{1.0\textwidth}{!}{
\includegraphics{{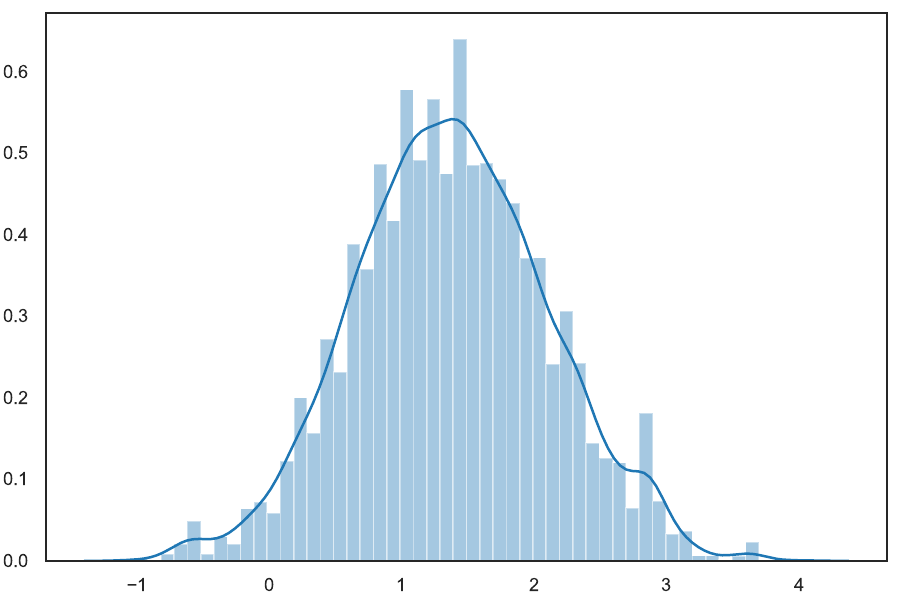}}
}
\caption{$\tm=20$,  $\Np=20$}
\end{subfigure}
\caption{Illustration of the kernel density estimates of the \gls{PIMH} marginal posterior approximation of $\xp_1$ and $\xp_\Tm$ for$T=20$ for $\Np=5$, $10$ and $20$.} \label{fig:illustration:pimh}
\end{figure}
\label{ex:pimh}
\end{example}

Analogously to \gls{PMMH}, we can motivate \gls{PMH} as a standard \gls{MH} algorithm on an extended space. We focus on the the setting where the model parameters are known and the extended space is given by $(\up,\xm)$. The proposal on this extended space is
\begin{align}
	\rsdist(\up) \, \widehat{p}(\xm\given\up),
\end{align}
and the target distribution is
\begin{align}
	\frac{\hatZ(\up) \, \rsdist(\up) \, \widehat{p}(\xm\given\up)}{p(\ym)}.
	\label{eq:pm:pimhtarget}
\end{align}
By the unbiasedness property of \gls{SMC}, \ie \cref{thm:unbiasedness}, we have that the posterior distribution $p(\xm\given\ym)$ is a marginal of \cref{eq:pm:pimhtarget}
\begin{align*}
	\frac{\int \hatZ(\up) \, \rsdist(\up) \, \widehat{p}(\dif \xm\given\up) \dif \up}{p(\ym)} \eqD \frac{p(\xm,\ym)}{p(\ym)} \dif \xm = p(\xm\given\ym) \dif \xm,
\end{align*}
where $\eqD$ denotes equal to in distribution. Studying the corresponding acceptance probability $A$ for the above proposal and target distribution we obtain
\begin{align*}
	A &= \min\left(1,
	\frac{\hatZ(\up') \, \rsdist(\up') \, \widehat{p}(\xm'\given\up') }
	{\hatZ(\up^{\jp-1}) \, \rsdist(\up^{\jp-1}) \, \widehat{p}(\xm^{\jp-1}\given\up^{\jp-1})}
	\frac{\rsdist(\up^{\jp-1}) \, \widehat{p}(\xm^{\jp-1}\given\up^{\jp-1})}
	{\rsdist(\up') \, \widehat{p}(\xm'\given\up')}
	\right) \\
	&= \min\left(1,
	\frac{\hatZ(\up')}{\hatZ(\up^{\jp-1})}
	\right),
\end{align*}
which is identical to the expression from \cref{alg:pimh} (when the model parameters are known).

\subsection{Correlated Pseudo-Marginal Methods}
One of the main limitations of the \gls{PM} methods that we have covered so far is that they require that the normalization constant estimate $\hatZ$ is sufficiently accurate to achieve good performance. This requirement means we have to focus much of the computational effort on the \gls{SMC} algorithm by using a high number of particles $\Np$, rather than generating more samples from the posterior $p(\mparams\given\ym)$. By studying the \gls{PMMH} algorithm we can see that the proposed parameter $\mparams'$ at each iteration is \emph{correlated} with the previous parameter $\mparams^{\jp-1}$, encouraging local exploration and higher acceptance rates. However, this is not true for the random seed $\up$ that is generated independently at each iteration. Introducing correlation also for this random seed, \ie replacing the independent proposal $\rsdist(\up)$ with a correlated proposal $\rsdist(\up' \given\up)$, can allow us to effectively lower the computational cost of the \gls{SMC} estimate through choosing a lower number of particles $\Np$. 
Since the dimension of $\up$ typically increases linearly with $\Np$, we need a proposal distribution which scales favorably to high dimensions. Examples include preconditioned Crank-Nicolson \citep{deligiannidis2018correlated,dahlin2015accelerating}, elliptical slice sampling \citep{MurrayG:2016} and Hamiltonian Monte Carlo \citep{LindstenD:2016}.


\section{Proper Weights and Sequential Monte Carlo}\label{sec:pwsmc}
So far in this manuscript we have focused on \gls{SMC} algorithms where 
the particles are sampled exactly from the proposal $\prop_\tm$ and 
where the weights $\uwp_\tm$ are deterministic functions of the 
particles. We can in fact relax both of these restrictions 
significantly by considering the weights and particles as joint random 
variables that satisfy a \emph{proper weighting} condition. This will 
lead us to powerful algorithms such as the random-weights particle 
filter \citep{fearnhead2010random}, \gls{SMC}$^2$ \citep{chopin2013smc2}, 
and \gls{NSMC} methods \citep{naesseth2015nested,naesseth2016high}.

To formally justify these algorithms we can make use of the concept of proper weighting. This is a property on the expected value of the weights and samples when applied to functions.
\begin{definition}[Proper Weighting]
We say the (random) pair $(\xp, \uwp)$ are \emph{properly weighted} for 
the (unnormalized) distribution $\umod$ if $\uwp \geq 0$ 
$\textrm{a.s.}$ and for all measurable functions $\fun$
\begin{align}
    \Exp\left[\fun(\xp) \uwp(\xp)\right] &= \mathcal{C}\int \fun(\xp) \umod(\xp) 
    \dif \xp,
    \label{eq:properweights}
\end{align}
for some positive constant $\mathcal{C}>0$ that is independent of 
$\xp$ and $\uwp$.
\label{def:properweights}
\end{definition}
This is highly related to the unbiasedness property of the \gls{SMC} method that we saw already in \cref{thm:unbiasedness}. However, while the basic \gls{SMC} has deterministic weights when conditioned on the samples, here we allow also for random weights within the confines of \cref{def:properweights}. See also \citet{liu2004monte} for an early discussion on this topic in the context of random weights and \gls{IS}.

We will denote the distribution of the random pair $(\uwp,\xp)$ by $\propwx(\dif \uwp, \dif \xp)$. By simulating from $\propwx$ independently
\begin{align*}
	(\uwp^\ip, \xp^\ip) &\iidsim \propwx(\dif \uwp, \dif \xp), ~\ip=1,\ldots,\Np,
\end{align*} 
we can construct an estimate of $\Exp_{\nmod}\left[\fun(\xp)\right]$ that satisfies the law of large numbers,
\begin{align}
	\Exp_{\nmod}\left[\fun(\xp)\right] &\approx \frac{ \sum_{\ip=1}^\Np \uwp^\ip \fun(\xp^\ip)}{\sum_{\jp=1}^\Np \uwp^\jp} = \frac{\frac{1}{\Np} \sum_{\ip=1}^\Np \uwp^\ip \fun(\xp^\ip)}{\frac{1}{\Np}\sum_{\jp=1}^\Np \uwp^\jp}.
\end{align}
Proper weighting ensures that the numerator converges to $\Z \mathcal{C} \Exp_{\nmod}[\fun(\xp)]$ and the denominator to $\Z\mathcal{C}$, and thus the right hand side converges to the true expectation $\Exp_{\nmod}[\fun(\xp)]$. Recall that the constant $\mathcal{C}$ is as defined in \cref{def:properweights}.

We have already seen one example of a properly weighted \gls{IS} method in \cref{sec:proposal}, 
namely nested \gls{IS} (\cref{alg:nis}).
Because standard \gls{IS} is a special case of nested \gls{IS} it follows that also standard \gls{IS} leads to weight-sample pairs that are properly weighted. Below we will see how this can be used to motivate more powerful algorithms for approximate inference. 

\subsection{Random Weights Sequential Monte Carlo}
Random weights \gls{SMC} and \gls{IS} focuses on solving the problem with intractable weights. Because we pick the proposal ourselves it is often easy to sample from, whereas the target distribution depends largely on the model we want to study and can lead to intractable importance weights. A concrete example of this is partially observed \glspl{SDE} as studied by \eg \citet{beskos2006exact,fearnhead2010random}. In this section we focus exclusively on random weights \gls{SMC} algorithms for \glspl{SDE}. The corresponding target distribution, and thus the weights, contains an exponentiated integral that is intractable. If we replace the intractable weights in the standard \gls{SMC} framework, $\uwp(\xp)$, with (non-negative) unbiased approximations of them, $\widehat{\uwp}(\xp)$, we obtain random weights \gls{SMC}. While we focus the exposition on the basic \gls{IS} case, the result and method follows for \gls{SMC} when we have unbiased estimates of the weights conditional on previous samples. It is straightforward to prove that the random pair $(\widehat{\uwp}(\xp), \xp)$ are properly weighted for $\umod$,
\begin{align}
	\Exp\left[ \widehat{\uwp}(\xp)\fun(\xp) \right] &= \Exp\left[ \Exp\left[ \widehat{\uwp}(\xp)\given\xp\right]\cdot\fun(\xp) \right] =
	\Exp\left[ \uwp(\xp)\fun(\xp) \right]  = \int\fun(\xp)\umod(\xp)\dif\xp.
\end{align}

A simplified version of the target distribution in an \gls{SDE} has the following form of its intractable part
\begin{align}
	\exp\left( \Phi(\xp) \right),
	\label{eq:pm:sdetarget}
\end{align}
where $\Phi(\xp)$ is an intractable functional, often an integral. We assume that we can generate unbiased approximations $\widehat\Phi(\xp)$ of the functional for a given value of $\xp$. Using the \emph{Poisson estimator} \citep{wagner1987unbiased} we can obtain an unbiased estimate of the exponentiated functional in \cref{eq:pm:sdetarget}. 
\begin{align}
	\exp\left( \Phi(\xp) \right) &= \sum_{\km=0}^\infty \frac{1}{k!} \left(\Phi(\xp)\right)^\km = 
	e^\lambda \sum_{\km=0}^\infty \frac{1}{k!} \left(\frac{\Phi(\xp)}{\lambda}\right)^\km \lambda^\km e^{-\lambda} \nonumber \\
	&= e^\lambda \Exp\left[\prod_{\jp=1}^\kappa  \frac{\Phi(\xp)}{\lambda}	\right],
	\label{eq:pm:poissonexp}
\end{align}
where the expectation is with respect to $\kappa \sim \Poisson(\lambda)$. A \gls{MC} estimate of \cref{eq:pm:poissonexp} leads to the Poisson estimator
\begin{align}
	e^\lambda  \Exp\left[\prod_{\jp=1}^\kappa  \frac{\Phi(\xp)}{\lambda}	\right] \approx e^\lambda \prod_{\jp=1}^\kappa  \frac{\widehat{\Phi}^\jp(\xp)}{\lambda},
	\label{eq:pm:poissonest}
\end{align}
where $\widehat{\Phi}^\jp(\xp)$ are \iid unbiased estimators of $\Phi(\xp)$. While \cref{eq:pm:poissonest} is unbiased, it is not necessarily non-negative and might be unsuitable for weights estimation in an \gls{SMC} algorithm. With proper modification to the estimate we can sometimes ensure that the estimate is non-negative \citep{fearnhead2010random,jacob2015}. Replacing the intractable part in the weights with its unbiased estimate, like \cref{eq:pm:poissonest} in the \gls{SDE} case, results in the random weights \gls{SMC} algorithm \cref{alg:rwsmc}. The samples $\xp_\tm^\ip$ are generated by a standard user chosen proposal $\prop_\tm(\xp_\tm\given\xp_{1:\tm-1})$ just like in standard \gls{SMC}.
\begin{algorithm}[tb]
    \SetKwInOut{Input}{input}\SetKwInOut{Output}{output}
    \Input{Proposals 
    $\propwx_\tm$ that generate sample pairs properly weighted for $\umod_\tm$, number of samples $\Np$.}
    \BlankLine
    \For{$\tm=1$ \KwTo $\Tm$}{
        \For{$\ip = 1$ \KwTo $\Np$}{
            Sample $(\uwp_\tm^\ip,\xp_{1:\tm}^\ip) \sim 
            \widehat\nmod_{\tm-1}(\xp_{1:\tm-1})
            \propwx_\tm\left(\cdot, \cdot \given \xp_{1:\tm-1}\right)$ 
        }
        Set $\nwp_\tm^\ip = \frac{\uwp_\tm^\ip}{\sum_\jp \uwp_\tm^\jp}$, 
        for $\ip =1,\ldots,\Np$
    }
    \caption{Random Weights Sequential Monte Carlo}\label{alg:rwsmc}
\end{algorithm}

Rather than interpreting random weights \gls{SMC} as a properly weighted \gls{SMC} algorithm, we can instead motivate it as a standard \gls{SMC} on an extended space.

\subsection{Nested Sequential Monte Carlo}\label{sec:pm:nsmc}
We already briefly touched on \gls{NSMC} in \cref{sec:proposal}, illustrating how a nested \gls{IS} sampler can be used to approximate the locally optimal proposal distribution. Here we will instead consider the case when the nested sampler is itself \gls{SMC} rather than \gls{IS}. This usually allows the user to derive more efficient proposal approximations, pushing the boundaries of the type of problems that \gls{SMC} methods can be effectively applied to.

\Gls{NSMC} \citep{naesseth2015nested,naesseth2016high} is specifically targeting models where the latent state $\xp_\tm$ is of dimension $\dimX > 1$ (typically in the order of 10's to 100's).
The goal of \gls{NSMC} is to approximate the locally optimal proposal distribution by running a separate internal (or nested) \gls{SMC} algorithm over the components of $\xm_\tm = (\xp_{\tm,1},\ldots,\xp_{\tm,\dimX})$ for each sample we need.  At first this might seem wasteful of our computational resources, running a nested \gls{SMC} sampler with $\Mp$ internal particles for each particle in the outer algorithm results in $\ordo(\Np\Mp)$ operations. However, for many models it results in much more accurate estimate than, for instance, a corresponding standard \gls{SMC} algorithm with $\Np\Mp$ particles and the prior proposal.

Unlike in the random weights \gls{SMC}, \gls{NSMC} changes not only the way we compute weights, but also the way we generate the samples $\xp_\tm^\ip$. The key idea is to construct a \gls{SMC} approximation to the locally optimal proposal $\prop_\tm^\star(\xp_\tm\given\xp_{1:\tm-1})$. For each particle $\xp_\tm^\ip$ we want to simulate we construct an \gls{SMC} approximation to $\prop_\tm^\star$
\begin{align}
    \widehat\prop_\tm^\star(\xp_{\tm}\given\xp_{1:\tm-1}^\ip) &= 
    \sum_{\jp=1}^\Mp \frac{\nuwp_{\dimX}^{\jp,\ip}}{\sum_\lm \nuwp_{\dimX}^{\lm,\ip}} 
    \dirac_{\nxp_{1:\dimX}^{\jp,\ip}}(\xp_\tm), 
    \label{eq:nsmc:prop}
\end{align}
where $(\nuwp_{\dimX}^{\jp,\ip},\nxp_{1:\dimX}^{\jp,\ip})$ are properly weighted for $\prop_\tm^\star(\xp_\tm\given\xp_{1:\tm-1}^\ip)$. We assume that \cref{eq:nsmc:prop} is constructed using a nested \gls{SMC} sampler on the components of $\xp_\tm$;  first targeting $\xp_{\tm,1}$, then $\xp_{\tm,1:2}$, \etc, where the final target is $\prop_\tm^\star(\xp_{\tm,1:\dimX}\given\xp_{1:\tm-1}^\ip)$. With this approximation \gls{NSMC} replaces the sampling step \cref{eq:smc:proposal} by simulating from \cref{eq:nsmc:prop}, and the weighting step \cref{eq:smc:weights} by
\begin{align}
	\uwp_\tm^\ip &= \prod_{\km=1}^{\dimX}\frac{1}{\Mp} \sum_{\jp=1}^\Mp \nuwp_{\km}^{\jp,\ip},
	\label{eq:nsmc:weights}
\end{align}
where $\nuwp_{\km}^{\jp,\ip}$ are the corresponding weights in the nested \gls{SMC} sampler. We summarize \gls{NSMC} in \cref{alg:nsmc} and give examples below of ways to construct $\widehat\prop_\tm^\star(\xp_{\tm}\given\xp_{1:\tm-1}^\ip)$.
\begin{algorithm}[tb]
    \SetKwInOut{Input}{input}\SetKwInOut{Output}{output}
    \Input{Unnormalized target distributions $\umod_\tm$, nested \gls{SMC} sampler targeting $\prop_\tm^\star(\xp_{\tm}\given\xp_{1:\tm-1})$, number of samples $\Np$ and nested samples $\Mp$.}
    \BlankLine
    \For{$\tm=1$ \KwTo $\Tm$}{
        \For{$\ip = 1$ \KwTo $\Np$}{
            Sample $\xp_{1:\tm}^\ip \sim 
            \widehat\nmod_{\tm-1}(\xp_{1:\tm-1})
            \widehat\prop_\tm^\star(\xp_{\tm}\given\xp_{1:\tm-1})$ 
            (see \cref{eq:nsmc:prop})\\
            Set $\uwp_\tm^\ip = \prod_{\km=1}^{\dimX}\frac{1}{\Mp} \sum_{\jp=1}^\Mp \nuwp_{\km}^{\jp,\ip}$
            (see \cref{eq:nsmc:weights})
        }
        Set $\nwp_\tm^\ip = \frac{\uwp_\tm^\ip}{\sum_\jp \uwp_\tm^\jp}$, 
        for $\ip =1,\ldots,\Np$
    }
    \caption{\glsreset{NSMC}\Gls{NSMC}}\label{alg:nsmc}
\end{algorithm}

\begin{example}[\Acrlong{NSMC} for \cref{ex:running}]
We revisit our running non-Markovian Gaussian example. Because $\xp_\tm$ is scalar for this example, we make a straightforward extension to the model to illustrate \gls{NSMC}. We replicate the model from \cref{ex:running} $\dimX$ times, \ie 
\begin{align*}
	\fmod_\tm(\xp_\tm \given \xp_{\tm-1}) &= \prod_{\km=1}^{\dimX} \fmod(\xp_{\tm,\km} \given \xp_{\tm-1,\km})  = \prod_{\km=1}^{\dimX} \Norm(\xp_{\tm,\km}\given \phi \xp_{\tm-1,\km}, q), \\
	\gmod_\tm(\y_\tm \given \xp_{1:\tm}) &= \prod_{\km=1}^{\dimX} \gmod(\y_{\tm,\km} \given \xp_{1:\tm,\km})  = \prod_{\km=1}^{\dimX} \Norm\left(\y_{\tm,\km}\given \sum_{\lm =1}^{\tm}\beta^{\tm-\lm}\xp_{\lm,\km}, r\right).
\end{align*}
To construct $\widehat\prop_\tm^\star(\xp_{\tm}\given\xp_{1:\tm-1}^\ip)$ we run an internal \gls{SMC} sampler with $\Mp$ samples targeting
\begin{align*}
	\normprop_\lm(\xp_{\tm,1:\lm}\given\xp_{1:\tm-1})  &\propto  \uprop_\lm(\xp_{\tm,1:\lm}\given\xp_{1:\tm-1}) \eqdef \prod_{\km=1}^{\lm} \fmod(\xp_{\tm,\km} \given \xp_{\tm-1,\km}) \gmod(\y_{\tm,\km} \given \xp_{1:\tm,\km}),
\end{align*}
for $\lm=1,\ldots,\dimX$. 
We denote the particles and weights for the nested \gls{SMC} by $\bar{\xp}_{1:\lm}^{\jp,\ip}$ and $\nuwp_\lm^{\jp,\ip}$, respectively.
A straightforward proposal for the above target distribution is $\nxp_{1:\lm}^{\jp,\ip} \sim \widehat{\normprop}_{\lm-1}(\xp_{\tm,1:\lm-1}\given\xp_{1:\tm-1}^\ip) \fmod(\xp_{\tm,\lm} \given \xp_{\tm-1,\lm}^\ip)$, \ie using the prior proposal. We have that $\widehat{\normprop}_\lm$ is
\begin{align*}
	\widehat{\normprop}_\lm(\xp_{\tm,1:\lm}\given\xp_{1:\tm-1}^\ip) &= \sum_{\jp=1}^{\Mp} \frac{\nuwp_{\lm}^{\jp,\ip}}{\sum_\mm \nuwp_{\lm}^{\mm,\ip}} 
    \dirac_{\nxp_{1:\lm}^{\jp,\ip}}(\xp_{\tm,1:\lm}),
\end{align*}
where the weights are given by
\begin{align*}
	\nuwp_{\lm}^{\jp,\ip} &= \gmod(\y_{\tm,\lm} \given (\xp_{1:\tm-1,\km}^\ip, \nxp_{\lm}^{\jp,\ip})) =
	\Norm\left(\y_{\tm,\lm}\given \nxp_{\lm}^{\jp,\ip}+\sum_{\mm =1}^{\tm-1}\beta^{\tm-\mm}\xp_{\mm,\lm}^\ip, r\right).
\end{align*}
We obtain the approximation to the locally optimal proposal $\widehat\prop_\tm^\star(\xp_{\tm}\given\xp_{1:\tm-1}^\ip)\equiv \widehat{\normprop}_{\dimX}(\xp_{\tm,1:\dimX}\given\xp_{1:\tm-1}^\ip)$. This defines the \gls{NSMC} method in \cref{alg:nsmc} for the multivariate extension to our running example. We illustrate the nested target distributions in \cref{fig:illustration:nsmc}.

\begin{figure}[t]
\centering
\begin{subfigure}[b]{0.48\textwidth}
\resizebox{1.0\textwidth}{!}{
\tikzstyle{edge} = [-]
\tikzstyle{edge2} = [->,very thick,>=latex]
\tikzstyle{edge3} = [->]
\tikzstyle{arrw} = [very thick,shorten <=2pt,shorten >=2pt]
\tikzstyle{var} = [draw,circle,inner sep=0,minimum width=0.7cm]
\tikzstyle{invisvar} = [draw=none,circle,inner sep=0,minimum width=0.7cm]
\tikzstyle{cvar} = [draw,circle,inner sep=0,minimum width=0.7cm, fill=black!20]
\tikzstyle{initvar} = [draw,circle,inner sep=0,minimum width=0.9cm]
\tikzstyle{obs} = [draw,circle,inner sep=0,minimum width=0.5cm, fill=black!20]
  \begin{tikzpicture}[>=stealth,node distance=0.6cm]
    \begin{scope}
      \foreach \x in {1} {
        \foreach \y in {0,1,2,3} {
          \pgfmathtruncatemacro\xend{\x+1}
          \pgfmathtruncatemacro\yend{4-\y}
          \node at (1.9*\x,\y) (x\x\y) [cvar] {};
        }
      }
      \foreach \x in {2} {
        \foreach \y in {3} {
          \pgfmathtruncatemacro\xend{\x+1}
          \pgfmathtruncatemacro\yend{4-\y}
          \node at (1.9*\x,\y) (x\x\y) [var] {$x_{t,\yend}$};
          \node at (1.9*\x+0.9,\y-0.15) (y\x\y) [obs] {};
        }
      }
      \foreach \x in {2} {
        \foreach \y in {0} {
          \pgfmathtruncatemacro\xend{\x+1}
          \pgfmathtruncatemacro\yend{4-\y}
          \node at (1.9*\x,\y) (x\x\y) [invisvar] {};
        }
      }
      
      \foreach \x in {1} {
        \pgfmathtruncatemacro\xend{\x+1}
        \node[draw,very thick,rectangle,rounded corners=3mm,minimum width=0.6cm,fit=(x\x0) (x\x3), label=above:{$x_{t-1}$}] (x\x){};
      }
      \foreach \x in {2} {
        \pgfmathtruncatemacro\xend{\x+1}
        \node[draw,very thick,rectangle,rounded corners=3mm,minimum width=0.6cm,fit=(x\x0) (x\x3), label=above:{$x_{t,1}$}] (x\x){};
      }
      
      \foreach \x in {1} {
        \pgfmathtruncatemacro\xend{\x+1}
          \draw[edge2] (x\x) -> (x\xend);
      }
       Draw diagonal edges
      \foreach \x in {2} {
        \pgfmathtruncatemacro\xend{\x+1}
        \foreach \y in {3} {
          \draw[edge3] (x\x\y) -- (y\x\y);
        }
      }
      \node at (0.5,1.5) (dots) [circle] {$\cdots$};
    \end{scope}
  \end{tikzpicture}
}
\caption{$\normprop_1(\xp_{\tm,1}\given\xp_{1:\tm-1})$}
\end{subfigure}
\begin{subfigure}[b]{0.48\textwidth}
\resizebox{1.0\textwidth}{!}{
\tikzstyle{edge} = [-]
\tikzstyle{edge2} = [->,very thick,>=latex]
\tikzstyle{edge3} = [->]
\tikzstyle{arrw} = [very thick,shorten <=2pt,shorten >=2pt]
\tikzstyle{var} = [draw,circle,inner sep=0,minimum width=0.7cm]
\tikzstyle{cvar} = [draw,circle,inner sep=0,minimum width=0.7cm, fill=black!20]
\tikzstyle{initvar} = [draw,circle,inner sep=0,minimum width=0.9cm]
\tikzstyle{obs} = [draw,circle,inner sep=0,minimum width=0.5cm, fill=black!20]
\begin{tikzpicture}[>=stealth,node distance=0.6cm]
    \begin{scope}
      \foreach \x in {1} {
        \foreach \y in {0,1,2,3} {
          \pgfmathtruncatemacro\xend{\x+1}
          \pgfmathtruncatemacro\yend{4-\y}
          \node at (1.9*\x,\y) (x\x\y) [cvar] {};
        }
      }
      \foreach \x in {2} {
        \foreach \y in {2,3} {
          \pgfmathtruncatemacro\xend{\x+1}
          \pgfmathtruncatemacro\yend{4-\y}
          \node at (1.9*\x,\y) (x\x\y) [var] {$x_{t,\yend}$};
          \node at (1.9*\x+0.9,\y-0.15) (y\x\y) [obs] {};
        }
      }
      
      \foreach \x in {1} {
        \pgfmathtruncatemacro\xend{\x+1}
        \node[draw,very thick,rectangle,rounded corners=3mm,minimum width=0.6cm,fit=(x\x0) (x\x3), label=above:{$x_{t-1}$}] (x\x){};
      }
      \foreach \x in {2} {
        \pgfmathtruncatemacro\xend{\x+1}
        \node[draw,very thick,rectangle,rounded corners=3mm,minimum width=0.6cm,fit=(x\x0) (x\x3), label=above:{$x_{t,1:2}$}] (x\x){};
      }
      
      \foreach \x in {1} {
        \pgfmathtruncatemacro\xend{\x+1}
          \draw[edge2] (x\x) -> (x\xend);
      }
      \foreach \x in {2} {
        \pgfmathtruncatemacro\xend{\x+1}
        \foreach \y in {2,3} {
          \draw[edge3] (x\x\y) -- (y\x\y);
        }
      }
      \foreach \x in {1} {
        \foreach \y in {0,1,2} {
          \pgfmathtruncatemacro\yend{\y+1}
          \draw[edge] (x\x\y) -- (x\x\yend) {};
        }
      }
      \node at (0.5,1.5) (dots) [circle] {$\cdots$};
    \end{scope}
  \end{tikzpicture}
}
\caption{$\normprop_2(\xp_{\tm,1:2}\given\xp_{1:\tm-1})$}
\end{subfigure}

\begin{subfigure}[b]{0.48\textwidth}
\resizebox{1.0\textwidth}{!}{
\tikzstyle{edge} = [-]
\tikzstyle{edge2} = [->,very thick,>=latex]
\tikzstyle{edge3} = [->]
\tikzstyle{arrw} = [very thick,shorten <=2pt,shorten >=2pt]
\tikzstyle{var} = [draw,circle,inner sep=0,minimum width=0.7cm]
\tikzstyle{cvar} = [draw,circle,inner sep=0,minimum width=0.7cm, fill=black!20]
\tikzstyle{initvar} = [draw,circle,inner sep=0,minimum width=0.9cm]
\tikzstyle{obs} = [draw,circle,inner sep=0,minimum width=0.5cm, fill=black!20]
\begin{tikzpicture}[>=stealth,node distance=0.6cm]
    \begin{scope}
      \foreach \x in {1} {
        \foreach \y in {0,1,2,3} {
          \pgfmathtruncatemacro\xend{\x+1}
          \pgfmathtruncatemacro\yend{4-\y}
          \node at (1.9*\x,\y) (x\x\y) [cvar] {};
        }
      }
      \foreach \x in {2} {
        \foreach \y in {1,2,3} {
          \pgfmathtruncatemacro\xend{\x+1}
          \pgfmathtruncatemacro\yend{4-\y}
          \node at (1.9*\x,\y) (x\x\y) [var] {$x_{t,\yend}$};
          \node at (1.9*\x+0.9,\y-0.15) (y\x\y) [obs] {};
        }
      }
      
      \foreach \x in {1} {
        \pgfmathtruncatemacro\xend{\x+1}
        \node[draw,very thick,rectangle,rounded corners=3mm,minimum width=0.6cm,fit=(x\x0) (x\x3), label=above:{$x_{t-1}$}] (x\x){};
      }
      \foreach \x in {2} {
        \pgfmathtruncatemacro\xend{\x+1}
        \node[draw,very thick,rectangle,rounded corners=3mm,minimum width=0.6cm,fit=(x\x0) (x\x3), label=above:{$x_{t,1:3}$}] (x\x){};
      }
      
      \foreach \x in {1} {
        \pgfmathtruncatemacro\xend{\x+1}
          \draw[edge2] (x\x) -> (x\xend);
      }
      \foreach \x in {2} {
        \pgfmathtruncatemacro\xend{\x+1}
        \foreach \y in {1,2,3} {
          \draw[edge3] (x\x\y) -- (y\x\y);
        }
      }
      \foreach \x in {1} {
        \foreach \y in {0,1,2} {
          \pgfmathtruncatemacro\yend{\y+1}
          \draw[edge] (x\x\y) -- (x\x\yend) {};
        }
      }
      \node at (0.5,1.5) (dots) [circle] {$\cdots$};
    \end{scope}
  \end{tikzpicture}
}
\caption{$\normprop_3(\xp_{\tm,1:3}\given\xp_{1:\tm-1})$}
\end{subfigure}
\begin{subfigure}[b]{0.48\textwidth}
\resizebox{1.0\textwidth}{!}{
\tikzstyle{edge} = [-]
\tikzstyle{edge2} = [->,very thick,>=latex]
\tikzstyle{edge3} = [->]
\tikzstyle{arrw} = [very thick,shorten <=2pt,shorten >=2pt]
\tikzstyle{var} = [draw,circle,inner sep=0,minimum width=0.7cm]
\tikzstyle{cvar} = [draw,circle,inner sep=0,minimum width=0.7cm, fill=black!20]
\tikzstyle{initvar} = [draw,circle,inner sep=0,minimum width=0.9cm]
\tikzstyle{obs} = [draw,circle,inner sep=0,minimum width=0.5cm, fill=black!20]
\begin{tikzpicture}[>=stealth,node distance=0.6cm]
    \begin{scope}
      \foreach \x in {1} {
        \foreach \y in {0,1,2,3} {
          \pgfmathtruncatemacro\xend{\x+1}
          \pgfmathtruncatemacro\yend{4-\y}
          \node at (1.9*\x,\y) (x\x\y) [cvar] {};
        }
      }
      \foreach \x in {2} {
        \foreach \y in {0,1,2,3} {
          \pgfmathtruncatemacro\xend{\x+1}
          \pgfmathtruncatemacro\yend{4-\y}
          \node at (1.9*\x,\y) (x\x\y) [var] {$x_{t,\yend}$};
          \node at (1.9*\x+0.9,\y-0.15) (y\x\y) [obs] {};
        }
      }
      
      \foreach \x in {1} {
        \pgfmathtruncatemacro\xend{\x+1}
        \node[draw,very thick,rectangle,rounded corners=3mm,minimum width=0.6cm,fit=(x\x0) (x\x3), label=above:{$x_{t-1}$}] (x\x){};
      }
      \foreach \x in {2} {
        \pgfmathtruncatemacro\xend{\x+1}
        \node[draw,very thick,rectangle,rounded corners=3mm,minimum width=0.6cm,fit=(x\x0) (x\x3), label=above:{$x_{t,1:4}$}] (x\x){};
      }
      
      \foreach \x in {1} {
        \pgfmathtruncatemacro\xend{\x+1}
          \draw[edge2] (x\x) -> (x\xend);
      }
      \foreach \x in {2} {
        \pgfmathtruncatemacro\xend{\x+1}
        \foreach \y in {0,1,2,3} {
          \draw[edge3] (x\x\y) -- (y\x\y);
        }
      }
      \node at (0.5,1.5) (dots) [circle] {$\cdots$};
    \end{scope}
  \end{tikzpicture}
}
\caption{$\normprop_4(\xp_{\tm,1:4}\given\xp_{1:\tm-1})$}
\end{subfigure}
\caption{Illustration of \gls{NSMC} for the running example with $\dimX = 4$. Note that the each variable $\xp_{\tm,\lm}$ depends on the complete sequence of variables from the previous steps $\xp_{1:\tm-1}$.} \label{fig:illustration:nsmc}
\end{figure}
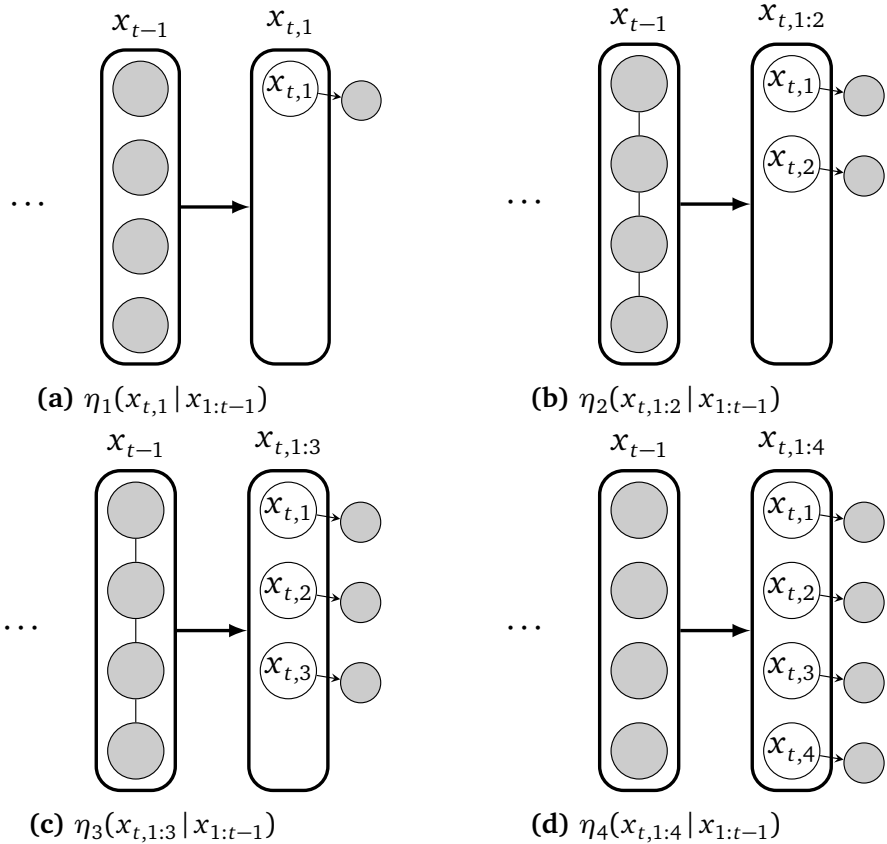

Just like in the standard \gls{SMC} sampler, there is room for improvement in the nested \gls{SMC} by \eg choosing a better internal proposal distributions and target distributions under the proper weighting constraint. It is possible to further improve performance by correlating samples through a combination of fully adapted \gls{SMC} and a forward-filtering backward-sampling procedure as discussed by \citet{naesseth2015nested,naesseth2016high}.
\label{ex:nsmc}
\end{example}

\section{Distributed Sequential Monte Carlo}\label{sec:islandsmc}
Developing computational methods that can take advantage of parallel and multi-core computing architectures is important for efficient approximate inference algorithms. 
In this section we discuss how the idea of nesting \gls{MC} algorithms can open up for distributed implementation of \gls{SMC}.
First we discuss independent importance weighted \gls{SMC} samplers, a way to combine the outputs from multiple \gls{SMC} samplers using an outer level importance sampler. Then we discuss the island particle filter \citep{verge2015parallel}, where the outer level sampling is done using \gls{SMC} as well.
It should be noted that the purpose of this section is to illustrate how the idea of nested \gls{MC} opens up for parallelization of \gls{SMC}, and
not to provide a comprehensive overview of distributed \gls{SMC} algorithms. Indeed, many useful methods have been proposed for parallelizing the \gls{SMC} algorithm itself. Often this is done at the granularity of individual particles. The key challenge is then to parallelize the resampling step, since this step requires interaction between particles. One way to address this challenge is to make use of \gls{MH} or rejection sampling to implement the resampling \citep{murray2016parallel}. Another approach is to modify the resampling step in order to limit the interaction between particles, thereby making a distributed implementation easier \citep{whiteley2016role}.


\subsection{Importance Weighted Independent SMC Samplers}
Assume that we have access to $M$ computing nodes. These could correspond to different machines, different cores, or even $M$ consecutive runs on the same machine.
A (very) simple way of distributing computation across these nodes is to run separate independent \gls{SMC} samplers on each node. Using the notation from above, let $\up^m \sim \rsdist(\up)$ denote all the random variables generated during the run of an \gls{SMC} sampler with $\Np$ particles in node $m\in\{1,\dots,M\}$.
The variable $\up^m$ can also be thought of as the random seed used to initialize the $m$:th sampler.
We then get $m$ independent approximations of the target
$\nmod_\Tm$ given by
\begin{align*}
\widehat\nmod_\Tm(\xp_{1:\Tm}\given\up^m) &= \sum_{\ip=1}^\Np \nwp_\Tm^{m,\ip} 
\dirac_{\xp_{1:\Tm}^{m,\ip}}(\xp_{1:\Tm}),
\end{align*}
where we have assumed that each node uses the same number of particles $\Np$ and the same \gls{SMC} algorithm.
Assume that we want to compute the expected value $\Exp_{\nmod_\Tm}[\fun(\xp_{1:\Tm})]$ of some test function $\fun$. Each node provides an estimate,
\begin{align}
	\widehat\fun(\up^m) \eqdef
	\sum_{\ip=1}^\Np \nwp_\Tm^{m,\ip} 
	\fun(\xp_{1:\Tm}^{m,\ip}),
		\label{eq:smc:distr:single-est}
\end{align}
and since these estimates are independent and identically distributed a natural approach is to take the average
\(
\frac{1}{M} \sum_{m=1}^M \widehat\fun(\up^m)
\)
as the final estimate. However, since each \gls{SMC} estimate is biased for finite $\Np$, the resulting average will also be biased regardless of the number of nodes $M$ used in the computation. A better approach can therefore be to weight the independent \gls{SMC} samplers according to their respective normalizing constant estimates:
\begin{align}
	\widehat\fun_{\text{aggr}} \eqdef \sum_{m=1}^M \frac{\widehat \Z_\Tm(\up^m)}{\sum_n \widehat \Z_\Tm(\up^n)} \widehat\fun(\up^m).
		\label{eq:smc:distr:is-of-smc}
\end{align}
This aggregated estimator is obtained by viewing the \gls{SMC} sampler in each node as a proposal for an outer-level importance sampler. Specifically, the target distribution for this importance sampler is given by 
\begin{align*}
	\Pi(\up) \eqdef \frac{\widehat\Z_\Tm(\up)\rsdist(\up)}{\Z_\Tm},
\end{align*}
and the proposal is given by $\rsdist(\up)$, resulting in weights $\Pi(\up)/\rsdist(\up) \propto \widehat\Z_\Tm(\up)$.
%
By computations analogous to those in \cref{sec:pmh} it holds that
\(
	\Exp_\Pi[\widehat\fun(\up)] = \Exp_\nmod[\fun(\xp_{1:\Tm})].
\)
Consequently, by a standard importance sampling argument it holds that \cref{eq:smc:distr:is-of-smc} is a consistent estimate of $\Exp_\nmod[\fun(\xp_{1:\Tm})]$ as $M$ becomes large, regardless of the number $\Np$ of particles used in each \gls{SMC} sampler. 

It is not expected that \cref{eq:smc:distr:is-of-smc} will be a better estimate than what we would obtain from a single \gls{SMC} sampler with a total of $\Np\Mp$ particles. However, as pointed out above, the computation of the estimators \cref{eq:smc:distr:is-of-smc} can be distributed, or even performed sequentially on the same machine in situations when memory constraints limit the number of particles than can be used.

\subsection{The Island Particle Filter}
In the preceding section we considered \acrlong{IS} of \gls{SMC} estimators. It is possible to take this one step further and develop \gls{SMC} sampling of \gls{SMC} estimators. This is the idea behind the island particle filter by \citet{verge2015parallel}.

Similarly to above, the island particle filter consists of $M$ nodes, each running an \gls{SMC} sampler with $\Np$ particles. However, these $M$ samplers are allowed to interact at each iteration according to the standard resampling, propagation, and weighting steps of \gls{SMC}. We provide pseudo-code in \cref{alg:smc:island-pf}, where we assume that resampling at the outer level only occurs when the \acrfull{ESS} drops below some threshold (see \cref{sec:ess}). This type of adaptive resampling is recommended for a distributed implementation of the island particle filter. Indeed, resampling at the outer level at each iteration would result in very high communication costs, since it involves copying complete \gls{SMC} states (particles and weights) from node to node. Furthermore, we can control the outer level \gls{ESS} by increasing $\Np$, thereby ensuring that outer level resampling happens rarely when done adaptively.

\newcommand\essisland{\mathrm{ESS}^{\mathrm{island}}}
\newcommand\ancisland{A}
\newcommand\nwpisland{\Omega}
\newcommand\uwpisland{\widetilde \Omega}
\newcommand\vp{\mathbf{v}}

\begin{algorithm}
	\caption{The island particle filter}
	\SetKwInOut{Input}{input}\SetKwInOut{Output}{output}
	\Input{Unnormalized target distributions $\umod_\tm$, proposals 
		$\prop_\tm$, number of nodes $M$ and number of samples per node $\Np$.}
	\BlankLine

	\For{$m = 1$ \KwTo $M$}{
		%
		\tcp{Initialize inner SMC samplers}
		Sample $\xp_{1}^{m,\ip} \sim \prop_1(\xp_1)$ and set $\uwp_1^{m,\ip} = \frac{\umod_1(\xp_{1}^{m,\ip})}{\prop_1(\xp_{1}^{m,\ip})}$ for $\ip =1,\ldots,\Np.$ \\
		Set $\nwp_1^{m,\ip} = \frac{\uwp_1^{m,\ip}}{\sum_\jp \uwp_1^{m,\jp}}$ for $\ip =1,\ldots,\Np.$ \\
		\tcp{Initialize outer SMC sampler}
		Set $\vp_1^m = \{ \xp_{1}^{m,\ip}, \uwp_1^{m,\ip} \}_{\ip=1}^\Np$.\\
    	Set $\uwpisland_1^m = \frac{1}{N}\sum_{\ip=1}^\Np \uwp_1^{m,\ip}$. 	
	}
   	Set $\nwpisland_1^{m} = \frac{\uwpisland_1^{m}}{\sum_n \uwpisland_1^{n}}$ for
   	$m =1,\ldots,M$   
	
	\For{$\tm=2$ \KwTo $\Tm$}{
		Compute $\essisland_{\tm-1} = \left[\sum_{m=1}^M \left(\nwpisland_{\tm-1}^m 
		\right)^2\right]^{-1}.$
					
		\If{$\essisland_{\tm-1} < \mathrm{threshold}$}{
			\tcp{Resample islands (overload notation)}
			$\vp_{\tm-1}^{1:M} \sim \text{Resampling}(\vp_{\tm-1}^{1:M}, \nwpisland_{\tm-1}^{1:M})$.\\
			$\nwpisland_{\tm-1}^m \gets \frac{1}{M}$, $m=1,\dots,M$.
		}
		\For{$m = 1$ \KwTo $M$}{
			\tcp{Update inner SMC samplers}
			Sample $\xp_{1:\tm}^{m,\ip} \sim 
			\widehat\nmod_{\tm-1}(\xp_{1:\tm-1} \given \vp_{t-1}^m)
			\prop_\tm(\xp_\tm\given\xp_{1:\tm-1})$, for $\ip =1,\ldots,\Np.$ \\
			Set $\uwp_\tm^{m,\ip} = \frac{\umod_\tm(\xp_{1:\tm}^{m,\ip})}
			{\umod_{\tm-1}(\xp_{1:\tm-1}^{m,\ip}) \prop_\tm(\xp_{\tm}^{m,\ip}\given \xp_{1:\tm-1}^{m,\ip})}$, for $\ip =1,\ldots,\Np.$ \\
			Set $\nwp_\tm^{m,\ip} = \frac{\uwp_\tm^{m,\ip}}{\sum_\jp \uwp_\tm^{m,\jp}}$, 
	 		for $\ip =1,\ldots,\Np.$ \\
	 		\tcp{Update outer SMC sampler}
			Set $\vp_{\tm}^m = \{ \xp_{1:\tm}^{m,\ip}, \uwp_\tm^{m,\ip} \}_{\ip=1}^\Np$.\\
	   		Set $\uwpisland_\tm^m = \frac{\nwpisland_{\tm-1}^m}{1/M} \cdot
	   		\left[ \frac{1}{N}\sum_{\ip=1}^\Np \uwp_\tm^{m,\ip} \right]$.
	 	}	 	
	 	Set $\nwpisland_\tm^{m} = \frac{\uwpisland_\tm^{m}}{\sum_n \uwpisland_\tm^{n}}$ for
	 	$m =1,\ldots,M.$   
	}
	\label{alg:smc:island-pf}
\end{algorithm}

The island particle filter fits into the standard \gls{SMC} framework. To make this explicit we start by writing $\up = \up_{1:\Tm}$ where $\up_\tm$ corresponds to all the basic random variables generated \emph{at iteration $\tm$} of an (internal) \gls{SMC} sampler with $\Np$ particles.
For example, $\up_\tm$ could correspond to a collection of independent uniform random variables needed to implement the resampling and propagation step at iteration $\tm$. Equivalently, we can view $\up_\tm$ as the random seed used for generating these variables.
By making the (non restrictive) assumption that these auxiliary variables are drawn independently at each iteration we can write
\[
	\rsdist(\up) = \prod_{\tm=1}^\Tm \rsdist_\tm(\up_\tm).
\]
The island particle filter is then a standard \gls{SMC} sampler with $M$ particles, targeting the sequence of distributions
\begin{align*}
	\Pi_\tm(\up_{1:\tm}) \eqdef
	\frac{\widehat\Z_\tm(\up_{1:\tm}) \left[ \prod_{\km=1}^\tm \rsdist_\km(\up_\km) \right]}{\Z_\tm}
\end{align*}
and using $\rsdist_\tm(\up_\tm)$ as a proposal distribution at iteration $\tm$.
Similarly to the preceding section, targeting these ``extended distributions'' is motivated by the fact that
\(
\Exp_{\Pi_\tm}[\widehat\fun_\tm({\up_{1:\tm}})] = \Exp_\nmod[\fun_\tm(\xp_{1:\tm})],
\)
where $\widehat\fun_\tm({\up_{1:\tm}})$ is defined analogously to \eqref{eq:smc:distr:single-est}.

If we denote the outer level (island level) unnormalized \gls{SMC} weights by $\{\uwpisland_\tm^m\}_{m=1}^M$ we get
\begin{align*}
	\uwpisland_\tm^m = 
	\frac{\widehat\Z_\tm(\up_{1:\tm}^m)}{\widehat\Z_{\tm-1}(\up_{1:\tm}^m)}
	= \frac{1}{N}\sum_{\ip=1}^\Np \uwp_\tm^{m,\ip}
\end{align*}
at iterations when the islands are resampled, and
\begin{align*}
	\uwpisland_\tm^m = \frac{\nwpisland_{\tm-1}^m}{1/M} \cdot
	\left[ \frac{1}{N}\sum_{\ip=1}^\Np \uwp_\tm^{m,\ip} \right],
\end{align*}
at iterations when the outer level resampling is omitted; see \cref{eq:smc:ess-weights}. For implementation purposes we can use the latter expression at all iterations if we overwrite $\nwpisland_{\tm-1}^m \gets \nicefrac{1}{M}$ whenever the islands are resampled. In the above expressions, $\{ \uwp_\tm^{m,\ip} \}_{\ip=1}^M$ denote the unnormalized (inner level) \gls{SMC} weights at node $M$.

The derivation above is based on viewing the complete random seed as an auxiliary variable. In practice it is usually more convenient to simply keep track of the resulting particle trajectories and weights, 
which can be computed deterministically given $\up_{1:\tm}$. Thus, in \cref{alg:smc:island-pf} we have
identified $\up_{1:\tm}$ with  $\vp_\tm \eqdef \left\{\left(\xp_{1:\tm}^{\ip}, \uwp_\tm^{\ip}\right)\right\}_{\ip=1}^\Np$
and express the steps of the algorithm using the latter variable.

\newpage
\begin{subappendices}
\section{Proof of Unbiasedness}\label{sec:smc:unbiasedZ}
For clarity we prove the unbiasedness of the normalizing constant estimate for the version of \gls{SMC} presented in \cref{alg:esmc:smc}, which involves multinomial resampling at each iteration of the algorithm. We note, however, that the property holds more generally, for instance when using low-variance and adaptive resampling. The proof presented here can be straightforwardly extended to such cases.

The distribution of all random variables generated by the \gls{SMC} 
method in \cref{alg:esmc:smc} is given by
\begin{align}
    \qSMC\left(\xp_{1:\Tm}^{1:\Np}, \ap_{1:\Tm-1}^{1:\Np}\right) &= 
    \prod_{\ip=1}^\Np \prop_1(\xp_1^\ip) \cdot \prod_{\tm=2}^\Tm 
    \prod_{\ip=1}^\Np \left[\nwp_{\tm-1}^{\ap_{\tm-1}^\ip} 
    \prop_\tm\left(\xp_\tm^\ip \given 
    \xp_{1:\tm-1}^{\ap_{\tm-1}^\ip}\right)\right],
    \label{eq:smc:smcdist}
\end{align}
where the particles are $\xp_{1:\Tm}^{1:\Np} = \bigcup_{\tm=1}^\Tm 
\left\{\xp_\tm^\ip\right\}_{\ip=1}^\Np$, and the ancestor variables from 
the resampling step are $\ap_{1:\Tm-1}^{1:\Np} = \bigcup_{\tm=1}^{\Tm-1} 
\left\{\ap_\tm^\ip\right\}_{\ip=1}^\Np$. Assuming that $\uwp_\tm^\ip 
\geq 0$ for all $\ip$ and $\tm$, we have that $\hatZ_\Tm\geq 0$. We are 
now going to prove that
\begin{align}
    \Exp_{\qSMC\left(\xp_{1:\Tm}^{1:\Np}, \ap_{1:\Tm-1}^{1:\Np}\right)}\left[\frac{\widehat\Z_\Tm}{\Z_\Tm}\right] = 1,
    \label{eq:smc:hatZexp}
\end{align}
from which the result follows. Note that the results holds for any 
value of $\Tm$. To prove the result we introduce another 
set of variables $\bp_\Tm^\ip = \ip$ and $\bp_\tm^\ip = \ap_\tm^{\bp_{\tm+1}^\ip}$, for 
$\tm=1,\ldots,\Tm-1$. The notation describes the ancestor index at 
iteration $\tm$ for the final particle $\xp_{1:\Tm}^\ip$. This means 
that we can write $\xp_{1:\Tm}^\ip = \xp_{1:\Tm}^{\bp_{1:\Tm}^\ip} =
(\xp_1^{\bp_1^\ip},\ldots,\xp_\Tm^{\bp_\Tm^\ip})$. Using 
this notation we rewrite the integrand in \cref{eq:smc:hatZexp},
\begin{align*}
    &\frac{\widehat\Z_\Tm }{\Z_\Tm} \qSMC\left(\xp_{1:\Tm}^{1:\Np}, 
    \ap_{1:\Tm-1}^{1:\Np}\right) \\
    &=\frac{1}{\Z_\Tm} \prod_{\tm=1}^\Tm \left[
    \frac{1}{\Np}\sum_{\ip=1}^\Np \uwp_\tm^\ip \right] \cdot \prod_{\ip=1}^\Np \prop_1(\xp_1^\ip) \cdot \prod_{\tm=2}^\Tm 
    \prod_{\ip=1}^\Np \left[ 
    \frac{\uwp_{\tm-1}^{\ap_{\tm-1}^\ip}}{\sum_{\jp=1}^\Np 
    \uwp_{\tm-1}^\jp}
    \prop_\tm\left(\xp_\tm^\ip \given 
    \xp_{1:\tm-1}^{\ap_{\tm-1}^\ip}\right)\right] \\
    &= \frac{1}{\Np^{\Tm} \Z_\Tm} \sum_{\ip=1}^\Np \left[
    \uwp_1^{\bp_1^\ip}\prop_1(\xp_1^{\bp_1^\ip})\prod_{\tm=2}^\Tm 
    \uwp_\tm^{\bp_\tm^\ip} 
    \prop_\tm(\xp_\tm^{\bp_\tm^\ip}\given\xp_{1:\tm-1}^{\bp_{\tm-1}^\ip})  \right. \\
    & \hspace{3cm}\left. \cdot \prod_{\jp\neq 
    \bp_1^\ip} \prop_1(\xp_1^\jp) \cdot \prod_{\tm=2}^\Tm 
    \prod_{\jp\neq \bp_{\tm-1}^\ip} \left[ 
    \nwp_{\tm-1}^{\ap_{\tm-1}^\jp}
    \prop_\tm\left(\xp_\tm^\jp \given 
    \xp_{1:\tm-1}^{\ap_{\tm-1}^\jp}\right)\right] \right] \\
    &= \frac{1}{\Np^{\Tm}} \sum_{\ip=1}^\Np \left[
    \frac{\umod_\Tm(\xp_{1:\Tm}^{\bp_{1:\Tm}^\ip})}{\Z_\Tm} \cdot \prod_{\jp\neq 
    \bp_1^\ip} \prop_1(\xp_1^\jp) \cdot \prod_{\tm=2}^\Tm
    \prod_{\jp\neq \bp_{\tm-1}^\ip} \left[ 
    \nwp_{\tm-1}^{\ap_{\tm-1}^\jp}
    \prop_\tm\left(\xp_\tm^\jp \given 
    \xp_{1:\tm-1}^{\ap_{\tm-1}^\jp}\right)\right] \right],
\end{align*}
where the first equality is the definition, the second equality 
separates the dependence of the particle $\xp_{1:\Tm}^\ip = \xp_{1:\Tm}^{\bp_{1:\Tm}^\ip}$ from the 
rest, and the third equality simplifies the product between weights 
and proposals for particle $\xp_{1:\Tm}^\ip = \xp_{1:\Tm}^{\bp_{1:\Tm}^\ip}$.

Now, using \cref{eq:smc:hatZexp} and the above rewrite of the integrand 
we obtain
\begin{align*}
    &\frac{1}{\Np^{\Tm}}\sum_{\ap_{1:\Tm-1}^{1:\Np}} \int  \sum_{\ip=1}^\Np \left[
    \frac{\umod_\Tm(\xp_{1:\Tm}^{\bp_{1:\Tm}^\ip})}{\Z_\Tm} \cdot \prod_{\jp\neq 
    \bp_1^\ip} \prop_1(\xp_1^\jp) \cdot \prod_{\tm=2}^\Tm
    \prod_{\jp\neq \bp_{\tm-1}^\ip} \left[ 
    \nwp_{\tm-1}^{\ap_{\tm-1}^\jp}
    \prop_\tm\left(\xp_\tm^\jp \given 
    \xp_{1:\tm-1}^{\ap_{\tm-1}^\jp}\right)\right] \right] \dif 
    \xp_{1:\Tm}^{1:\Np} \\
    &=\frac{1}{\Np^{\Tm-1}}\sum_{\ap_{1:\Tm-1}^{1:\Np}} \int  
    \frac{\umod_\Tm(\xp_{1:\Tm}^{\bp_{1:\Tm}^1})}{\Z_\Tm} \cdot \prod_{\jp\neq 
    \bp_1^1} \prop_1(\xp_1^\jp) \cdot \prod_{\tm=2}^\Tm
    \prod_{\jp\neq \bp_{\tm-1}^1} \left[ 
    \nwp_{\tm-1}^{\ap_{\tm-1}^\jp}
    \prop_\tm\left(\xp_\tm^\jp \given 
    \xp_{1:\tm-1}^{\ap_{\tm-1}^\jp}\right)\right]  \dif 
    \xp_{1:\Tm}^{1:\Np} \\
    &= \frac{1}{\Np^{\Tm-1}}\sum_{\bp_{1:\Tm-1}^{1}} \int  
    \frac{\umod_\Tm(\xp_{1:\Tm}^{\bp_{1:\Tm}^1})}{\Z_\Tm}   \dif 
    \xp_{1:\Tm}^{1} 
    = \int \frac{\umod_\Tm(\xp_{1:\Tm})}{\Z_\Tm}   \dif 
    \xp_{1:\Tm} = 1,
\end{align*}
where in the first equality we note that the sum over $\ip$ 
generates $\Np$ equal values, and thus we can arbitrarily choose one 
of these (in this case $\ip=1$) to evaluate and multiply the result by 
$\Np$. In the second equality we marginalize the variables not 
involved in the particle $\xp_{1:\Tm}^{\bp_{1:\Tm}^1}$. The two final equalities 
follow because we are averaging $\Np^{\Tm-1}$ values that are all 
equal to $1$. This concludes the proof.
\end{subappendices}

\chapter{Conditional SMC: Algorithms and Applications}\label{sec:csmc}
\Acrfull{CSMC} is a recent algorithm combining \gls{SMC} and \gls{MCMC} \citep{andrieu2010particle}, originally developed as a means to approximate and mimic idealized Gibbs samplers for local latent variable models. The intractability of the exact Gibbs sampler stems from the challenge in simulating from the exact conditional of latent variables given the data and the parameters. The key idea in \gls{CSMC} is to simulate from this conditional distribution approximately using a slightly modified \gls{SMC} approximation, retaining a single particle at each iteration. By iterating this procedure we obtain a \gls{MCMC} method that has the desired target distribution as its stationary distribution. We will see how this can be used not only for simulating from a target distribution, but also how we can use it to generate unbiased inverse normalization constant estimates, speed up inference for \eg probabilistic programming, and evaluating other approximate inference methods.

In the first section below we introduce the \gls{CSMC} algorithm and illustrate it with an example. Then, we discuss the unbiasedness property of the 
\gls{CSMC} inverse normalization constant estimate, a key property underpinning 
the rest of the section. Using \gls{CSMC} as a tool we move on by deriving the distribution of the expected \gls{SMC} empirical distribution approximation. Next, we introduce the \gls{IPMCMC} algorithm, leveraging \gls{CSMC} and \gls{SMC} to derive a method that can take advantage of distributed and multi-core computational architectures. 
Finally, we discuss ways in which \gls{SMC} and \gls{CSMC} can be used for evaluating approximate inference.

\section{Conditional Sequential Monte Carlo}\label{sec:csmc:csmc}
One way to look at \gls{CSMC} is to view it as  Markov kernel with the target distribution $\nmod_\Tm$ as its stationary distribution. Each iteration takes the previous sample as an input, a reference trajectory, and outputs an updated sample with no need for an accept--reject step. If we keep iterating, applying \gls{CSMC} over and over, we obtain the \emph{iterated} \gls{CSMC} algorithm. The iterated \gls{CSMC} sampler is a type of Gibbs sampler on an extended space of all random variables generated at each step \citep{andrieu2010particle}. We will sometimes refer to the previous sample $\xp_{1:\Tm}'$, \ie the reference trajectory, as the \emph{retained particle}.

The \gls{CSMC} algorithm differs from standard \gls{SMC} described in \cref{sec:smc} only in that a single particle is set deterministically to the retained particle from the previous iteration. This means that we condition on the event that a specific particle survives and is equal to the retained particle $\xp_{1:\Tm}'$, hence the name \emph{conditional} \gls{SMC}. We obtain the \gls{CSMC} algorithm by a simple modification of the \gls{SMC} proposal and weight update equations. At each iteration, for the first $\Np-1$ particles, $\ip =1,\ldots,\Np-1$, we follow the standard updates according to \cref{eq:smc:proposal,eq:smc:weights}. The proposal is
\begin{align}
	\xp_{1:\tm}^\ip &\sim \widehat\nmod_{\tm-1}(\xp_{1:\tm}) \prop_\tm(\xp_\tm\given\xp_{1:\tm-1}),
	\label{eq:csmc:proposal}
\end{align}
and the weights
\begin{align}
	\uwp_{\tm}^\ip &= \frac{\umod_{\tm}\left(\xp_{1:\tm}^\ip\right)}{\umod_{\tm-1}\left(\xp_{1:\tm-1}^\ip\right)\prop_\tm(\xp_\tm^\ip\given\xp_{1:\tm-1}^\ip)} .
	\label{eq:csmc:weights}
\end{align}
The last particle $\ip=\Np$ we set deterministically equal to the retained particle $\xp_{1:\tm}^\Np = \xp_{1:\tm}'$. Its weight is computed according to \cref{eq:csmc:weights}, that is we use the same weight expression for all $\Np$ particles. The empirical distribution approximation of $\nmod_\Tm$ is, analogously to \gls{SMC}, equal to
\begin{align}
		\widehat\nmod_{\Tm}(\xp_{1:\Tm}) = \sum_{\ip=1}^\Np \frac{\uwp_\Tm^\ip}{\sum_\jp \uwp_\Tm^\jp} \dirac_{\xp_{1:\Tm}^\ip}.
		\label{eq:csmc:distapprox}
\end{align}
The updated retained particle is obtained by drawing a single sample from the empirical distribution approximation in \cref{eq:csmc:distapprox}, \ie $\xp_{1:\Tm}'\sim\widehat\nmod_{\Tm}(\xp_{1:\Tm})$.

We summarize the \gls{CSMC} algorithm in \cref{alg:csmc}. Iterated \gls{CSMC} is obtained by letting the retained particle at the next iteration be the output of the previous step. This algorithm constructs a Markov chain, consisting of the retained particles at each iteration, with the target distribution $\nmod_\Tm$ as its stationary distribution. For a thorough treatment of the theory of iterated \gls{CSMC}, \gls{PG} and \gls{PMCMC}, \ie the extension to sampling also model parameters, see \citet{andrieu2010particle}.
\begin{algorithm}[tb]
    \SetKwInOut{Input}{input}\SetKwInOut{Output}{output}
    \Input{Unnormalized target distributions $\umod_\tm$, proposals 
    $\prop_\tm$, retained particle $\xp_{1:\Tm}'$, number of samples $\Np$.}
    \BlankLine
    \For{$\tm=1$ \KwTo $\Tm$}{
        \For{$\ip = 1$ \KwTo $\Np-1$}{
            Sample $\xp_{1:\tm}^\ip \sim 
            \widehat\nmod_{\tm-1}(\xp_{1:\tm-1})
            \prop_\tm(\xp_\tm\given\xp_{1:\tm-1})$ 
            (see \cref{eq:csmc:proposal})\\
            Set $\uwp_\tm^\ip = \frac{\umod_\tm(\xp_{1:\tm}^\ip)}
            {\umod_{\tm-1}(\xp_{1:\tm-1}^\ip) \prop_\tm(\xp_{\tm}^\ip\given \xp_{1:\tm-1}^\ip)}$
            (see \cref{eq:csmc:weights})
        }
        Set $\xp_{1:\tm}^\Np = \xp_{1:\tm}'$\\
        Set $\uwp_\tm^\Np = \frac{\umod_\tm(\xp_{1:\tm}^\Np)}
            {\umod_{\tm-1}(\xp_{1:\tm-1}^\Np) \prop_\tm(\xp_{\tm}^\Np\given \xp_{1:\tm-1}^\Np)}$\\
        Set $\nwp_\tm^\ip = \frac{\uwp_\tm^\ip}{\sum_\jp \uwp_\tm^\jp}$, 
        for $\ip =1,\ldots,\Np$
    }
    Sample $\xp_{1:\Tm}' \sim\widehat\nmod_{\Tm}(\xp_{1:\Tm})$\\
    \caption{\glsreset{CSMC}\Gls{CSMC}}\label{alg:csmc}
\end{algorithm}

The path degeneracy issue discussed in \cref{sec:smc} is particularly prominent in iterated \gls{CSMC}. Because we enforce that the retained particle survives to the very end, we simultaneously force that all the particles collapse for early steps $\tm$ to $\xp_{1:\tm}'$. This means that the output from \cref{alg:csmc}, with high probability, is equal to the input for early values of $\xp_{1:\Tm}'$. The samples we get are highly correlated and any estimator will require a significant number of samples to reasonably approximate expectation with respect to the target distribution $\nmod_\Tm$. By modifying the way we update particle $\Np$, the retained particle, we can alleviate this problem. One way to do this is using \emph{ancestor sampling} \citep{lindsten2014particle}. With ancestor sampling, instead of setting
 $\xp_{1:\tm}^\Np = \xp_{1:\tm}'$ deterministically as in \cref{alg:csmc}, we sample at each time step
\begin{align}
	\xp_{1:\tm}' &\sim \sum_{\ip=1}^\Np \frac{\nuwp_\tm^{\ip}}{\sum_\jp \nuwp_\tm^{\jp}}\dirac_{\xp_{1:\tm-1}^\ip}(\xp_{1:\tm-1}) \cdot \dirac_{\xp_\tm'}(\xp_\tm),
\end{align}
where the auxiliary weights $\nuwp_\tm^{\ip}$ are given by
\begin{align}
	\nuwp_\tm^{\ip} &= \nwp_{\tm-1}^\ip \frac{\umod_{\Tm}\left((\xp_{1:\tm-1}^\ip,\xp_{\tm:\Tm}')\right)}{\umod_{\tm-1}\left(\xp_{1:\tm-1}^\ip\right)}.
	\label{eq:csmc:as}
\end{align}
Using ancestor sampling does not avoid path degeneracy but it can allow for the particle trajectory we degenerate to to be different than the retained particle from a previous iteration. For an in-depth treatment of the path degeneracy issue and its other potential solutions we refer to \citet{lindsten2013backward}.

\begin{example}[\Acrlong{CSMC} for \cref{ex:running}]
We revisit our running non-Markovian Gaussian example. We let $\Tm=20$ and study the performance of iterated \gls{CSMC} without (blue) and with ancestor sampling (red) in \cref{fig:illustration:csmc}. We set $\Np=5$, $10$ and study the kernel density estimates based on $10,000$ \gls{MCMC} samples of the marginal distributions $\nmod_\Tm(\xp_1)$ and $\nmod_\Tm(\xp_\Tm)$, respectively. 

We can see the striking effect that path degeneracy has the iterated \gls{CSMC} approximation in \cref{fig:illustration:csmc:pathdegen}; effectively we approximate $\nmod_\Tm(\xp_1)$ using a single sample when $\Np=5$, whereas the approximation for $\nmod_\Tm(\xp_\Tm)$ is more diverse. When using ancestor sampling on the other hand, the approximation is already accurate for $\Np=5$. This accuracy comes at a cost however, the ancestor sampling step requires that we evaluate \cref{eq:csmc:as} at each iteration. This extra step has a computational cost that in general may be proportional to $\Tm^2$, whereas the cost of standard \gls{CSMC} is often only on the order of $\Tm$. 
For some target distributions, such as the \gls{SSM} or our running example, there is only a small constant computational overhead associated with evaluating \cref{eq:csmc:as}.
However, whether the benefit of ancestor sampling will outweigh its computational cost will have to be evaluated on a case by case basis.
\begin{figure}[tbp]
\centering
\begin{subfigure}[b]{0.42\textwidth}
\resizebox{1.0\textwidth}{!}{
\includegraphics{{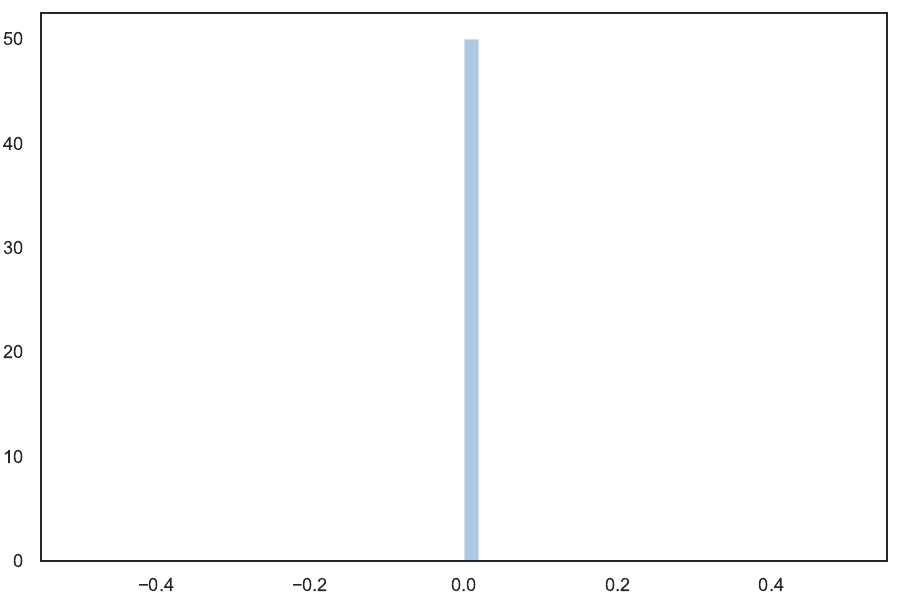}}
}
\caption{$\tm=1$,  $\Np=5$}\label{fig:illustration:csmc:pathdegen}
\end{subfigure}
\begin{subfigure}[b]{0.42\textwidth}
\resizebox{1.0\textwidth}{!}{
\includegraphics{{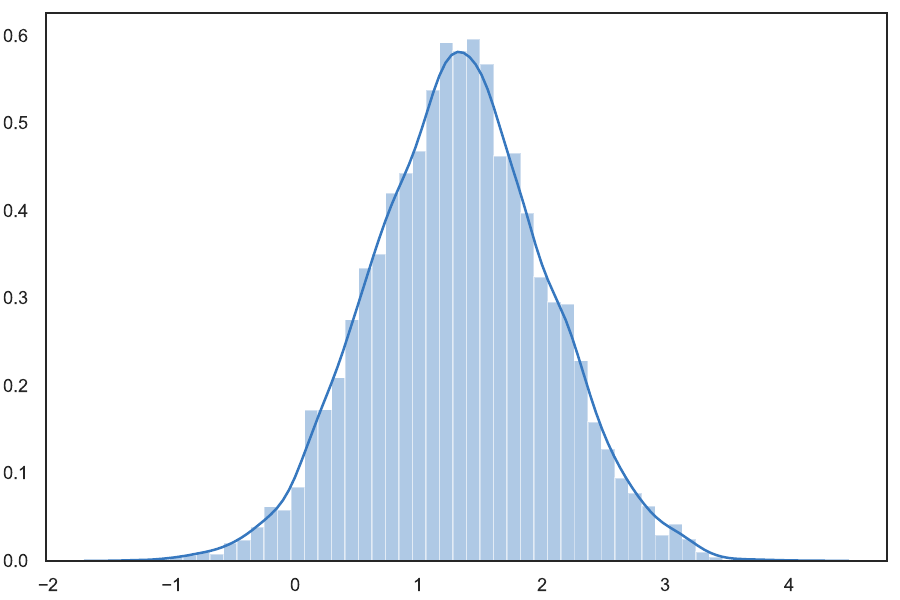}}
}
\caption{$\tm=20$,  $\Np=5$}
\end{subfigure}

\begin{subfigure}[b]{0.42\textwidth}
\resizebox{1.0\textwidth}{!}{
\includegraphics{{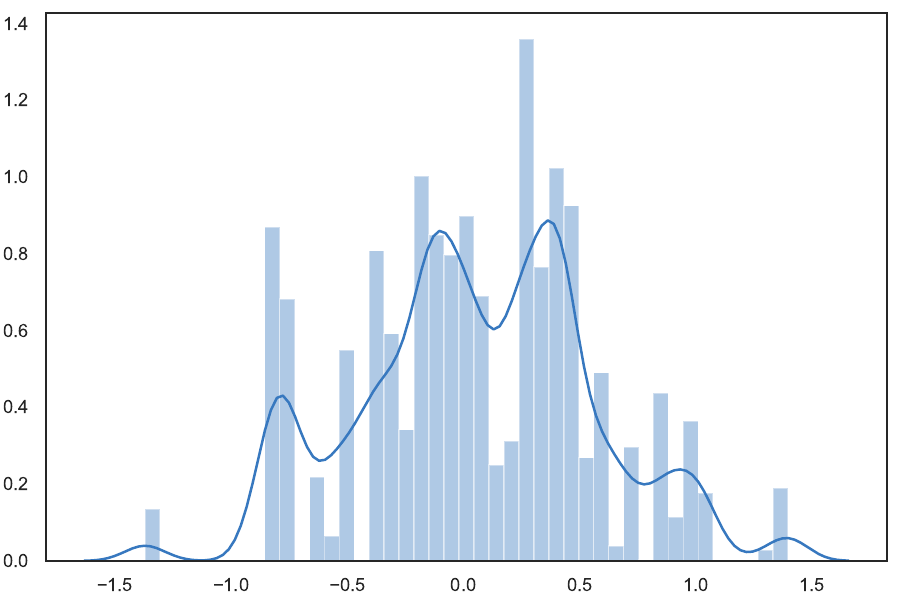}}
}
\caption{$\tm=1$,  $\Np=10$}
\end{subfigure}
\begin{subfigure}[b]{0.42\textwidth}
\resizebox{1.0\textwidth}{!}{
\includegraphics{{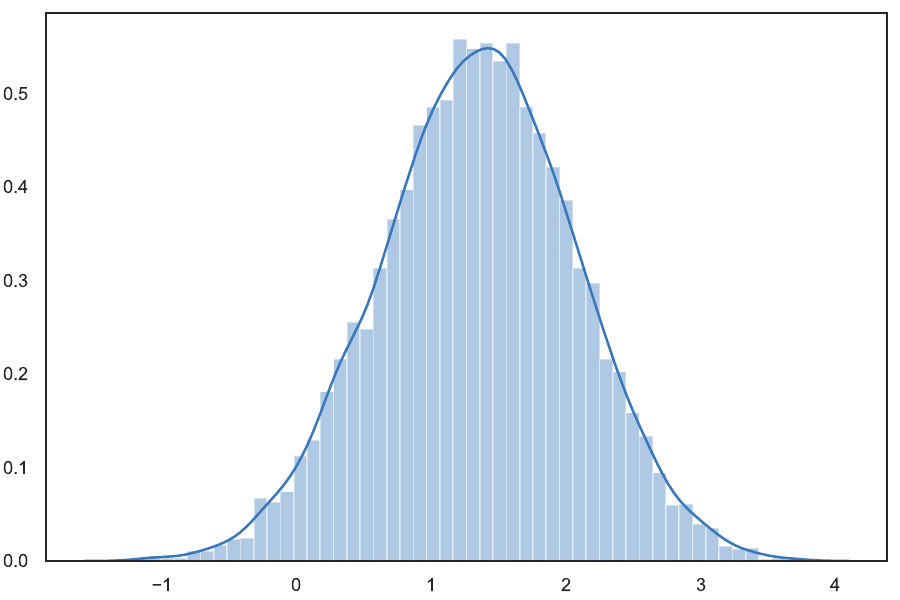}}
}
\caption{$\tm=20$,  $\Np=10$}
\end{subfigure}

\begin{subfigure}[b]{0.42\textwidth}
\resizebox{1.0\textwidth}{!}{
\includegraphics{{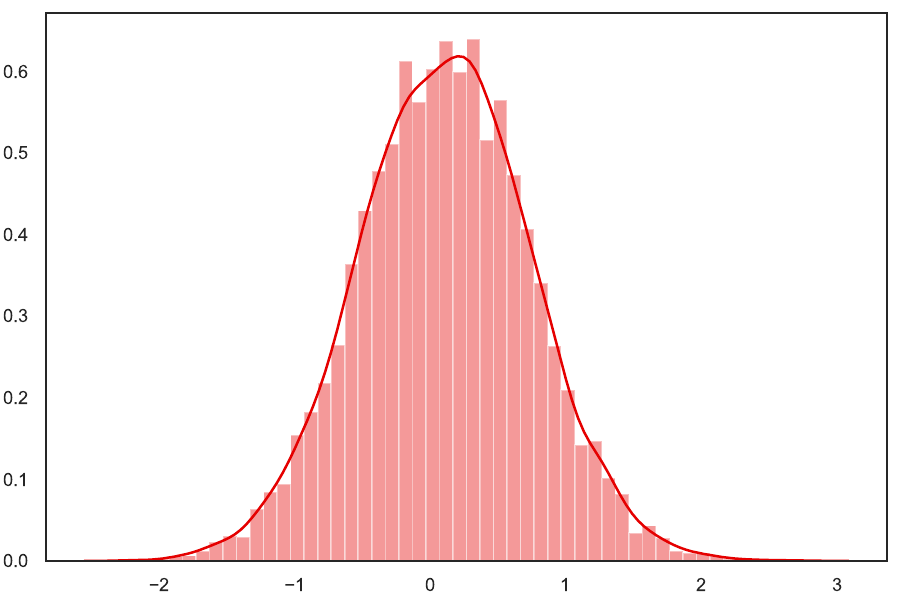}}
}
\caption{$\tm=1$,  $\Np=5$}
\end{subfigure}
\begin{subfigure}[b]{0.42\textwidth}
\resizebox{1.0\textwidth}{!}{
\includegraphics{{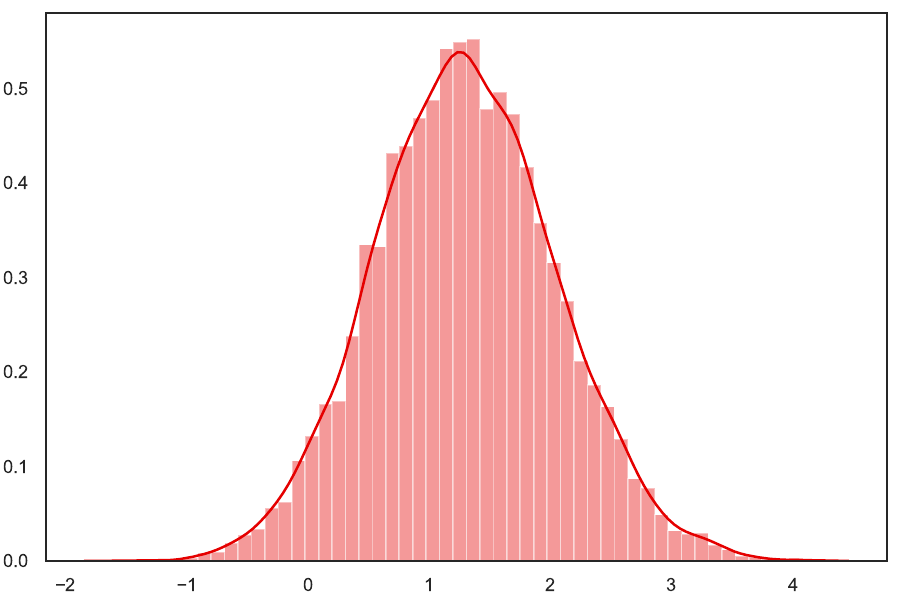}}
}
\caption{$\tm=20$,  $\Np=5$}
\end{subfigure}

\begin{subfigure}[b]{0.42\textwidth}
\resizebox{1.0\textwidth}{!}{
\includegraphics{{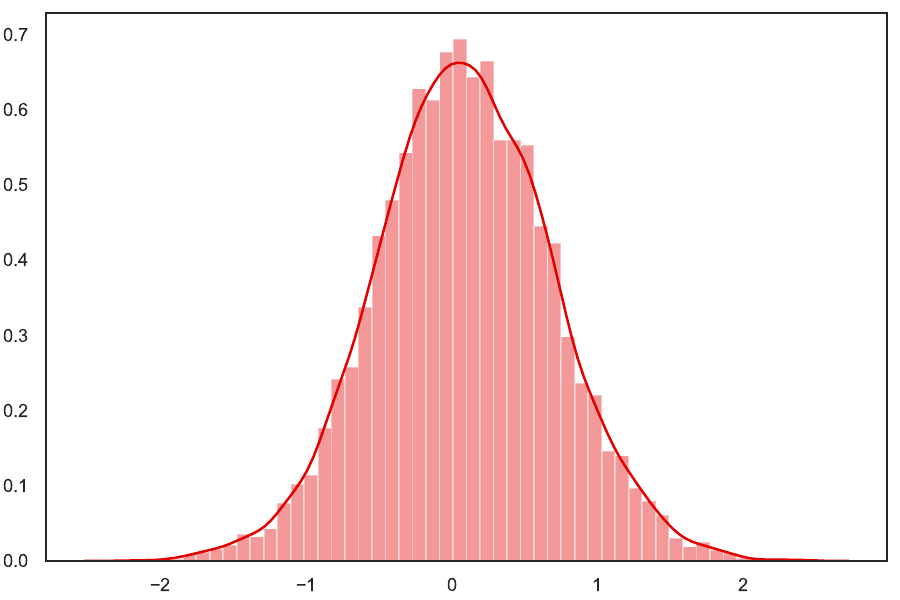}}
}
\caption{$\tm=1$,  $\Np=10$}
\end{subfigure}
\begin{subfigure}[b]{0.42\textwidth}
\resizebox{1.0\textwidth}{!}{
\includegraphics{{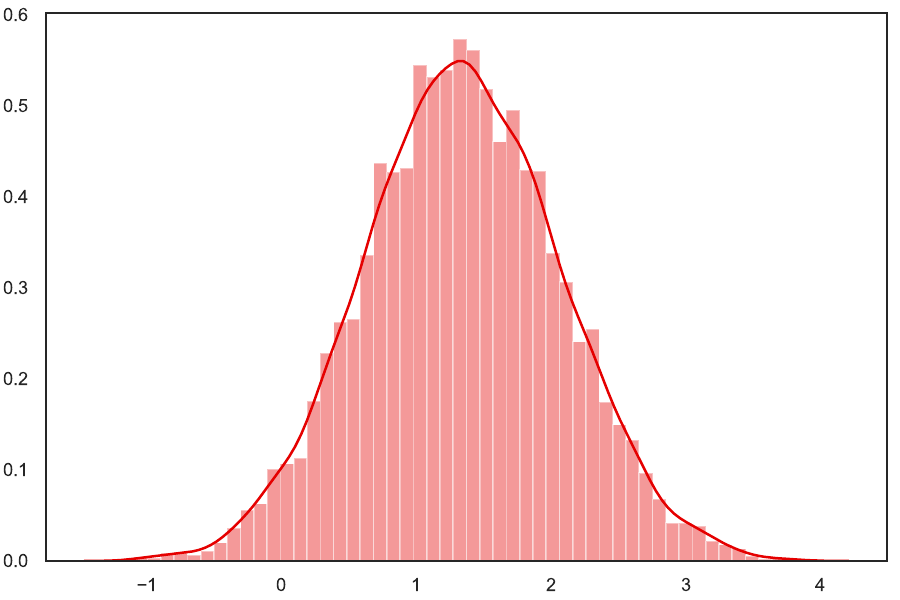}}
}
\caption{$\tm=20$,  $\Np=10$}
\end{subfigure}
\caption{Illustration of the kernel density estimates based on $10,000$ \gls{MCMC} samples of the marginal distributions $\nmod_\Tm(\xp_1)$ and $\nmod_\Tm(\xp_\Tm)$ from iterated \gls{CSMC} (top two rows, blue) and \gls{CSMC} with ancestor sampling (bottow two rows, red).} \label{fig:illustration:csmc}
\end{figure}
\label{ex:csmc}
\end{example}

\section{The Unbiasedness of the Inverse Normalization Constant Estimate}\label{sec:unbiasedZinv}
Unlike in the standard \gls{SMC} algorithm, the normalization constant estimate for \gls{CSMC} based on the product of the average weights is biased. We restate the normalization constant estimate below:
\begin{align}
	\hatZ_\Tm &= \prod_{\tm=1}^\Tm \frac{1}{\Np} \sum_{\ip=1}^\Np \uwp_\tm^\ip.
\end{align}
Because we condition on a reference trajectory, this normalization constant estimate has different properties than does the \gls{SMC} estimate. However, it is still consistent,
converging almost surely to $\Z_\Tm$ as the number of particles tends to infinity.

While the above estimator is biased for a finite number of particles, interestingly the \gls{CSMC} inverse normalization constant estimate $\nicefrac{1}{\hatZ_\Tm}$ \emph{is unbiased}. 
We will see that this is particularly relevant for the inference evaluation methods that we will study in \cref{sec:evaluation}.
The unbiasedness holds under the assumption that the reference trajectory
is distributed according to the target distribution $\xp_{1:\Tm}' \sim \nmod_\Tm(\xp_{1:\Tm})$. 
This assumption relies on the Markov chain having already converged to its stationary distribution. However, in some applications it is possible to make use of \gls{CSMC} with the problem set up in such a way that this is always true.

We formalize and prove the unbiasedness result in \cref{prop:csmc:unbiased}.
\begin{proposition}
Assuming that the retained particle $\xp_{1:\Tm}' \sim \nmod_\Tm(\xp_{1:\Tm})$, then the 
\gls{CSMC} inverse normalization constant estimate $\frac{1}{\hatZ_\Tm}$ is 
non-negative and unbiased
    \begin{align}
        \Exp \left[\frac{1}{\hatZ_\Tm} \right] = \frac{1}{\Z_\Tm}, \quad
        \frac{1}{\hatZ_\tm} &\geq 0.
    \end{align}
    \label{prop:csmc:unbiased}
\end{proposition}
\begin{proof}
See Appendix~\ref{sec:csmc:unbiasedZinv}.
\end{proof}


\section{The Expected SMC Empirical Distribution}\label{sec:csmc:smcdist}
Another method for representing the expected \gls{SMC} empirical distribution $\displaystyle\Exp\left[\widehat{\nmod}_{\Tm}(\xp_{1:\Tm})\right]$ is using \gls{CSMC}. This can be useful for some applications of \gls{SMC}, such as in variational \gls{SMC} or in inference evaluation which we will see later in this chapter. The marginal distribution of a single sample from the \gls{SMC} empirical distribution $\widehat{\nmod}_{\Tm}$ is
\begin{align}
	\Exp\left[\widehat{\nmod}_{\Tm}(\xp_{1:\Tm})\right] &= \Exp_{\qSMC}\left[\sum_{\ip=1}^\Np \nwp_\Tm^\ip \dirac_{\xp_{1:\Tm}^\ip}\left(\xp_{1:\Tm}\right)\right],
\end{align}
where $\qSMC$ (see \cref{eq:smc:smcdist} for its definition) is the distribution of all particles and ancestor variables generated by a single run of the \gls{SMC} algorithm.

We formalize the alternate representation of the expected \gls{SMC} empirical distribution in \cref{prop:csmc:smcdist}. 
\begin{proposition}
The expected \gls{SMC} empirical distribution satisfies
    \begin{align}
        \Exp_{\qSMC}\left[\widehat{\nmod}_{\Tm}(\xp_{1:\Tm})\right] &=\umod_\Tm(\xp_{1:\Tm}) \Exp_{\qCSMC}\left[\hatZ_\Tm^{-1}\given \xp_{1:\Tm}\right],
    \end{align}
    where $\qCSMC$ is the distribution of all random variables generated by a single run of the \gls{CSMC} algorithm with $\xp_{1:\Tm}$ as its retained particle.
    \label{prop:csmc:smcdist}
\end{proposition}
\begin{proof}
See Appendix~\ref{sec:csmc:supp:smcdist}.
\end{proof}

\section{Parallelization}\label{sec:ipmcmc}
In \cref{sec:islandsmc} we briefly discussed a few approaches for leveraging multi-core and distributed computational architectures. Here we will see how we can combine \gls{SMC} and \gls{CSMC} to derive \gls{IPMCMC} \citep{rainforth2016interacting}, an algorithm that takes advantage of parallelization to speed up inference for amongst other things probabilistic programming.

The \gls{IPMCMC} method, like the iterated \gls{CSMC}, can be thought of as a type of Gibbs sampler on an extended space that draws approximate samples from the target distribution $\nmod_\Tm$. However unlike iterated \gls{CSMC}, the \gls{IPMCMC} algorithm generates not only one sample at each iteration but several. The method is based on a pool of several interacting standard and conditional sequential Monte Carlo samplers. These pools of samplers are run as parallel processes and will be referred to as nodes below. Each node either runs an iteration of \gls{CSMC} like in \cref{alg:csmc}, or an iteration of \gls{SMC} like in \cref{alg:esmc:smc}. After each run of this pool, we update the node indices running \gls{CSMC} by applying successive Gibbs updates. This means that the nodes running \gls{CSMC} can change from iteration to iteration. This approach can let the user trade off exploration (\gls{SMC}) and exploitation (\gls{CSMC}) to achieve a faster and more accurate approximate inference method.

To simplify notation we will let $\xm$ denote the full trajectory $\xp_{1:\Tm}$. We assume that we have a total of $M$ nodes, indexed by $\jp=1,\ldots, M$, of which $P$ run \gls{CSMC} and the remaining run standard \gls{SMC}. The conditional nodes, \ie the nodes running \gls{CSMC} at any given time, are denoted by a set of indices $c_{1:P} \in \{1,\ldots,M\}^P$ where each $c_\jp$ is distinct. The conditional node $c_\jp$'s retained particle is referred to by $\xm_{\jp}'$. After each run all nodes produce a normalization constant estimate $\hatZ_\Tm^\jp$, \ie the product of the average unnormalized weights in that node. \Gls{IPMCMC} proceeds by successively updating the conditional node indices and retained particles by Gibbs steps. For $\jp=1,\ldots,P$ we sample
\begin{align}
	\Prb(c_\jp = \ip) &= \frac{\hatZ_\Tm^\ip \cdot \ident\left(\ip \notin c_{\neg \jp}\right)}{\sum_{\km=1}^P \hatZ_\Tm^\km \cdot \ident\left(\km \notin c_{\neg \jp}\right)},
	\label{eq:ipmcmc:nodeidx}
\end{align}
where $\ident\left(\cdot\right)$ is the identity function evaluating to one if its argument is true and zero otherwise. Furthermore, we have used the node index shorthand $c_{\neg j} \eqdef \{c_1,\ldots,c_{\jp-1},c_{\jp+1},\ldots,c_{P}\}$. After selecting new nodes to become conditional nodes through \cref{eq:ipmcmc:nodeidx} we select their new corresponding retained particle by simulating
\begin{align}
	\xm_\jp' &\sim \widehat{\nmod}_\Tm^{c_\jp}(\xm), \quad \jp=1,\ldots,P,
	\label{eq:ipmcmc:nodeparticle}
\end{align}
where $\widehat{\nmod}_\Tm^{c_\jp}$ is the standard target distribution approximation from node $c_\jp$.

We summarize one iteration of \gls{IPMCMC} in \cref{alg:ipmcmc}. Iterating the update steps, letting the retained particles at the previous step be the input for the subsequent step, generates a Markov chain on an extended space where each retained particle is approximately distributed according to the target distribution.
\begin{algorithm}[tb]
    \SetKwInOut{Input}{input}\SetKwInOut{Output}{output}
    \Input{Number of nodes $M$, conditional nodes $P$, retained particles $\xm_{1:P}'$, conditional node indices $c_{1:P}$.}
    \BlankLine
   Nodes $1:M\backslash c_{1:P}$ run \cref{alg:esmc:smc} (\gls{SMC})\\
   Nodes $c_{1:P}$ run \cref{alg:csmc} (\gls{CSMC})\\
        \For{$\jp = 1$ \KwTo $P$}{
            Select a new conditional node by sampling $c_\jp$ 
            (see \cref{eq:ipmcmc:nodeidx})\\
            Set a new retained particle $\xm_\jp'$ by simulating from $\widehat{\nmod}_\Tm^{c_\jp}$
            (see \cref{eq:ipmcmc:nodeparticle})
        }
    \caption{\Acrlong{IPMCMC}}\label{alg:ipmcmc}
\end{algorithm}

\Gls{IPMCMC} is suited to distributed and multi-core architectures. The main computational cost is running the \gls{SMC} and \gls{CSMC} iterations, which can be done in parallel by the nodes. At each iteration, only a single normalization constant estimate needs be communicated between nodes and the calculation of the updates of $c_{1:P}$ is negligible. Retained particles do not need to be communicated and may be stored locally in each node until needed. Furthermore, \gls{IPMCMC} may be amenable to asynchronous adaptation under certain assumptions \citep{rainforth2016interacting}. The algorithm is also suited for use in probabilistic programming where we are often only able to simulate from the prior and evaluate the likelihood. \Acrlong{IPMCMC} can allow for tackling more challenging inference problems, \ie more complex probabilistic programs, compared to standard \gls{PMCMC} methods due to leveraging distributed and multi-core computational architectures \citep{rainforth2016interacting}.


\section{Evaluating Inference}\label{sec:evaluation}
Approximate Bayesian inference is central to a large part of probabilistic machine learning and its application. It is a coherent approach to treat uncertainty in both the data and the model with applications to everything from robotics to health care. A notoriously difficult problem is to measure the accuracy of our posterior approximation compared to the true posterior distribution.  Inference checking, or inference evaluation, refers to the process of comparing our approximate inference result to what we would have obtained using the true posterior. In practice whether an approximation is accurate or not may differ between different applications, we might be interested in different aspects of the posterior. Some examples are tail probabilities in rare event estimation or posterior mean for minimum mean square error estimation. In this section we will focus on comparing directly the posterior approximation to the exact posterior distribution in terms of symmetrized \gls{KL} divergence, also known as the Jeffreys divergence, following \citet{grosse2016reliability,towner-nips-2017}.

The posterior distribution that we want to approximate is denoted by $p(\xm\given\ym)$, and the approximate posterior will be denoted by $q(\xm)$. The approximation $q$ could for example be the expected \gls{SMC} empirical distribution approximation, similar to \gls{VSMC}, or it could be an optimized variational approximation, or any of the other examples detailed in \citet{grosse2016reliability,towner-nips-2017}. 
In either case, the symmetrized \gls{KL} divergence from $q$ to $p$ is given by
\begin{align}
	\symKL\left(q(\xm)\,\|\,p(\xm\given\ym)\right) &= 	\Kl\left(q(\xm)\,\|\, p(\xm\given\ym)\right)+\Kl\left(p(\xm\given\ym) \,\|\, q(\xm)\right),
	\label{eq:csmc:jeffreysdiv}
\end{align}
where the standard \acrlong{KL} divergence from $q(\xm)$ to $p(\xm)$ is defined by $\Kl\left(q(\xm)\,\|\, p(\xm)\right) \eqdef \Exp_{q(\xm)}\left[\log q(\xm)-\log p(\xm)\right]$. We can rewrite \cref{eq:csmc:jeffreysdiv} to eliminate the need for computing the marginal likelihood $p(\ym)$ as follows
\begin{align}
	&\symKL\left(q(\xm)\,\|\,p(\xm\given\ym)\right) = 	\Exp_{q(\xm)}\left[\log \frac{q(\xm)}{p(\xm\given\ym)}\right] + \Exp_{p(\xm\given \ym)}\left[\log \frac{p(\xm\given\ym)}{q(\xm)}\right] \nonumber \\
	&=	\Exp_{q(\xm)}\left[\log q(\xm) -\log p(\xm,\ym)\right] + \Exp_{p(\xm\given \ym)}\left[\log p(\xm,\ym) -\log q(\xm)\right].
	\label{eq:csmc:alt:jeffreysdiv}
\end{align}
To be able to estimate \cref{eq:csmc:alt:jeffreysdiv} using a standard \gls{MC} estimator we require that we can sample $q(\xm)$ and $p(\xm\given\ym)$, respectively. Furthermore, we require that we are able to evaluate the \glspl{PDF} $q(\xm)$ and $p(\xm,\ym)$ pointwise. These requirements can be restrictive for the types of problems we encounter in practice, and \citet{grosse2016reliability,towner-nips-2017} make various relaxations and assumptions to make the problem tractable. Below we will detail the two approaches to inference checking known as \gls{BREAD} \citep{grosse2016reliability} and \gls{AIDE} \citep{towner-nips-2017}.

\paragraph{Bounding Divergences with Reverse Annealing} The \gls{BREAD} approach focuses on bounding the symmetrized \gls{KL} for a posterior approximation based on \gls{AIS} \citep{neal2001annealed} for \emph{simulated data}. By using simulated data we ensure that we can generate an exact sample from the posterior distribution $\xm\sim p(\xm\given\ym)$. We can think of \gls{AIS} as a single sample, $\Np=1$, \gls{SMC} method that uses as its proposal a Markov kernel with stationary distribution $\nmod_\tm(\xp_\tm)$. The target distribution is typically chosen according to a data or likelihood tempering scheme as discussed in \cref{sec:seqmodels}. Per definition the posterior distribution is obtained as the marginal at the final iteration, \ie $p(\xm\given\ym) = \nmod_\Tm(\xp_	\Tm)$ where $\xm \equiv \xp_\Tm$. The posterior distribution approximation is then defined by $q(\xm) \eqdef \Exp[\widehat{\nmod}_\Tm(\xp_\Tm)]$ and the expectation is with respect to all the proposal distributions.

Because of this set up the intractability in estimating, or bounding, the Jeffreys divergence in \cref{eq:csmc:alt:jeffreysdiv} comes from having to evaluate the \gls{PDF} $q(\xm)$ pointwise. We can bound \cref{eq:csmc:alt:jeffreysdiv} from above by noting that marginalizing out variables cannot increase the \gls{KL} divergence 
\begin{align}
	\symKL\left(q(\xm)\,\|\,p(\xm\given\ym)\right) &\leq \symKL\left(\Exp[\widehat{\nmod}_\Tm(\xp_{1:\Tm})] \,\|\, \nmod_\Tm(\xp_{1:\Tm})\right) \nonumber \\
	&=  \symKL\left(\umod_\Tm(\xp_{1:\Tm}) \Exp_{\qCSMC}\left[\hatZ_\Tm^{-1}\given \xp_{1:\Tm}\right] \,\|\, \nmod_\Tm(\xp_{1:\Tm})\right),
	\label{eq:csmc:skl:bound1}
\end{align}
where the equality follows from \cref{prop:csmc:smcdist}. We can further bound \cref{eq:csmc:skl:bound1} using Jensen's inequality and simplifying, leading to
\begin{align}
	& \symKL\left(\umod_\Tm(\xp_{1:\Tm}) \Exp_{\qCSMC}\left[\hatZ_\Tm^{-1}\given \xp_{1:\Tm}\right] \,\|\, \nmod_\Tm(\xp_{1:\Tm})\right) \nonumber \\
	& \leq \Exp_{\qSMC}\left[-\log \hatZ_\Tm \right] + \Exp_{\nmod_\Tm}\left[ \Exp_{\qCSMC}\left[\log \hatZ_\Tm \given \xp_{1:\Tm}\right] \right]. 
	\label{eq:csmc:skl:bound2}
\end{align}
The inequality in \cref{eq:csmc:skl:bound2} holds generally even if we would use a standard \gls{SMC} algorithm rather than \gls{AIS} as our posterior approximation. The details are provided in \cref{sec:csmc:supp:bound2}. However, when \gls{AIS} is used we can further simplify to obtain the same bound as \citet{grosse2016reliability}
\begin{align}
	 &\Exp_{\qSMC}\left[-\log \hatZ_\Tm \right] + \Exp_{\nmod_\Tm}\left[ \Exp_{\qCSMC}\left[\log \hatZ_\Tm \given \xp_{1:\Tm}\right] \right] \nonumber \\
	 &= \Exp_{\prop(\xp_{1:\Tm})}\left[ - \sum_{\tm=1}^\Tm \log \uwp_\tm \right] + \Exp_{p(\xp_\Tm\given\ym)\, \smod(\xp_{1:\Tm-1}\given\xp_\Tm)}\left[\sum_{\tm=1}^\Tm \log \uwp_\tm \right],
	 \label{eq:csmc:skl:bound3}
\end{align}
where $\prop(\xp_{1:\Tm})= \prop_1(\xp_1)\prod_{\tm=1}^\Tm \prop_\tm(\xp_\tm\given\xp_{\tm-1})$ are the proposals (forward kernels), $\smod(\xp_{1:\Tm-1}\given\xp_\Tm) = \prod_{\tm=1}^{\Tm-1} \smod_\tm(\xp_\tm\given\xp_{\tm+1})$ are the corresponding backward kernels (\cf \cref{sec:seqmodels}, Conditionally independent models), and the weights are in the likelihood tempering variant given by
\begin{align*}
	\uwp_\tm &= p(\ym\given\xp_{\tm-1})^{\tau_\tm-\tau_{\tm-1}},
\end{align*}
with $\uwp_1 = 1$. This means it is possible to obtain a stochastic upper bound, \ie a \gls{MC} estimate of the upper bound in \cref{eq:csmc:skl:bound3}, by running a standard (forward) \gls{AIS} evaluating the negative log-weights, and running a \emph{reverse} \gls{AIS} evaluating the log-weights; the left and right expressions in \cref{eq:csmc:skl:bound3}, respectively.

\paragraph{Auxiliary Inference Divergence Estimator} The \gls{AIDE} approach, on the other hand, treats a more general class of posterior approximations such as variational inference, \gls{SMC}, or \gls{MCMC}. However, rather than assuming simulated data, \citet{towner-nips-2017} propose to study the Jeffreys divergence from the approximation to that of a gold standard inference approximation.  Typically this gold standard approximation would consist of a long run of an \gls{MCMC} method or \gls{SMC} with a large number of particles that we know approximates the posterior well.  

In this setting not only is the \gls{PDF} for the approximation $q(\xm)$ difficult to evaluate, so is the \gls{PDF} for our gold standard algorithm. Rather than introducing new notation for the gold standard algorithm we will keep the posterior distribution notation $p(\xm\given\ym)$. The approach that the \gls{AIDE} takes is to use \emph{meta-inference} methods, based on \gls{SMC} or \gls{CSMC}, to obtain approximations of the \glspl{PDF} and their reciprocals. A meta-inference algorithm takes as input a sample from the distribution of interest and outputs an unbiased estimate of either its \gls{PDF} or the reciprocal. In fact it suffices that the meta-inference algorithms provide estimates that are unbiased up to a constant of proportionality because any constant is canceled by the symmetry in the Jeffrey's divergence.

We note that the approach suggested by \citet{towner-nips-2017} is essentially equivalent to the bound derived in \cref{eq:csmc:skl:bound1,eq:csmc:skl:bound2,eq:csmc:skl:bound3}. However, instead of the exact target distribution $\nmod_\Tm(\xp_{1:\Tm})$ we have another \gls{SMC} approximation to represent the gold standard algorithm. To differentiate between the two \gls{SMC} approximations we will use superscript $q$ for the approximation and $p$ for the gold standard. With this the corresponding Jeffrey's divergence is
\begin{align}
	&\symKL\left(q(\xm)\,\|\,p(\xm\given\ym)\right) = \symKL\left(\Exp[\widehat{\nmod}_\Tm^q(\xp_{1:\Tm})] \,\|\, \Exp[\widehat{\nmod}_\Tm^p(\xp_{1:\Tm})]  \right) \nonumber \\
	&=  \symKL\left(\umod_\Tm(\xp_{1:\Tm}) \Exp_{\qCSMC^q}\left[\nicefrac{1}{\hatZ_\Tm^q}\given \xp_{1:\Tm}\right] \,\|\, \umod_\Tm(\xp_{1:\Tm}) \Exp_{\qCSMC^p}\left[ \nicefrac{1}{\hatZ_\Tm^p} \given \xp_{1:\Tm}\right]\right),
	\label{eq:csmc:skl:general}
\end{align}
where $\qCSMC^q$, $\hatZ_\Tm^q$ and $\qCSMC^p$, $\hatZ_\Tm^p$ are the \gls{CSMC} distributions and normalization constant estimates for the approximation and gold standard algorithm, respectively. The first equality follows from the definition of $q(\xm)$ and $p(\xm\given\ym)$. The second equality follows from \cref{prop:csmc:smcdist} for the two distributions, respectively. The divergence in \cref{eq:csmc:skl:general} can be bounded, analogously to the \gls{BREAD} case, as follows
\begin{align}
	&\symKL\left(\umod_\Tm(\xp_{1:\Tm}) \Exp_{\qCSMC^q}\left[\nicefrac{1}{\hatZ_\Tm^q}\given \xp_{1:\Tm}\right] \,\|\, \umod_\Tm(\xp_{1:\Tm}) \Exp_{\qCSMC^p}\left[ \nicefrac{1}{\hatZ_\Tm^p} \given \xp_{1:\Tm}\right]\right) \nonumber \\
	&\leq \Exp_{q(\xm)}\left[\Exp_{\qCSMC^p}\left[ \log \hatZ_\Tm^p \given \xp_{1:\Tm}\right]\right]- \Exp_{\qSMC^q}\left[\log\hatZ_\Tm^q \right] + \nonumber \\
	&+\Exp_{p(\xm\given\ym)}\left[\Exp_{\qCSMC^q}\left[ \log \hatZ_\Tm^q \given \xp_{1:\Tm}\right]\right]- \Exp_{\qSMC^p}\left[\log\hatZ_\Tm^p \right],
	\label{eq:csmc:skl:aide}
\end{align}
where $q(\xm)$ and $p(\xm\given\ym)$ are the two corresponding expected \gls{SMC} empirical distributions.

The bound above differs slightly from that derived by \citet{towner-nips-2017}. The \gls{AIDE} procedure uses a form of importance weighted version of \cref{eq:csmc:skl:aide}, in their notation it is equivalent to the bound based on a single run of the meta-inference for the approximation ($M_t=1$) and gold standard ($M_g=1$). For more details we refer the interested reader to that paper.

\newpage
\begin{subappendices}
\section{Proof of Unbiasedness}\label{sec:csmc:unbiasedZinv}
When introducing the \gls{CSMC} method we consistently set the $\Np$:th particle to the reference trajectory. In practice it does not matter which particle is deterministically set to the reference trajectory. It will prove simplifying to let the particle that is conditioned on, \ie the reference trajectory, have its own index $\bp_{1:\Tm} = (\bp_1,\ldots,\bp_\Tm) \in \{1,\ldots,\Np\}^\Tm$. For clarity we focus on using multinomial resampling at each iteration of the algorithm. The distribution of all random variables generated by the \gls{CSMC} method in \cref{alg:csmc} is then given by
\begin{align}
    &\qCSMC\left(\xp_{1:\Tm}^{\neg \bp_{1:\Tm}}, \ap_{1:\Tm-1}^{\neg \bp_{1:\Tm-1}}\given \xp_{1:\Tm}^{\bp_{1:\Tm}}, \bp_{1:\Tm}\right) \nonumber\\
    & = \prod_{\ip \neq b_1} \prop_1(\xp_1^\ip) \cdot \prod_{\tm=2}^\Tm 
    \prod_{\ip\neq \bp_\tm} \left[\nwp_{\tm-1}^{\ap_{\tm-1}^\ip} 
    \prop_\tm\left(\xp_\tm^\ip \given 
    \xp_{1:\tm-1}^{\ap_{\tm-1}^\ip}\right)\right],
    \label{eq:smc:csmcdist}
\end{align}
where the particles are $\xp_{1:\Tm}^{\neg \bp_{1:\Tm}} = \bigcup_{\tm=1}^\Tm 
\left\{\xp_\tm^\ip\right\}_{\ip=1}^\Np \backslash\{\xp_{\tm}^{\bp_\tm}\}_{\tm=1}^\Tm$, and the ancestor variables from 
the resampling step are $\ap_{1:\Tm-1}^{\neg \bp_{1:\Tm-1}} = \bigcup_{\tm=1}^{\Tm-1} 
\left\{\ap_\tm^\ip\right\}_{\ip=1}^\Np \backslash \{\bp_\tm\}_{\tm=1}^{\Tm-1}$. Assuming that $\uwp_\tm^\ip 
\geq 0$ for all $\ip$ and $\tm$, we have that $\hatZ_\Tm\geq 0$ and thus $\nicefrac{1}{\hatZ_\Tm} \geq 0$. We are 
now going to prove that when $\left(\xp_{1:\Tm}^{\bp_{1:\Tm}},\bp_{1:\Tm}\right) \sim \nmod_\Tm(\xp_{1:\Tm}) \cdot \Np^{-\Tm}$
\begin{align}
   \Exp_{\bp_{1:\Tm}}\left[\Exp_{\nmod_\Tm\cdot\qCSMC}\left[\frac{\Z_\Tm}{\widehat\Z_\Tm}\right] \right]= 1,
    \label{eq:smc:invhatZexp}
\end{align}
from which the result follows. We rewrite the integrand in \cref{eq:smc:invhatZexp},
\begin{align*}
	&\frac{\Z_\Tm}{\widehat\Z_\Tm} \frac{\nmod_\Tm(\xp_{1:\Tm}^{\bp_{1:\Tm}}) }{\Np^\Tm}\qCSMC\left(\xp_{1:\Tm}^{\neg \bp_{1:\Tm}}, \ap_{1:\Tm-1}^{\neg \bp_{1:\Tm-1}}\given \xp_{1:\Tm}^{\bp_{1:\Tm}}, \bp_{1:\Tm}\right) \\
	&= \frac{1}{\widehat\Z_\Tm} \frac{\umod_\Tm(\xp_{1:\Tm}^{\bp_{1:\Tm}}) }{\Np^\Tm} \frac{\qSMC\left(\xp_{1:\Tm}^{1:\Np}, \ap_{1:\Tm-1}^{1:\Np}\right)}{ \prop_1(\xp_1^{\bp_1}) \prod_{\tm=2}^\Tm \nwp_{\tm-1}^{\bp_{\tm-1}} 
    \prop_\tm\left(\xp_\tm^{\bp_\tm} \given  \xp_{1:\tm-1}^{\bp_{\tm-1}}\right)} \\
    &=\frac{\umod_\Tm(\xp_{1:\Tm}^{\bp_{1:\Tm}})}{\widehat\Z_\Tm} \cdot
    \frac{\prod_{\tm=1}^\Tm \frac{1}{\Np}\sum_\ip \uwp_\tm^\ip}{\prop_1(\xp_1^{\bp_1}) \prod_{\tm=2}^\Tm \uwp_{\tm-1}^{\bp_{\tm-1}}  \prop_\tm\left(\xp_\tm^{\bp_\tm} \given  \xp_{1:\tm-1}^{\bp_{\tm-1}}\right)}
   \cdot \frac{\qSMC\left(\xp_{1:\Tm}^{1:\Np}, \ap_{1:\Tm-1}^{1:\Np}\right)}{\sum_\ip\uwp_\Tm^\ip} \\
   &=	\frac{\uwp_\Tm^{\bp_\Tm}}{\sum_\ip\uwp_\Tm^\ip} \qSMC\left(\xp_{1:\Tm}^{1:\Np}, \ap_{1:\Tm-1}^{1:\Np}\right),
\end{align*}
where the first line is the integrand. In the first equality we have divided and multiplied by $\prop_1 \prod_{\tm=2}^\Tm \nwp_{\tm-1}^{\bp_{\tm-1}} \prop_\tm$ to be able to identify the difference between $\qSMC$ and $\qCSMC$. Note that we have replaced $\bp_\tm=\ap_\tm^{\bp_{\tm+1}}$ in the notation for $\qSMC\left(\xp_{1:\Tm}^{1:\Np}, \ap_{1:\Tm-1}^{1:\Np}\right)$ to identify the \gls{SMC} distribution. For the second equality we have rewritten the normalized weights in the denominator, together with $\Np^{-\Tm}$ and dividing and multiplying by $\sum_\ip\uwp_\Tm^\ip$, to identify the normalization constant estimator. For the final equality we see that the normalization constant estimators cancel, and the product in the denominator simplifies together with $\umod_\Tm$ to become $\uwp_\Tm^{\bp_\Tm}$. 

Using this simplified expression for the integrand in \cref{eq:smc:invhatZexp} we obtain
\begin{align*}
   &\sum_{\bp_\Tm}\sum_{\ap_{1:\Tm-1}^{1:\Np}} \int 	\frac{\uwp_\Tm^{\bp_\Tm}}{\sum_\ip\uwp_\Tm^\ip} \qSMC\left(\xp_{1:\Tm}^{1:\Np}, \ap_{1:\Tm-1}^{1:\Np}\right) \dif \xp_{1:\Tm}^{1:\Np} \\
   &=\sum_{\ap_{1:\Tm-1}^{1:\Np}} \int \qSMC\left(\xp_{1:\Tm}^{1:\Np}, \ap_{1:\Tm-1}^{1:\Np}\right) \dif \xp_{1:\Tm}^{1:\Np} = 1,
\end{align*}
where again we first replace the original sum over $\bp_{1:\Tm-1}$ with the ancestor variable $\ap$ notation. The first equality marginalizes $\bp_\Tm$ and the final equality follows because $\qSMC$ is a distribution on $\xp_{1:\Tm}^{1:\Np}, \ap_{1:\Tm-1}^{1:\Np}$. This concludes the proof.

\section{Proof of SMC Distribution Representation}\label{sec:csmc:supp:smcdist}
We will prove the following alternate representation of the expected \gls{SMC} empirical distribution
\begin{align}
        \Exp_{\qSMC}\left[\widehat{\nmod}_{\Tm}(\xp_{1:\Tm})\right] &= \umod_\Tm(\xp_{1:\Tm}) \Exp_{\qCSMC}\left[\hatZ_\Tm^{-1}\given \xp_{1:\Tm}\right],
        \label{eq:csmc:supp:smcdist}
 \end{align}
were we for clarity again focus on using multinomial resampling at each iteration of the algorithm. Note that the left hand side is a probability measure and the right hand side a \acrlong{PDF}. To be precise the equality means that the corresponding probability measure of the right hand side (usually adding either the Lebesgue our counting measure) is equal almost everywhere to the left hand side.
 
To prove the result we need to show that the expectation of any bounded and continuous function $\fun(\xp_{1:\Tm})$ coincides for the left and right hand sides of \cref{eq:csmc:supp:smcdist}. We have that the expectation of $\fun(\xp_{1:\Tm})$ with respect to the right hand side of \cref{eq:csmc:supp:smcdist} is
 \begin{align*}
 	&\int \fun(\xp_{1:\Tm})\umod_\Tm(\xp_{1:\Tm}) \Exp_{\qCSMC}\left[\hatZ_\Tm^{-1}\given \xp_{1:\Tm}\right] \dif \xp_{1:\Tm} \\
 	&=\sum_{\bp_{1:\Tm}} \int \frac{1}{\Np^\Tm} \fun(\xp_{1:\Tm}^{\bp_{1:\Tm}})\umod_\Tm(\xp_{1:\Tm}^{\bp_{1:\Tm}}) \Exp_{\qCSMC}\left[\hatZ_\Tm^{-1}\given \xp_{1:\Tm}^{\bp_{1:\Tm}}\right] \dif \xp_{1:\Tm}^{\bp_{1:\Tm}} \\
 	&=\sum_{\bp_{1:\Tm}} \sum_{\ap_{1:\Tm-1}^{\neg \bp_{1:\Tm-1}}} \int \Bigg[ \frac{ \fun(\xp_{1:\Tm}^{\bp_{1:\Tm}})\umod_\Tm(\xp_{1:\Tm}^{\bp_{1:\Tm}})}{\Np^\Tm \prod_{\tm=1}^\Tm \frac{1}{\Np} \sum_{\ip=1}^\Np \uwp_\tm^\ip} \cdot  \\
 	& \quad\quad \cdot \qCSMC\left(\xp_{1:\Tm}^{\neg \bp_{1:\Tm}}, \ap_{1:\Tm-1}^{\neg \bp_{1:\Tm-1}}\given \xp_{1:\Tm}^{\bp_{1:\Tm}}, \bp_{1:\Tm}\right) \Bigg]  \dif \xp_{1:\Tm}^{\bp_{1:\Tm}}  \dif \xp_{1:\Tm}^{\neg \bp_{1:\Tm}} \\
 	&= \sum_{\bp_\Tm}\sum_{\ap_{1:\Tm-1}^{1:\Np}} \int \fun(\xp_{1:\Tm}^{\bp_\Tm}) \frac{\uwp_\Tm^{\bp_\Tm}}{\sum_{\ip=1}^\Np \uwp_\Tm^\ip} \qSMC\left(\xp_{1:\Tm}^{1:\Np}, \ap_{1:\Tm-1}^{1:\Np}\right) \dif \xp_{1:\Tm}^{1:\Np}  \\
 	&=   \int \fun(\xp_{1:\Tm}) \cdot \\
 	&\quad\quad  \cdot \left[ \int \sum_{\ap_{1:\Tm-1}^{1:\Np}} \sum_{\bp_\Tm} \frac{\uwp_\Tm^{\bp_\Tm}}{\sum_{\ip=1}^\Np \uwp_\Tm^\ip} \dirac_{\xp_{1:\Tm}^{\bp_\Tm}}(\xp_{1:\Tm}) \qSMC\left(\xp_{1:\Tm}^{1:\Np}, \ap_{1:\Tm-1}^{1:\Np}\right) \dif \xp_{1:\Tm}^{1:\Np}  \right] \dif \xp_{1:\Tm} \\
 	&= \int \fun(\xp_{1:\Tm})  \Exp_{\qSMC}\left[\widehat{\nmod}_{\Tm}(\xp_{1:\Tm})\right] \dif \xp_{1:\Tm},
 \end{align*}
where the first equality is the definition. In the second equality we have expanded the terms and can identify the expression derived in \cref{sec:csmc:unbiasedZinv} above. In the third equality we have replaced it with the corresponding $\qSMC$ distribution, again using $\bp_\tm=\ap_\tm^{\bp_{\tm+1}}$ in the notation for $\qSMC$. In the penultimate equality we re-arrange a few terms and replace the sum over $\bp_\Tm$ with an integration with respect to a point mass distribution. In the final equality we identify the expectation with respect to $\Exp_{\qSMC}\left[\widehat{\nmod}_{\Tm}(\xp_{1:\Tm})\right]$. This concludes the proof.

\section{Jeffreys Divergence Bound}\label{sec:csmc:supp:bound2}
We detail the derivation of the bound in \cref{eq:csmc:skl:bound2} below under the assumption of multinomial resampling at each iteration of the algorithm,
\begin{align*}
	& \symKL\left(\umod_\Tm(\xp_{1:\Tm}) \Exp_{\qCSMC}\left[\hatZ_\Tm^{-1}\given \xp_{1:\Tm}\right] \,\|\, \nmod_\Tm(\xp_{1:\Tm})\right) \\
	&=\int \umod_\Tm(\xp_{1:\Tm}) \Exp_{\qCSMC}\left[\hatZ_\Tm^{-1}\given \xp_{1:\Tm}\right] \log \frac{\umod_\Tm(\xp_{1:\Tm}) \Exp_{\qCSMC}\left[\hatZ_\Tm^{-1}\given \xp_{1:\Tm}\right]}{ \nmod_\Tm(\xp_{1:\Tm})} \dif \xp_{1:\Tm} \\
	&\quad+ \int \nmod_\Tm(\xp_{1:\Tm}) \log \frac{ \nmod_\Tm(\xp_{1:\Tm})} {\umod_\Tm(\xp_{1:\Tm}) \Exp_{\qCSMC}\left[\hatZ_\Tm^{-1}\given \xp_{1:\Tm}\right]} \dif \xp_{1:\Tm} \\
	&=\int \umod_\Tm(\xp_{1:\Tm}) \Exp_{\qCSMC}\left[\hatZ_\Tm^{-1}\given \xp_{1:\Tm}\right] \log  \Exp_{\qCSMC}\left[\hatZ_\Tm^{-1}\given \xp_{1:\Tm}\right] \dif \xp_{1:\Tm} \\
	&\quad- \int \nmod_\Tm(\xp_{1:\Tm}) \log \Exp_{\qCSMC}\left[\hatZ_\Tm^{-1}\given \xp_{1:\Tm}\right] \dif \xp_{1:\Tm}  \\
	&\leq\int \umod_\Tm(\xp_{1:\Tm}) \Exp_{\qCSMC}\left[\hatZ_\Tm^{-1} \log  \hatZ_\Tm^{-1} \given \xp_{1:\Tm}\right] \dif \xp_{1:\Tm} \\
	&\quad- \int \nmod_\Tm(\xp_{1:\Tm})  \Exp_{\qCSMC}\left[ \log\hatZ_\Tm^{-1}\given \xp_{1:\Tm}\right] \dif \xp_{1:\Tm}  \\
	& = \Exp_{\qSMC}\left[-\log \hatZ_\Tm \right] + \Exp_{\nmod_\Tm}\left[ \Exp_{\qCSMC}\left[\log \hatZ_\Tm \given \xp_{1:\Tm}\right] \right],
\end{align*}
where the first equality is the definition. The second equality is a simplification of the previous expression where the terms cancelling, $\umod_\Tm, \nmod_\Tm, \Z_\Tm$, have been removed. Next, the inequality follows from the (conditional) Jensen's inequality for the convex functions $x \log x$ and $-\log x$, respectively. The final equality follows from \cref{prop:csmc:smcdist}, noting that the first term is equivalent to an expectation with respect to the full \gls{SMC} distribution $\qSMC$.

The Jeffrey's divergence bound between two generic expected \gls{SMC} empirical distributions in \cref{eq:csmc:skl:general} is a straightforward extension to the above derivation. 
\end{subappendices}

\chapter{Discussion}\label{sec:discussion}
We have described \acrlong{SMC}, methods that use random samples to 
approximate the posterior. The goal is to approximate the 
distribution of the latent (unobserved) variables given the data. The 
idea is to draw samples from a (simpler) proposal distribution, then
correct for the discrepancy between the proposal and the posterior using 
weights. \gls{SMC} is generally applicable and can be considered to 
be an alternative, or complement, to \gls{MCMC} methods.

First, we discussed the roots of \gls{SMC} in (sequential) importance sampling and 
sampling/importance resampling. Next, we introduced the \gls{SMC} 
algorithm in its full generality. The algorithm was illustrated with 
synthetic data from a non-Markovian latent variable model. Then, we 
discussed practical issues and theoretical results. The \gls{SMC} 
estimate satisfies favorable theoretical properties; under appropriate 
assumptions it is both consistent and has a central limit 
theorem. 
Then, we discussed and explained how to learn a good 
proposal and target distribution. As a motivating example, we described 
how to go about learning a proposal- and target distribution for a 
stochastic recurrent neural network. 
Next, we discussed \emph{proper weighting} and how to use \gls{SMC} to construct nested estimators based on several layers of \gls{MC} approximations to approximate idealized \gls{MC} algorithms. We provided a straightforward proof to the unbiasedness property of the \gls{SMC} normalization constant estimate. We studied \gls{PMH} algorithms, methods that construct high-dimensional proposals for \gls{MH} based on \gls{SMC} samplers. We studied random and nested \gls{SMC} methods that deal with intractability in the weights and proposal distributions with applications to stochastic differential equations and optimal proposal approximation. We illustrated how to leverage proper weighting and \gls{SMC} to design algorithms that makes use of distributed and multi-core computing architectures. 
Finally, we introduced \acrlong{CSMC} algorithms and applications of these. We provided proof that the \gls{CSMC} sampler obtains unbiased estimates of the \emph{inverse} normalization constant. We showed how the \gls{IPMCMC} method combines \gls{SMC} and \gls{CSMC} to design a powerful \gls{MCMC} method that can take advantage of parallel computing. As a motivating example we show how \gls{SMC} and \gls{CSMC} can be used to evaluate inference in the context of approximate Bayesian inference.

There are several avenues open for further work and research in \gls{SMC} methods and their application. We provide a few  examples below. 
%
The proposal distribution is crucial for performance, and finding efficient ways to learn useful proposals is an area of continued research. The twisting potential can drastically improve accuracy. However, jointly learning twisting potentials and proposal distributions for complex models is a challenging and largely unsolved problem. 

%
A rather obvious---but still very interesting and important---direction consists in adapting the \gls{SMC} algorithm to parallel and distributed computing environments. This is indeed an active area of research, see e.g.  \citet{LeeW:2016,Lee:2010gpuSMC}. However, we have reason to believe that there is still a lot that can be done in this  direction. Just to give a few concrete examples we mention the fact that we can assemble \gls{SMC} algorithm in new ways that we directly designed for parallel and distributed computing. As a concrete example we mention the smoother developed by \citet{JacobLS:2019} which is based on the clever debiasing scheme of \citet{GlynnR:2014}. Perhaps the divide and conquer SMC algorithm \citep{LindstenJNKSAB:2017} can be useful in this effort. Its successful implementation within a probabilistic programming environment with support for parallel computation might give rise to an interesting scalable use of \gls{SMC}.

%
An important observation is that \gls{SMC} is increasingly being \emph{used as a component} in various larger (often iterative)  algorithms. Just to give a few examples of this we have Bayesian learning, where the particle filter is used within a Markov chain Monte Carlo algorithm \citep{andrieu2010particle,lindsten2014particle}, iterated filtering \citep{Ionides:2015} and maximum likelihood using expectation maximization \citep{SchonWN:2011}. The idea of using \gls{SMC} within itself has occurred in several forms, including SMC$^2$ by \citet{chopin2013smc2} and nested \gls{SMC} by \citet{naesseth2015nested}. There are still most likely many interesting algorithms to uncover by following this idea in new directions.

There has recently been a very interesting line of research when it comes to merging deep learning architectures---such as recurrent neural networks, temporal convolutional networks and gated recurrent units---with stochastic latent models---such as the state space model---for sequence modelling, see  e.g.~\citep{BaiKK:2018,AksanH:2019,OsendorferB:2014,FraccaroSPW:2016}. An important idea here is to combine the expressiveness of the deep neural networks with the stochastic latent variable model's ability to represent uncertainty. These models are often trained by optimizing the ELBO objective. It has recently been shown that we can exploit the fact that the \gls{SMC} estimator of the likelihood is unbiased to construct tighter lower bounds of this objective \citep{anh2018autoencoding,maddison2017,naesseth18a}. This can most likely open up for better training of combinations of deep learning architectures and stochastic latent models.

Given that we have now understood that  the \gls{SMC} methods are much more generally applicable than most of us first thought we foresee an interesting and creative future for these algorithms. We hope that this paper can serve as a catalyst for further research on these methods and in particular on their application in machine learning.


\newpage 
\chapter*{Acknowledgments}
This research was financially supported by the Swedish Foundation for Strategic Research (SSF) via the projects \emph{ASSEMBLE} (contract number: RIT15-0012) and \emph{Probabilistic Modeling and Inference for Machine Learning} (contract number: ICA16-0015), and by the Swedish Research Council via the projects \emph{NewLEADS - New Directions in Learning Dynamical Systems} (contract number: 621-2016-06079) and 
\emph{Learning of Large-Scale Probabilistic Dynamical Models} (contract number: 2016-04278).

\bibliographystyle{abbrvnat}
\bibliography{references}

\end{document}